\documentclass{article}

\usepackage{amsthm}

\usepackage[accent=0072B2,font=palatino,headings=sans]{simpleicml}
\usepackage{sectionnav}
\usepackage{capt-of} 
\icmltitle{LeJEPA: Provable and Scalable\\ Self-Supervised Learning Without the Heuristics
}
\icmlauthors{%
  Randall Balestriero\textsuperscript{1,2},* \quad
  Yann LeCun\textsuperscript{3,2},*
}
\icmlaffiliations{%
  \textsuperscript{1} Brown University \quad 
  \textsuperscript{3} New York University (NYU) \quad
  \textsuperscript{2} Meta-FAIR \\
  \textsuperscript{*} Equal contribution
}
\icmlabstract{
\vspace{-0.2cm}
Learning manipulable representations of the world and its dynamics is central to AI. Joint-Embedding Predictive Architectures (JEPAs) offer a promising blueprint, but lack of practical guidance and theory has led to ad‑hoc R\&D. We present a comprehensive theory of JEPAs and instantiate it in {\bf LeJEPA}, a lean, scalable, and theoretically grounded training objective. First, we identify the isotropic Gaussian as the optimal distribution that JEPAs' embeddings should follow to minimize downstream prediction risk. Second, we introduce a novel objective--{\bf Sketched Isotropic Gaussian Regularization} (SIGReg)--to constrain embeddings to reach that ideal distribution. Combining the JEPA predictive loss with SIGReg yields LeJEPA with numerous theoretical and practical benefits: (i) single trade‑off hyperparameter, (ii) linear time and memory complexity, (iii) stability across hyper-parameters, architectures (ResNets, ViTs, ConvNets) and domains,  (iv) heuristics-free, e.g., no stop‑gradient, no teacher–student, no hyper-parameter schedulers, and (v) distributed training-friendly implementation requiring only $\approx$50 lines of code. Our empirical validation covers 10+ datasets, 60+ architectures, all with varying scales and domains. As an example, using imagenet-1k for pretraining and linear evaluation with frozen backbone, LeJEPA reaches 79\% with a ViT-H/14. We hope that the simplicity and theory-friendly ecosystem offered by LeJEPA will reestablish self-supervised pre-training as a core pillar of AI research (\href{https://github.com/rbalestr-lab/lejepa}{GitHub repo}).
\\[0em]
\begin{center}

    \vspace{-0.7cm}
    \begin{minipage}{0.44\linewidth}
        \centering
        \includegraphics[width=\linewidth]{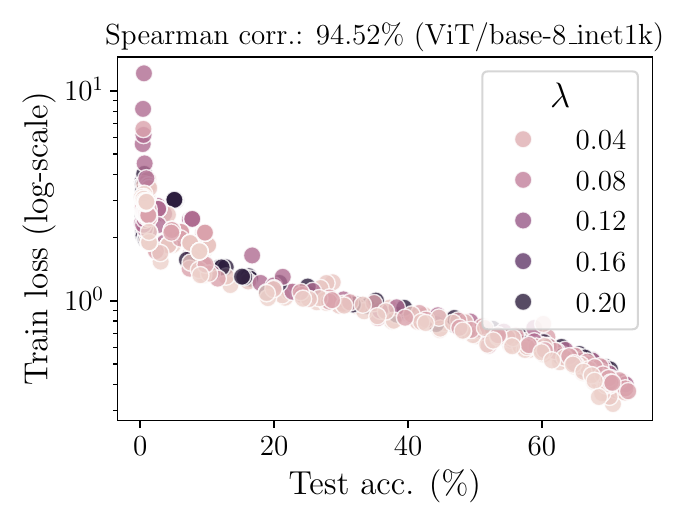}
    \end{minipage}%
    \hfill
    \begin{minipage}{0.54\linewidth}
        \centering
        \includegraphics[width=\linewidth]{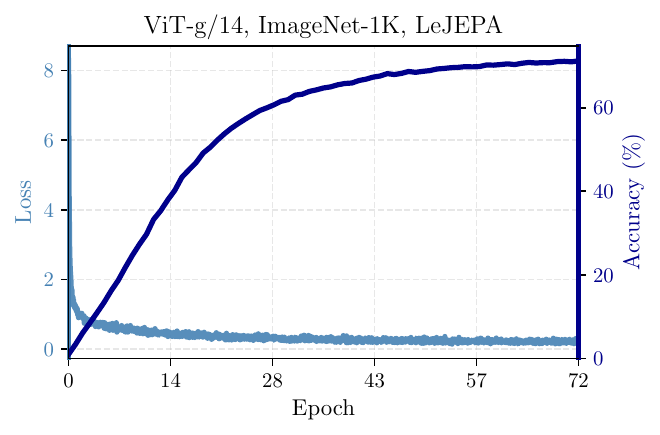}
    \end{minipage}
    
    \vspace{0.03cm}
    
    \begin{minipage}{0.44\linewidth}
        \centering
        \includegraphics[width=0.485\linewidth]{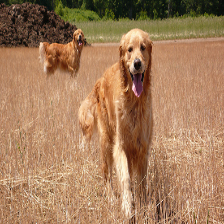}\hfill
        \includegraphics[width=0.485\linewidth]{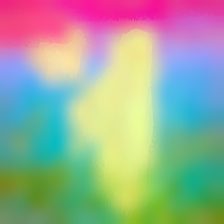}
    \end{minipage}%
    \hfill
    \begin{minipage}{0.54\linewidth}
        \centering
        \footnotesize
        \setlength{\tabcolsep}{2.5pt}
        
        \begin{tabular}{lcccc}
        \toprule
         & \multicolumn{2}{c}{\textbf{Full FT}} & \multicolumn{2}{c}{\textbf{Frozen}} \\
        \cmidrule(lr){2-3} \cmidrule(lr){4-5}
        \textbf{Method} & \textbf{1-sh} & \textbf{Full} & \textbf{1-sh} & \textbf{Full} \\
        \midrule
        \multicolumn{5}{l}{\textit{LeJEPA (in-domain)}} \\
        \;\;ConvNeXt-V2 Nano & \textbf{29.42} & 82.72 & 28.74 & 76.52 \\
        \;\;ResNet-34 & 24.27 & \textbf{83.28} & \textbf{31.08} & \textbf{78.17} \\
        \midrule
        \multicolumn{5}{l}{\textit{Frontier (transfer)}} \\
        \;\;DINOv2 ViT-S/16 & 21.05 & 78.34 & 27.68 & 67.62 \\
        \;\;DINOv3 ViT-S/16 & 24.71 & 81.60 & 30.17 & 71.38 \\
        \bottomrule
        \end{tabular}
    \end{minipage}
    
    \vspace{0.05cm}
    \captionof{figure}{\textbf{LeJEPA overview.} \textbf{Top-left:} Training loss exhibits strong correlation with downstream linear probe performance on ImageNet-1k (ViT-base), providing the first practical loss for model selection without supervised probing. \textbf{Top-right:} Training stability without heuristics even on 1.8B ViT-g models, stable training loss. \textbf{Bottom-left:} PCA features from ImageNet-1k pretrained LeJEPA ViT-Large demonstrate clear semantic relationships. \textbf{Bottom-right:} Galaxy10 in-domain results showcasing LeJEPA's in-domain pretraining consistently outperforms state-of-the-art frontier foundation models transfer learning (DINOv2/v3 trained on natural images) across data regimes from 1-shot to full supervision. This demonstrates that \textit{domain-specific SSL beats generic transfer learning}, even against massive-scale frontier models, when the framework scales effortlessly to any domain, model, and data scale.}
    \label{fig:teaser}
    \vspace{-0.5cm}
\end{center}
}
\icmlrunningtitle{LeJEPA:}
\usepackage{natbib}
\usepackage{hyperref}
\usepackage{microtype}
\usepackage{graphicx}
\usepackage{subfigure}
\usepackage{booktabs} 
\usepackage[toc,page,header]{appendix}
\usepackage{minitoc}
\usepackage{makecell}
\usepackage{dblfloatfix}
\usepackage{amsmath}
\usepackage[capitalize,noabbrev]{cleveref}
\usepackage{mathtools}
\usepackage{cuted} 
\usepackage[most]{tcolorbox}
\tcbuselibrary{theorems}
\usepackage{caption}
\usepackage{enumitem}
\usepackage{listings}

\usepackage{tgheros} 
\definecolor{bg}{HTML}{F8FAFD}
\definecolor{keyword}{HTML}{0077B6}
\definecolor{string}{HTML}{E76F51}
\definecolor{comment}{HTML}{A0AEC0}
\definecolor{number}{HTML}{457B9D}
\definecolor{function}{HTML}{2A9D8F}
\definecolor{class}{HTML}{F4A261}
\definecolor{text}{HTML}{2D3A4A}
\lstdefinestyle{pytorchheros}{
    backgroundcolor=\color{bg},
    basicstyle=\fontfamily{qhv}\selectfont\footnotesize\color{text}, 
    keywordstyle=\color{keyword}\bfseries,
    stringstyle=\color{string},
    commentstyle=\color{comment}\itshape,
    numberstyle=\color{number},
    identifierstyle=\color{text},
    classoffset=1,
    morekeywords={Net},
    keywordstyle=\color{class}\bfseries,
    classoffset=0,
    emph={__init__,forward,print}, emphstyle=\color{function},
    frame=single,
    framerule=0pt,
    rulecolor=\color{bg},
    tabsize=4,
    showstringspaces=false,
    breaklines=true,
    linewidth=\linewidth,
    xleftmargin=0em,
    xrightmargin=0em,
    aboveskip=1em,
    belowskip=1em,
    literate={~}{{\textasciitilde}}1
}
\renewcommand\lstlistingname{Algorithm}
\crefname{lstlisting}{\MakeLowercase\lstlistingname}{\MakeLowercase\lstlistingname s}
\Crefname{lstlisting}{\lstlistingname}{\lstlistingname s}
\usepackage{float} 
\newfloat{lstfloat}{htbp}{lop}
\floatname{lstfloat}{Listing}

\usepackage{amsmath,amsthm}
\usepackage{bm,bbm}

\usepackage{varwidth}

\usepackage{hyperref}
\usepackage{cleveref}
\usepackage{mathtools}
\usepackage{algpseudocode}
\usepackage{mdframed}

\usepackage{multicol}
\usepackage{multirow}
\usepackage{setspace}

\usepackage{colortbl}

\usepackage[svgnames]{xcolor}
\usepackage{framed}

\newcommand{\1}{\mathbf{1}}
\newcommand{\norm}[1]{\left\lVert #1\right\rVert}
\newcommand{\grad}{\nabla}
\newcommand{\Hess}{H}
\newcommand{\tr}{\mathrm{tr}}
\newcommand{\Ball}{\mathrm{B}}
\newcommand{\vol}{v_d}

\usepackage{tikz,ifthen}
\usetikzlibrary{positioning}
\usetikzlibrary{shapes}
\usetikzlibrary{arrows}
\usetikzlibrary{fit}
\usetikzlibrary{calc}
\usetikzlibrary{shapes.misc}

\crefname{defi}{defn.}{defns.}

\def\1{\bm{1}}

\DeclareMathOperator{\Tr}{Tr}

\def\vtheta{{\bm{\theta}}}

\def\va{{\bm{a}}}

\def\vq{{\bm{q}}}

\def\vt{{\bm{t}}}
\def\vu{{\bm{u}}}

\def\vx{{\bm{x}}}
\def\vy{{\bm{y}}}
\def\vz{{\bm{z}}}

\def\mX{{\bm{X}}}

\def\mZ{{\bm{Z}}}

\DeclareMathAlphabet{\mathsfit}{\encodingdefault}{\sfdefault}{m}{sl}
\SetMathAlphabet{\mathsfit}{bold}{\encodingdefault}{\sfdefault}{bx}{n}

\def\sA{{\mathbb{A}}}

\newcommand{\E}{\mathbb{E}}

\newcommand{\R}{\mathbb{R}}

\newcommand{\Cov}{\mathrm{Cov}}

\sectionheaderline{%
  \seclink{sec:intro}{1}{Intro} | 
  \seclink{sec:background}{2}{Background} | 
  \seclink{sec:gaussian}{3}{Why Gaussian?} |
  \seclink{sec:bcs}{4}{SIGReg} |
  \seclink{sec:lejepa}{5}{LeJEPA} |
  \seclink{sec:experiments}{6}{Experiments}
}

\begin{document}
\icmlmaketitle

\newtcbtheorem
  [
  crefname={detail}{detail}]
  {detail}
  {Experiment Details}
  {%
    fontupper=\small,
    colback=orange!5,
    colframe=orange!35!black,
    fonttitle=\bfseries,
    boxsep=1pt,
    left=1.5mm,
    right=1.5mm,
    top=2mm,
    bottom=1mm,
  }
  {detail}

\newtcbtheorem
  [
  crefname={def.}{def.}]
  {definition}
  {Definition}
  {%
    fontupper=\small,
    colback=green!5,
    colframe=green!35!black,
    fonttitle=\bfseries,
    boxsep=1pt,
    left=1.5mm,
    right=1.5mm,
    top=2mm,
    bottom=1mm,
  }
  {def}

\newtcbtheorem
  [
  crefname={thm.}{thms.}]
  {theorem}
  {Theorem}
  {%
    fontupper=\small,
    colback=red!5,
    colframe=red!35!black,
    fonttitle=\bfseries,
    boxsep=1pt,
    left=1.5mm,
    right=1.5mm,
    top=2mm,
    bottom=1mm,
  }
  {theorem}
\newtcbtheorem
  [
  crefname={lemma.}{lemmas.}]
  {lemma}
  {Lemma}
  {%
    fontupper=\small,
    colback=blue!3,
    colframe=blue!35!black,
    fonttitle=\bfseries,
    boxsep=1pt,
    left=1.5mm,
    right=1.5mm,
    top=2mm,
    bottom=1mm,
  }
  {lemma}
\newtcbtheorem
  [
  crefname={prop.}{props.}]
  {proposition}
  {Proposition}
  {%
    breakable,
    enhanced,
    fontupper=\small,
    colback=red!5,
    colframe=red!35!black,
    fonttitle=\bfseries,
    boxsep=1pt,
    left=1.5mm,
    right=1.5mm,
    top=2mm,
    bottom=1mm,
  }
  {proposition}
\newtcbtheorem
  [
  crefname={cor.}{corrs.}]
  {corollary}
  {Corollary}
  {%
    breakable,
    enhanced,
    fontupper=\small,
    colback=red!5,
    colframe=red!35!black,
    fonttitle=\bfseries,
    boxsep=1pt,
    left=1.5mm,
    right=1.5mm,
    top=2mm,
    bottom=1mm,
  }
  {corollary}

\begin{figure*}[t!]
    \centering
    \begin{minipage}{0.34\linewidth}
    \includegraphics[width=\linewidth]{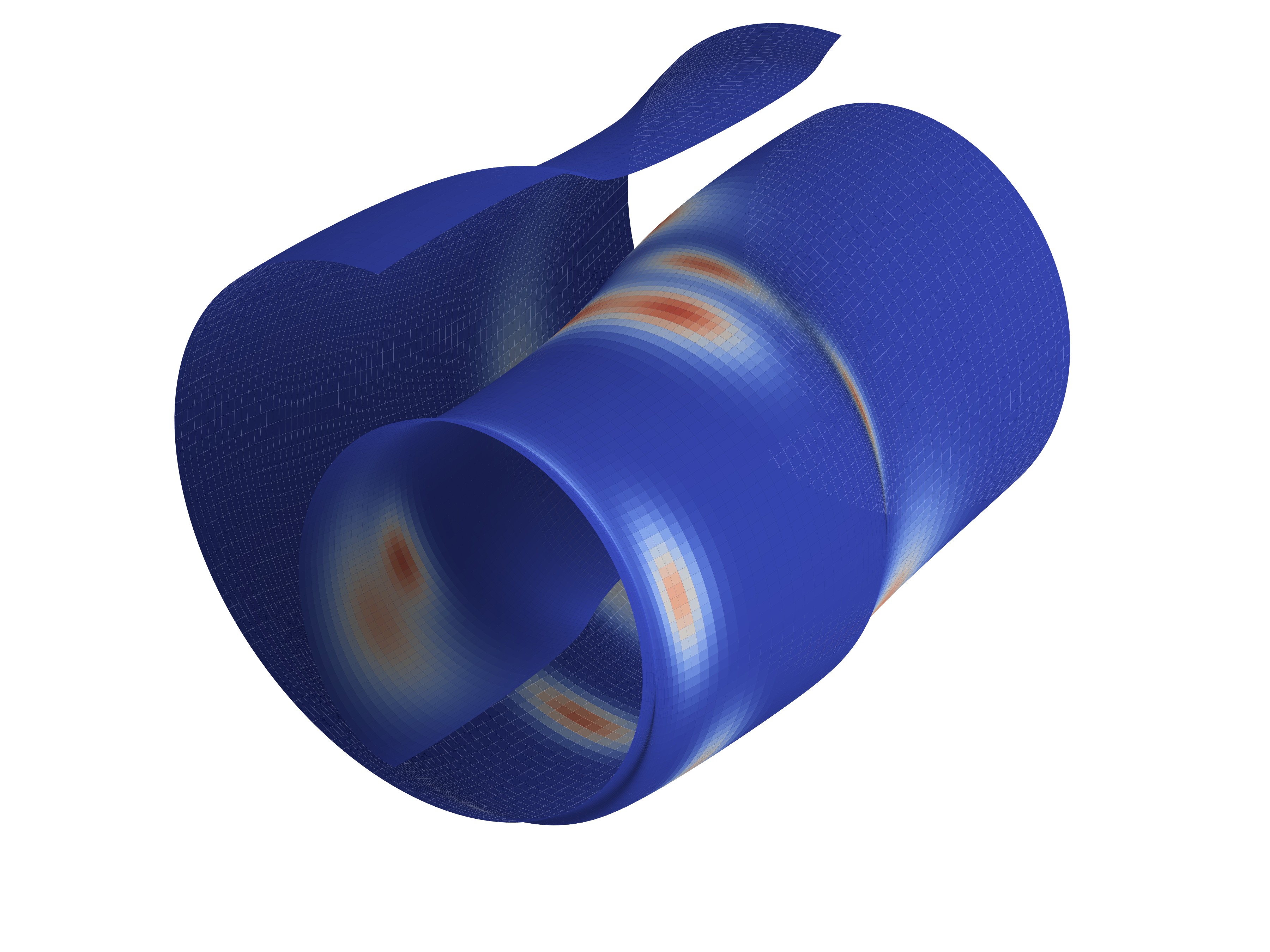}
    \end{minipage}
    \begin{minipage}{0.05\linewidth}
    \hspace{-0.4cm}{\Large $\underset{\rightarrow}{f_{\vtheta}}$}
    \end{minipage}
    \begin{minipage}{0.6\linewidth}    \hspace{-0.3cm}\includegraphics[width=\linewidth]{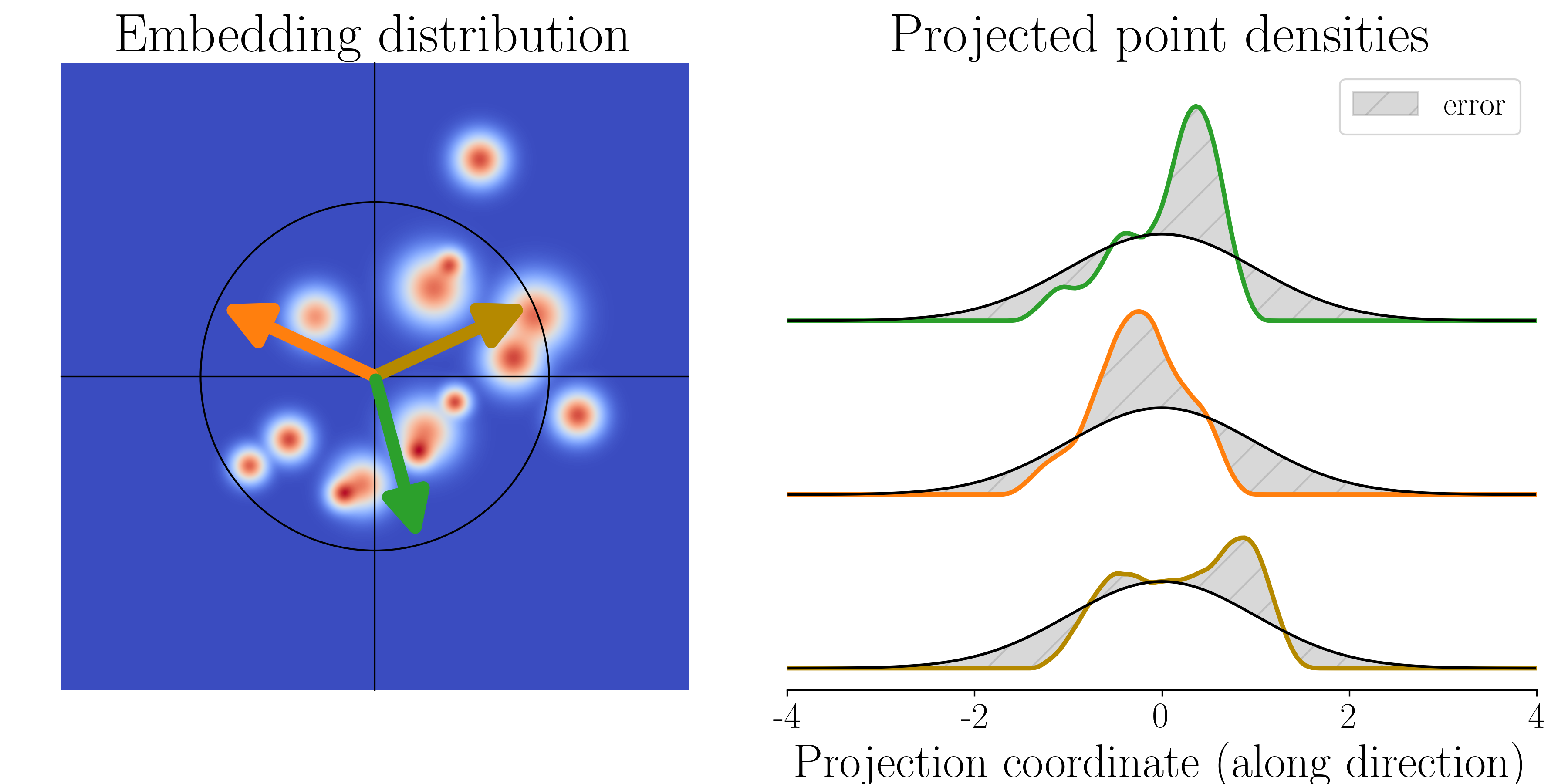}
    \end{minipage}
    \caption{ \textbf{Sketched Isotropic Gaussian Regularization (SIGReg):}~Given some arbitrary input data with density $p_{x}$ with support that may or may not lie on a manifold ({\bf left}), a Deep network (DN) encoder ($f_{\vtheta}$) produces embeddings $\vz=f_{\vtheta}(\vx)$ with some distribution $\vz \sim p_{z}$ ({\bf middle}). Our proposed Backward Cramér-Wold Statistics (\cref{sec:bcs}) objective pushes $p_z$ to match a target distribution $p_t$ by projecting the embeddings along $1d$ directions ({\bf middle, arrows}) and enforcing that the univariate densities ({\bf right, colored lines}) match the distribution of $p_t$, projected along the same directions. Any popular statistical test (provided in \cref{sec:tests}) can assess the goodness-of-fit--in practice we argue for characteristic function tests (\cref{sec:CF_better}). By using SIGReg with $p_t$ isotropic Gaussian ({\bf right, black lines}), we introduce a lean and provably optimal (\cref{sec:gaussian}) JEPA, coined LeJEPA, free of numerous heuristics and able to produce competitive performances (\cref{sec:lejepa,sec:experiments}).}
    \label{fig:bcs_teaser}
\end{figure*}

\section{Introduction}
\label{sec:intro}

Learning manipulable representations of the world and its dynamics is a long‑standing question in AI, with roots dating back centuries ago \citep{von1867handbuch,tolman1948cognitive,gregory1980perceptions,sutton1991dyna,friston2010free}. Across domains, e.g., image recognition, robotics, physics, space exploration, the unifying question is {\em how to learn an organized and actionable high‑dimensional embedding space from observations?} Using Deep Networks--parameterized nonlinear operators $f_{\vtheta}$--to map observations to embeddings is a standard first piece of that puzzle \citep{lecun2015deep,goodfellow2016deep}. The second, less standardized, piece of that puzzle is {\em how to train $f_{\vtheta}$}. Joint-Embedding Predictive Architectures (JEPAs) suggest training $f_{\vtheta}$ by maximizing predictive agreement between the embeddings of semantically related {\em views} \citep{bromley1993signature,lecun2022path,balestriero2023cookbook}. Views can come in two forms: transformations or corruptions.  They can involve masking, cropping, blurring, temporal or spatial translations, geometric or photometric transformations, viewpoint changes, views from different sensor modalities, etc. The supervised forms involve human-produced components such as image-caption pairs, text-code pairs, etc \citep{tian2020makes}. In any case, views are expected to share some degree of semantic relationship to allow the prediction task to align $f_{\vtheta}$'s embeddings towards the underlying knowledge present in the data.

Alas, JEPA's prediction task admits failure modes, such as representation collapse, where $f_{\vtheta}$ maps all inputs to 
nearly identical embeddings ({\em complete collapse}) or to a low-dimensional subspace 
({\em dimensional collapse}) \citep{jing2021understanding}\citep{jing2021understanding,cosentino2022toward,balestriero2022contrastive}. To mitigate such shortcut solutions, state‑of‑the‑art recipes rely on heuristics--stop‑gradient \citep{chen2020simple}, asymmetric view generation \citep{wang2022importance}, teacher–student networks with carefully tuned EMA schedules \citep{caron2021emerging,tian2021understanding}, explicit normalization and whitening layers \citep{ermolov2021whitening,chen2021empirical}--and a delicate balance of hyperparameters. As a result, today's JEPA training is brittle and most research has shifted toward scaling data \citep{vo2024automatic}, models \citep{fan2025scaling} and even post-training \cite{rodas2025diet} while leaving the theoretical foundations of JEPAs largely unexplored.

Our study proposes to break that cycle by questioning some of the fundamental design principles underpinning JEPAs. That introspection will start by asking {\em what are the necessary conditions that JEPAs should abide by?} Those minimal conditions will then act as {\em axioms} for us to design a novel and lean JEPA. We identify two axioms: (i) solving the prediction task while (ii) enforcing an isotropic Gaussian distribution of the embeddings (\cref{sec:gaussian}). While (i) follows standard practice \citep{balestriero2022contrastive}, we introduce  in \cref{sec:bcs} a novel  distribution matching objective--Sketched Isotropic Gaussian Regularization (SIGReg)--to enforce (ii). The use of SIGReg not only removes the need for the numerous heuristics previously employed to prevent representation collapse, but SIGReg also exhibits favorable scaling properties as its {\em memory and computational complexity is linear in dimension and sample size}. Crucially, SIGReg's isotropic Gaussian enforcement solves the collapsed shortcut solution and provably minimizes the model's expected risk over the space of downstream tasks to be encountered post-training. The resulting JEPA solution--coined Latent-Euclidean JEPA (LeJEPA)--is introduced in \cref{sec:lejepa}. Beyond theoretical optimality, LeJEPA offers numerous benefits such as (i) provable statistical guarantees, (ii) removal of heuristics such as teacher-student networks, (iii) linear memory and computational complexity, and most importantly (iv) a unified design with a single trade-off parameter that works out of the box across datasets, architectures and scales (see \cref{sec:experiments}). We summarize our contributions below.

{\bf Contribution 1: We prove the optimal embedding distribution for foundation models.}~We 
establish that the isotropic Gaussian uniquely minimizes downstream prediction risk across 
broad task families. In \cref{sec:gaussian}, we derive this result rigorously for both 
linear (\cref{sec:linear_probing}) and nonlinear probes (\cref{sec:nonlinear_probing}), 
providing the first principled answer to what distribution $f_{\vtheta}$'s embeddings 
should follow. This theoretical result transforms JEPA design from heuristic exploration 
to targeted optimization.
{\bf Contribution 2: We introduce SIGReg, a distribution matching objective that uniquely 
combines provable correctness with computational efficiency at scale.}~We present {\em 
Sketched Isotropic Gaussian Regularization} (SIGReg), a novel objective that enforces 
distributional alignment via random projections and characteristic-function matching 
(\cref{sec:bcs,fig:bcs_teaser}). SIGReg provides statistical guarantees 
(\cref{sec:general_test,sec:tests}) while achieving linear complexity and bounded 
gradients—a combination that existing distribution matching methods do not offer. 
Critically, its projection-based construction defeats the curse of dimensionality 
(\cref{sec:dimension}), making it both theoretically sound and practically efficient 
for high-dimensional embeddings.

{\bf Contribution 3: We design LeJEPA, a statistically optimal JEPA that eliminates 
collapse by construction.}~By combining JEPA's predictive objective with SIGReg targeting 
the isotropic Gaussian, we introduce {\em LeJEPA}—Latent-Euclidean JEPA 
(\cref{sec:lejepa}). LeJEPA requires only a single hyperparameter, 
eliminates representational collapse without stop-gradients or teacher-student 
architectures, and transfers across architectures and datasets without hyperparameter 
tuning. This demonstrates that principled theory directly yields practical simplicity.

{\bf Contribution 4: We validate LeJEPA at scale across diverse architectures and 
establish in-domain pretraining as viable.}~Our experiments (\cref{sec:experiments}) 
span ViTs, ConvNeXts, ResNets, MaxViTs, and Swin Transformers at scales approaching 
1 billion parameters, where LeJEPA matches or exceeds state-of-the-art methods while 
maintaining training simplicity and robustness. Critically, on domain-specific datasets 
(Galaxy10, Food101), LeJEPA outperforms DINOv2-based transfer learning when pretrained 
directly on target data. This challenges the transfer learning paradigm and demonstrates 
that principled SSL can unlock effective in-domain pretraining—previously considered 
impractical for small datasets.
\section{Background and Notations}
\label{sec:background}

We start by introducing some of the notations we will be using throughout our manuscript (\cref{sec:notations}), followed by a review of JEPAs (\cref{sec:JEPA}), and existing literature studying their design (\cref{sec:mi}).

\subsection{Notations and Definitions}
\label{sec:notations}

\begin{figure*}[t!]
    \begin{definition}{JEPA}{}
    \begin{align}
    {\rm JEPA}(\vx) \iff& {\rm Enc}\left(\vx_{n,t+1,.}\right) \text { is predictable from }{\rm Enc}\left(\vx_{n,t,.}\right), \forall n,t\text{ and } {\rm Enc}\left(\vx_{.,.,.}\right) \text{ is not degenerate}.\label{def:SSL}
\end{align}
\begin{minipage}{0.32\linewidth}
\includegraphics[width=\linewidth]{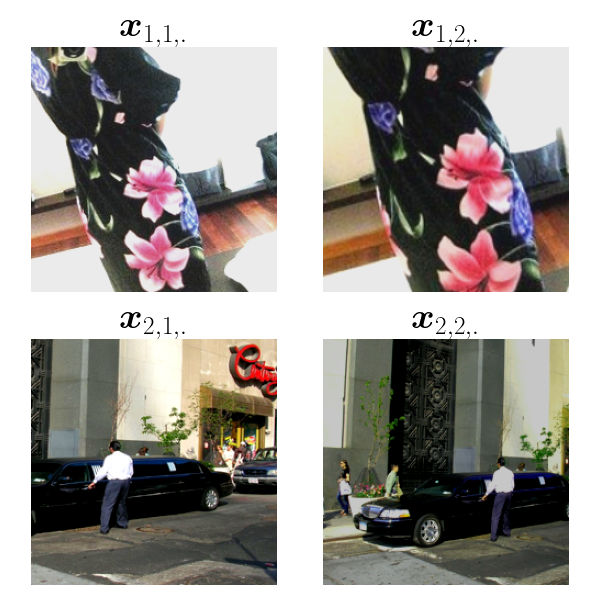}
\end{minipage}
\begin{minipage}{0.32\linewidth}
\includegraphics[width=\linewidth]{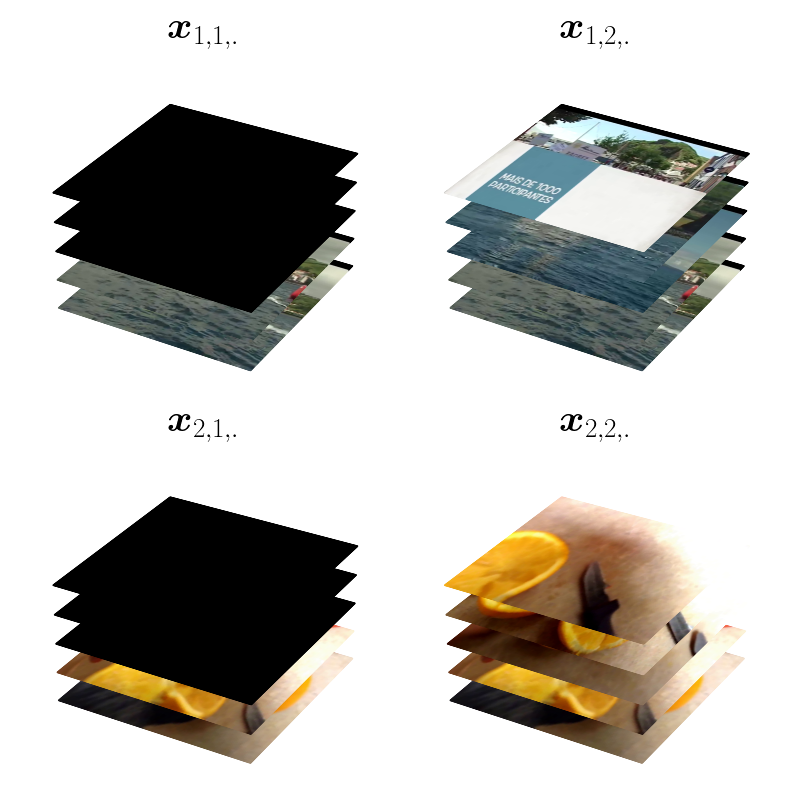}
\end{minipage}
\begin{minipage}{0.32\linewidth}
\includegraphics[width=\linewidth]{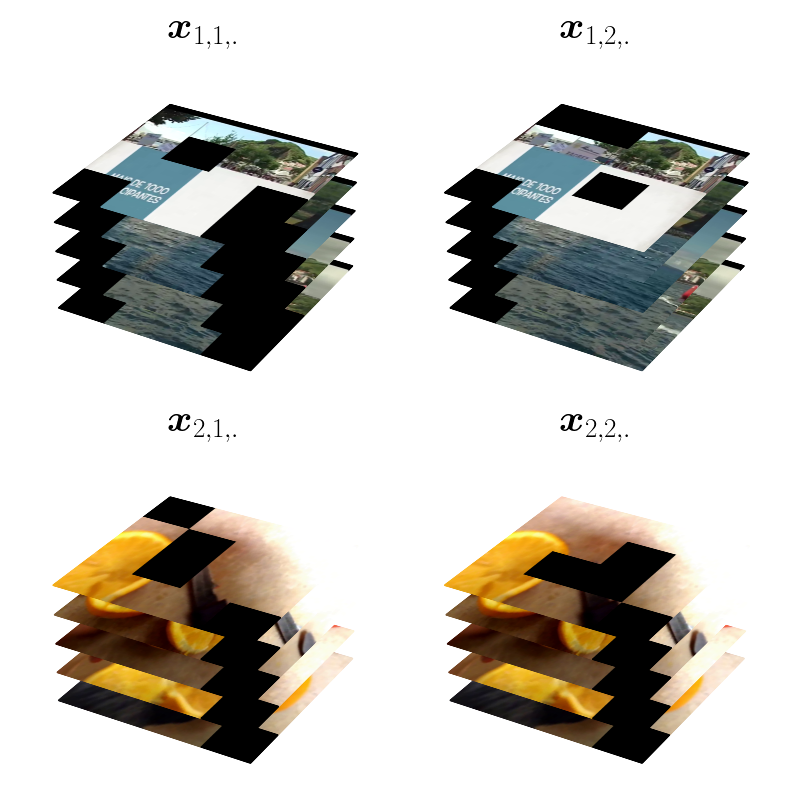}
\end{minipage}
\end{definition}
\end{figure*}

{\bf Data.}~We are in possession of a dataset of shape $(N, V, D) \in {\mathbb{N}^*}^3$ where $N$ is the number of samples, $V$ is the number of views, and $D$ is the dimension. One entry of this dataset is accessed via $\vx_{n,v,d}$. Those dimensions are often interpreted as follows: ({\bf N}) is the number of independent samples, e.g., different images or different videos, ({\bf V}) is the number of {\em views}, e.g., data-augmentations for images, frames for videos, and ({\bf D}) is the dimension of each $\vx_{n,v}$, e.g., number of RGB pixels for images. In many cases the ordering over $V$ is given by {\em time}--but in some cases, e.g., data-augmentation of an image, ordering becomes irrelevant. Our study does not require any particular choice to organize one's dataset into a $(N,V,D)$ tensor--{\em and none of our theory and implementation assumes a particular design decision for that tensor}. However, we will rely on the following two properties, ({\em independence}) the samples $\vx_{n},\vx_{n'}$ have been obtained independently from each other $\forall n\not = n'$, and ({\em identically distributed}) the sampling process was identical among $\vx_{n}, \forall n$.

{\bf Deep Networks.}~Today's AI solutions rely on {\em Deep (Neural) Networks} (DNs), which are compositions of a large number of parameterized linear and nonlinear operators.
We denote the DN's mapping as $f_\vtheta: \R^D \rightarrow \R^K$ with $K$ the dimension of the embedding space. The internals of $f_\vtheta$ are designed by the researcher to incorporate as much prior knowledge about the data as possible. The details of $f_\vtheta$ are irrelevant to our study--as we will see the proposed LeJEPA works out-of-the-box on any $f_\vtheta$. In any case, all the {\em learnable parameters} are gathered in the vector  $\vtheta \in \R^P$, with $P$ counting the total number of parameters. A central challenge in AI research is to design the right architecture and training objective so that $\vtheta$ can be learned from gradient descent to ultimately produce a useful system, or foundation model, $f_{\vtheta}$.

{\bf JEPAs.}~A foundation model is any system, e.g., a DN, able to solve numerous downstream tasks without requiring any change in its internal parameters $\vtheta$. This is in sharp contrast with a supervised model that only considers its training task. JEPAs have formally been introduced by \citet{lecun2022path} as a vehicle to produce foundation models. The core building blocks of JEPAs rely on numerous well-established techniques such as siamese networks \citep{bromley1993signature} and predictive coding \citep{helmholtz1867handbook,bruner1949perception}. While the exact blueprint of JEPAs varies greatly between use-cases, they all rely on two core principles: (i) being able to predict the embedding of a view $\vx_{n,v}$ from the embedding of another view $\vx_{n,v'}, v' \not = v$, all while (ii) ensuring that the embeddings do not become degenerate. Concretely, once a JEPA is designed and trained, it should be able to solve numerous downstream tasks in zero or few shots. 
The JEPA objective function, along with some examples for $\vx$, is provided in \cref{def:SSL}. The {\em predictability} criterion can be done by directly comparing the embeddings of the partial views $Enc(\vx_{n,v,.})$ and $Enc(\vx_{n,v',.})$ with a metric, e.g., $\ell_p$. In some cases, an additional DN coined {\em Pred}, is employed to compare $Pred(Enc(\vx_{n,v,.}))$ against $Enc(\vx_{n,v',.})$--which is only justified when there exists an asymmetry between the information content of the different views, e.g., by conditioning the predictions on observed actions from robotics data \citep{khazatsky2024droid}.

\subsection{The Need for Reliable Pretraining}
\label{sec:JEPA}

The JEPA's prediction task is designed based on a priori knowledge of the data. Its design is often quite natural since it is relatively intuitive to form $\vx$ so that its views share the relevant information content one hope to capture. On the other hand, the design of the ``anti-collapse'' criterion is much closer to a game of Whac-A-Mole. Today's designs rely on many different under-specified safeguards which are carefully combined in the hope that degenerate shortcut solutions are avoided during training. Such mechanisms include (i) feature whitening \citep{ermolov2021whitening,bardes2021vicreg}, (ii) negative samples \citep{chen2020simple,he2020momentum}, and (iii) asymmetric views and teacher-student networks with stop-gradient \citep{caron2021emerging,assran2023self}. Those mechanisms all suffer from at least two of the following limitations: (i) under-specification, i.e., the criteria can be minimized while embeddings are in a degenerate configuration, (ii) quadratic time and memory complexity with mini-batch size and/or embedding dimension, (iii) sensitivity to data distribution, hyperparameters, architecture, and (iv) lack of theoretical understanding and guarantees.

\subsection{The Need for Actionable Theory}
\label{sec:mi}

For decades, the two major solutions for AI were supervised learning \citep{lecun2015deep} and learning by reconstruction \citep{rumelhart1986learning}--sometimes combined together, e.g., for semi-supervised learning \citep{kingma2014semi}. In supervised learning, the labels both ensure that semantically similar samples are close to each other in embedding space while preventing complete representation collapse. In particular, it is possible to measure the amount of collapse in supervised learning as a function of the number of classes \citep{papyan2020prevalence}. The reconstruction objective is similarly well suited to prevent representation collapse as the original input must be recovered from the embeddings, i.e., the embeddings must be as informative about the input as possible--up to some optional denoising tasks that users can setup as part of the training \citep{vincent2010stacked}.

Because supervised and reconstruction-based learning have been widely studied for decades, there exists a large body of work to explain and inform practical designs--as well as studying their limitations in producing foundation models \citep{balestriero2024learning,van2025joint}. This is not the case for the more recent JEPAs where empirical advances quickly outpace anyone hoping to delve into their inner workings. This dynamic led the community to focus on post-hoc theoretical justification of already found solutions \citep{liu2021self,shwartz2024compress,shwartz2022we,zhang2023matrix}. In most cases, those studies involve the {\em Mutual Information (MI)} \citep{shannon1948mathematical,cover1999elements} whose different bounds recover established methods \citep{gutmann2010noise,ma2018noise,oord2018representation,poole2019variational,hjelm2018learning,mcallester2020formal}. Because existing studies focus on explaining and interpreting already developed JEPAs, too little principled guidance and innovation has been brought forward. Instead, most of the recent empirical advances take the form of collecting larger dataset, scaling up pre-existing training recipes \citep{goyal2019scaling,chen2020big,oquab2023dinov2,fan2025scaling}, and deriving novel data curation processes \citep{vo2024automatic,kerdreux2025efficient}. 

In contrast, our goal in the following \cref{sec:gaussian,sec:bcs,sec:lejepa} will be to derive a novel JEPA solution from first principles, i.e., whose design relies on proved necessary conditions for optimality, and with a pretraining recipe that can finally reconcile exploratory research, scalability, and state-of-the-art performances.

\section{Latent Euclidean: Embeddings Should be Isotropic Gaussian}
\label{sec:gaussian}

We address a fundamental question: {\em which distribution should ${\rm Enc}(\vx)$ 
follow to minimize empirical risk on any downstream task?} We prove that the 
isotropic Gaussian is the unique optimal distribution for both linear 
(\cref{sec:linear_probing}) and nonlinear probing (\cref{sec:nonlinear_probing}), 
with geometric intuition provided in \cref{sec:le}. This theoretical result 
establishes the necessary design principle for our JEPA; \cref{sec:bcs} 
then provides the practical implementation to achieve it.

\subsection{Linear Probing}
\label{sec:linear_probing}

We begin by identifying the optimal distribution for $f_\vtheta$'s embeddings 
by analyzing linear probes--one of the most popular methods for frozen 
encoder evaluation. Specifically, we ask: \emph{which distribution for 
$f_\vtheta(\vx)$ would be most favorable for solving arbitrary downstream tasks, 
i.e., for any realization of targets $\vy$?}

Denote as $\mZ \in \mathbb{R}^{N \times K}$ the matrix of $N$ embeddings, each $K$-dimensional, from $f_{\vtheta}(\vx_n)$. The {\em unknown} corresponding labels are denoted as $\vy \in \mathbb{R}^{N}$. Without loss of generality, we consider univariate targets; the 
following analysis extends to multivariate targets. The linear probe minimizes the following least square problem \citep{bishop2006pattern}
\begin{equation}
\hat{\beta} = \underset{\beta \in \mathbb{R}^K}{\arg\min} \|\vy - \mZ\beta\|_2^2+\lambda \|\beta\|_2^2,\tag{OLS}\label{eq:OLS}
\end{equation}
where $\hat{\beta}$ is the optimal probe parameters, and $\lambda \geq 0$ is an hyperparameter controlling the Tikhonov regularizer strength \citep{bishop1995training,golub1999tikhonov}. Despite not knowing $\vy$, it is possible to describe the bias and variance of the estimator $\hat{\beta}$ as a function of the distribution of $\mZ$. Consider two embeddings with identical column spans $\mZ_{\rm aniso}, \mZ_{\rm iso}$. $\mZ_{\rm aniso}$'s covariance matrix eigenvalues are given by $\{\lambda_k\}_{k=1}^K$ with at least two distinct values, while $\mZ_{\rm iso}$'s covariance matrix eigenvalues are all equal to $\frac{1}{K}\sum_{k=1}^{K}\lambda_k$. Hence, the two candidate embeddings $\mZ_{\rm aniso}, \mZ_{\rm iso}$ capture the same intrinsic features and have same energy, but different geometries.

\begin{lemma}[label={thm:linear_probe_bias}]{Anisotropy amplifies bias}{}
Whenever $\lambda_K>\lambda_1$, there always exists a downstream task ($\vy$) for which $\mZ_{\rm aniso}$ produces a higher bias estimator than $\mZ_{\rm iso}$ for $\lambda>0$. (Proof in \cref{proof:linear_probe_bias}.)
\end{lemma}
\begin{lemma}[label={thm:linear_probe_variance}]{Anisotropy amplifies variance}{}
With $\lambda=0$, the total variance of $\hat{\beta}$ \eqref{eq:OLS} is minimized for $\mZ_{\rm iso}$ with
$\text{tr}(\text{Var}(\hat{\boldsymbol{\beta}}_{\text{aniso}})) > \text{tr}(\text{Var}(\hat{\boldsymbol{\beta}}_{\text{iso}}))$. (Proof in \cref{proof:linear_probe_variance}.)
\end{lemma}

From the above \cref{thm:linear_probe_variance,thm:linear_probe_bias} we obtain that the distribution of features must be isotropic. We now move to nonlinear probing where the standard Gaussian will emerge as the unique optimum.

\subsection{Nonlinear Probing}
\label{sec:nonlinear_probing}

To allow for more flexible evaluation of the pretrained encoder $f_{\vtheta}$, it has become increasingly common to work with a nonlinear probe. We analyze two widely-used nonlinear methods: radius-based k-NN \citep{taunk2019brief,sun2010adaptive,zhang2017efficient,abu2019effects} for 
its simplicity and kernel methods \citep{nadaraya1964estimating,watson1964smooth} for their theoretical tractability. 

As in \cref{sec:linear_probing}, we ask ourselves which distribution of embeddings would be preferable for a foundation model. We first define our prediction function. The training data consists of the $N$ embeddings along with their training labels $\{(\vz_n,\vy_{n})\}_{n=1}^{N}$. The prediction, using radius-based k-NN for a query vector $\vq$ is formed as
\begin{align}
\widehat{\vy}(\vq) := \frac{1}{|\mathcal{N}_{r_0}(\vq)|}\sum_{n \in \mathcal{N}_{r_0}(\vq)}\vy_n,
\tag{kNN}\label{eq:kNN}
\end{align}
where $\mathcal{N}_{r_0}(\vq) = \{n : \|\vz_n - \vq\| \le r_0\}$. The specific choice of radius $r_0$ controls how many neighbors predictions are averaged to form the query's prediction. The kernel's prediction at a query $\vq\in\mathbb{R}^K$ is given by
\begin{align}
\widehat \vy(\vq)\triangleq \frac{\sum_{n=1}^N K_h(\vq-\vz_n)\vy_n}{\sum_{n=1}^N K_h(\vq-\vz_n)}.\tag{Kernel}\label{eq:NW}
\end{align}

We search over all distributions of Z subject to a fixed total variance 
constraint, e.g., $\Tr(\Cov(\mZ)) = \kappa_1$ or $\|\Cov(\mZ)\|_F=\kappa_2$. The specific value 
of $\kappa$ does not affect the optimal distribution shape.
Following the same type of derivations as done in the linear regime--with the exception of some additional regularity conditions--we are able to precisely identify the isotropic Gaussian as the unique optimum to minimize bias as formalized below.

\begin{theorem}[label={thm:nonlinear_optimal}]{isotropic Gaussian Optimality}{}
The integrated square bias (ISB) over query points is given by
\begin{align*}
\text{ISB}_{k\text{-NN}} =\frac{r_0^4}{(K+2)^2}\tau_g^2J(p)+O(r_0^4),&&\text{(k-NN)}\\
\text{ISB}_{\text{kernel}} \le \Big(\frac{h^2\mu_2(K)}{2}\Big)^2 \Big(2 B^2 + 8 L^2J(p)\Big)+o(h^4),&&\text{(kernel)}
\end{align*}
and among distributions with a scalar-based covariance constraint, the isotropic Gaussian  is the unique minimizer of the integrated square bias. (Proof in \cref{proof:knn_optimal,proof:kernel_optimal}.)
\end{theorem}

Numerous additional details and discussions on the regularity assumptions we employed are provided in \cref{sec:additional_nonlinear}.
Together, these results establish the isotropic Gaussian distribution as the optimal design to minimize the worst-case risk of a foundation model across downstream tasks.

\begin{figure}[t!]
    \centering
    \includegraphics[width=\linewidth]{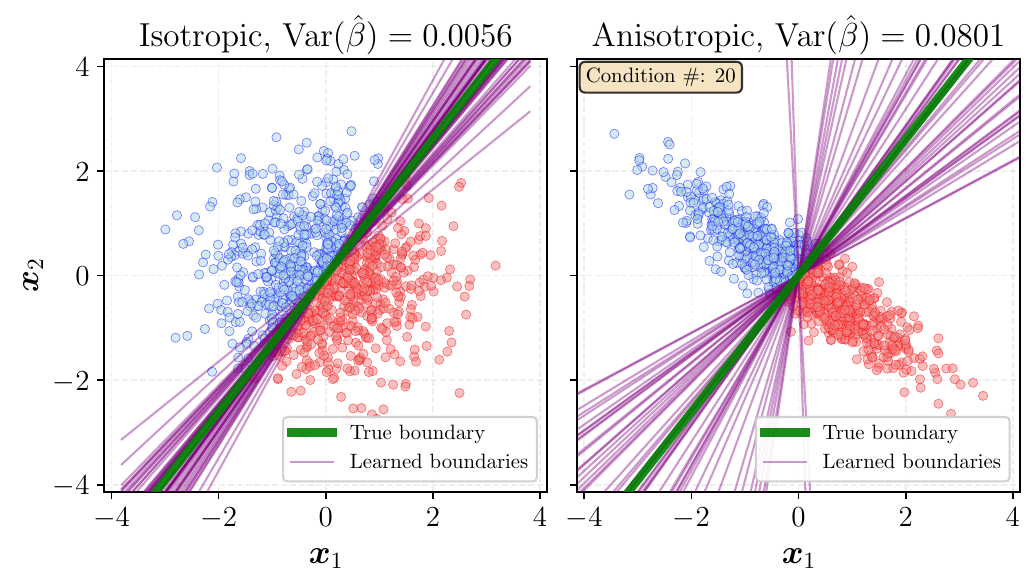}
    \caption{Illustration of \cref{thm:linear_probe_variance} showcasing how anisotropic ({\bf right}) embeddings lead to higher variance estimator compared to isotropic embeddings ({\bf left}). We sample $100$ training points for the $2$-class classification task and fit a logistic regression--repeating the process over numerous training set sample. Each sampling results in a decision boundary ({\bf purple}).}
    \label{fig:boundaries}
\end{figure}

\subsection{Geometric and Practical Insights}
\label{sec:le}

We now empirically validate that the isotropic Gaussian is optimal when 
no information about downstream tasks is available. We focus on linear 
probing (\cref{sec:linear_probing}), where all considered distributions 
have the same total variance.

When employing a linear probe, an anisotropic distribution increases 
both bias (with Tikhonov regularization) and variance. Examining bias 
first (\cref{thm:linear_probe_bias}), we present in \cref{fig:ols_bias} 
visualizations for both continuous regression and discrete classification 
tasks. We observe that the cosine similarity between estimated and 
ground-truth parameters equals 1 only for isotropic distributions, 
degrading for anisotropic cases regardless of sample size or 
regularization strength.
Regarding variance (\cref{thm:linear_probe_variance}), we show in 
\cref{fig:boundaries} that learned parameters vary significantly more 
across training sets when the covariance is anisotropic (right) compared 
to isotropic (left)—even when using logistic regression instead of OLS. 
\cref{fig:beta_distributions} further illustrates this effect, showing 
the distribution of learned $\beta$ parameters across different training 
samples for both cases. The anisotropic distribution clearly produces 
higher-variance estimators.

These theoretical and empirical results establish our design principle 
for LeJEPA: \emph{embeddings $f_{\vtheta}(\vx)$ should follow an isotropic Gaussian 
distribution to minimize worst-case risk across downstream tasks 
encountered post-training}. \Cref{sec:bcs} introduces a novel 
regularizer to achieve this distribution.

\section{SIGReg: Reliable Isotropic Gaussian Regularization in High-Dimension}
\label{sec:bcs}
Having established the isotropic Gaussian as the optimal embedding distribution 
(\cref{sec:gaussian}), we now introduce {\em Sketched Isotropic Gaussian 
Regularization} (SIGReg)--a distribution matching objective that is simultaneously 
(i) {\em differentiable}, (ii) {\em scalable}, (iii) {\em provable}, and 
(iv) {\em interpretable}. SIGReg builds on three key innovations. First, we 
formulate distribution matching as a statistical test under the null hypothesis 
$P_{\vtheta}=Q$ (\cref{sec:general_test}). Second, we identify a test that 
guarantees bounded gradients and curvature while maintaining linear complexity 
and efficient multi-GPU scaling (\cref{sec:CF_better}). Third, SIGReg bypasses 
the curse of dimensionality, eliminating collapsed shortcut solutions entirely 
(\cref{sec:dimension}).

\subsection{Hypothesis Testing as a Judge}
\label{sec:general_test}

Asking for $f_{\vtheta}(\vx)$'s distribution $P_{\vtheta}$ to match a target distribution $Q$ is typically done by creating various measures of distance or divergence, and estimating them in high-dimension. We propose a different starting point grounded in statistics. 
Consider the hypothesis testing framework \citep{fisher1928statistical,neyman1933ix} given by
\begin{align}
H_0: P_{\vtheta}=Q \quad \text{vs.} \quad H_1: P_{\vtheta}\neq Q,    \label{eq:null}
\end{align}
with $H_0$ being referred to as the {\em null hypothesis}. That is, we are asking in \cref{eq:null} if there is enough empirical evidence to reject the null. To answer that question, one (i) employs a {\em test-statistic}, i.e., a single scalar value summarizing the evidence from the empirical samples, (ii) determines a critical value $\tau_{\alpha}$ for the test-statistic based on the probability $\alpha$ of Type I error, i.e., of mistakenly rejecting a true null hypothesis, (iii) compares the test-statistic to the critical value $\tau_{\alpha}$; if the test-statistic exceeds $\tau_{\alpha}$, reject the null hypothesis. If the null is not rejected, we can only claim that {\em there is not sufficient empirical evidence against $P_{\vtheta}=Q$}.

As it stands, \cref{eq:null} remains impractical in large dimension as existing tests have at least quadratic complexity with the number of samples considered (more details in \cref{sec:multivariate_tests}). We thus propose to derive a sketching strategy by decomposing \cref{eq:null} into simpler univariate tests.
Denoting the push-forward distributions $P_\vtheta^{(\va)}\triangleq (\va^\top)_\# P_\vtheta$ and $Q^{(\va)}\triangleq (\va^\top)_\# Q$, we can define the following {\em directional} univariate test
\begin{align}
    H_0(\va):  P_\vtheta^{(\va)} = Q^{(\va)}\;\;\text{vs.} \;\; H_1(\va): P_\vtheta^{(\va)} \not = Q^{(\va)},   \label{eq:HUV}
\end{align}
for a given directional unit-norm vector $\va \in \mathcal{S}^{K-1}$. The corresponding \emph{directional test-statistic} of \cref{eq:HUV} is computed as $T(\{\va^\top f_{\vtheta}(\vx_n)\}_{n=1}^N)$. Examples of tests $T$ will be provided in the later \cref{sec:tests}. Repeating that process over a set of $M$ directions $\sA\triangleq \{ \va_1,\dots,\va_M\}$ and aggregating the individual values lead to the following \emph{global test-statistic}
\begin{align}
T_{\sA}(\{f_{\vtheta}(\vx_n)\}_{n=1}^N)\triangleq \max_{\va \in \sA} T(\{\va^\top f_{\vtheta}(\vx_n)\}_{n=1}^N).\label{eq:T_max}
\end{align}
We now provide a formal statement asserting the consistency of \cref{eq:T_max} to test the original multivariate null hypothesis from \cref{eq:null}. Our result leverages the well-known union-intersection principle \citep{roy1953heuristic}, and a slightly modified Cramér-Wold theorem. We denote by $\stackrel{d}{=}$ equality in distribution.

\begin{lemma}[label={thm:spherical_cramer}]{Hyperspherical Cramér-Wold}{}
Let $X,Y$ be $\mathbb{R}^d$-valued random vectors, then
\[
\langle \vu,X\rangle \stackrel{d}{=} \langle \vu,Y\rangle, \forall \vu \in \mathbb{S}^{d-1}\iff X \stackrel{d}{=} Y.
\]
Convergence in distribution also holds. (Proof in \cref{proof:spherical_cramer}.)
\end{lemma}

\begin{theorem}[label={thm:bcs}]{Sufficiency of directional tests}{}
\Cref{eq:T_max} is a valid statistical test for \cref{eq:HUV} as
\begin{align*}
P=Q&\implies\limsup_{n\to\infty}\Pr\left(T_{\sA}(\{f_{\vtheta}(\vx_n)\}_{n=1}^N) \ge \tau_\alpha\right)\le \alpha,\text{\bf (level)}
\\
P\neq Q& \implies\limsup_{n\to\infty}\Pr\left(T_{\sA}(\{f_{\vtheta}(\vx_n)\}_{n=1}^N) \ge \tau_\alpha\right)=1,\text{\bf (power)}
\end{align*}
(Proof in \cref{proof:bcs}.)
\end{theorem}

The assumptions required in the proof of \cref{thm:bcs} hold for classical consistent univariate tests $T$ such as the ones presented in the following \cref{sec:tests}.

\begin{figure}[t!]
    \begin{minipage}{0.49\linewidth}
\includegraphics[width=\linewidth]{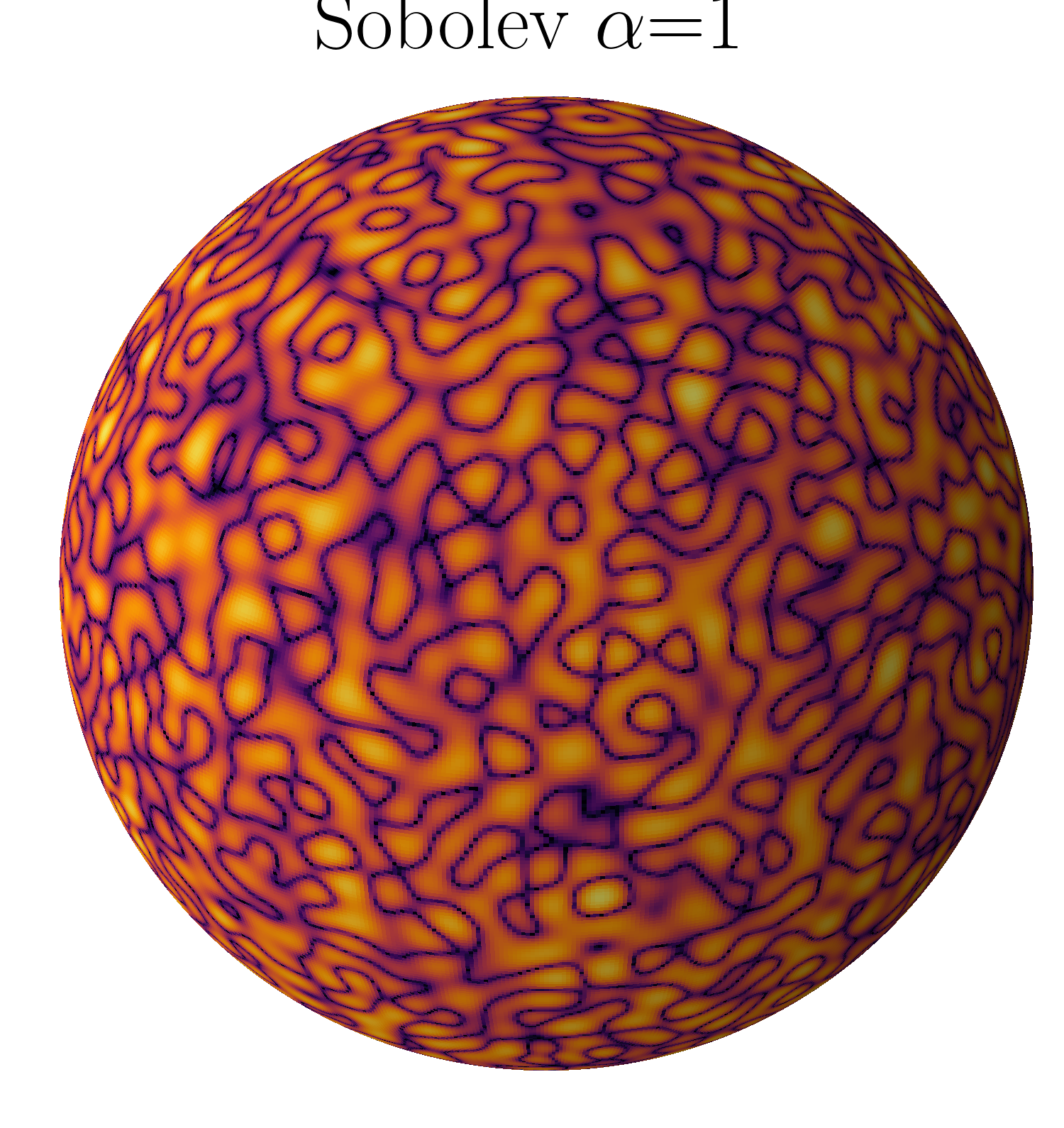}
    \end{minipage}
    \begin{minipage}{0.49\linewidth}
\includegraphics[width=\linewidth]{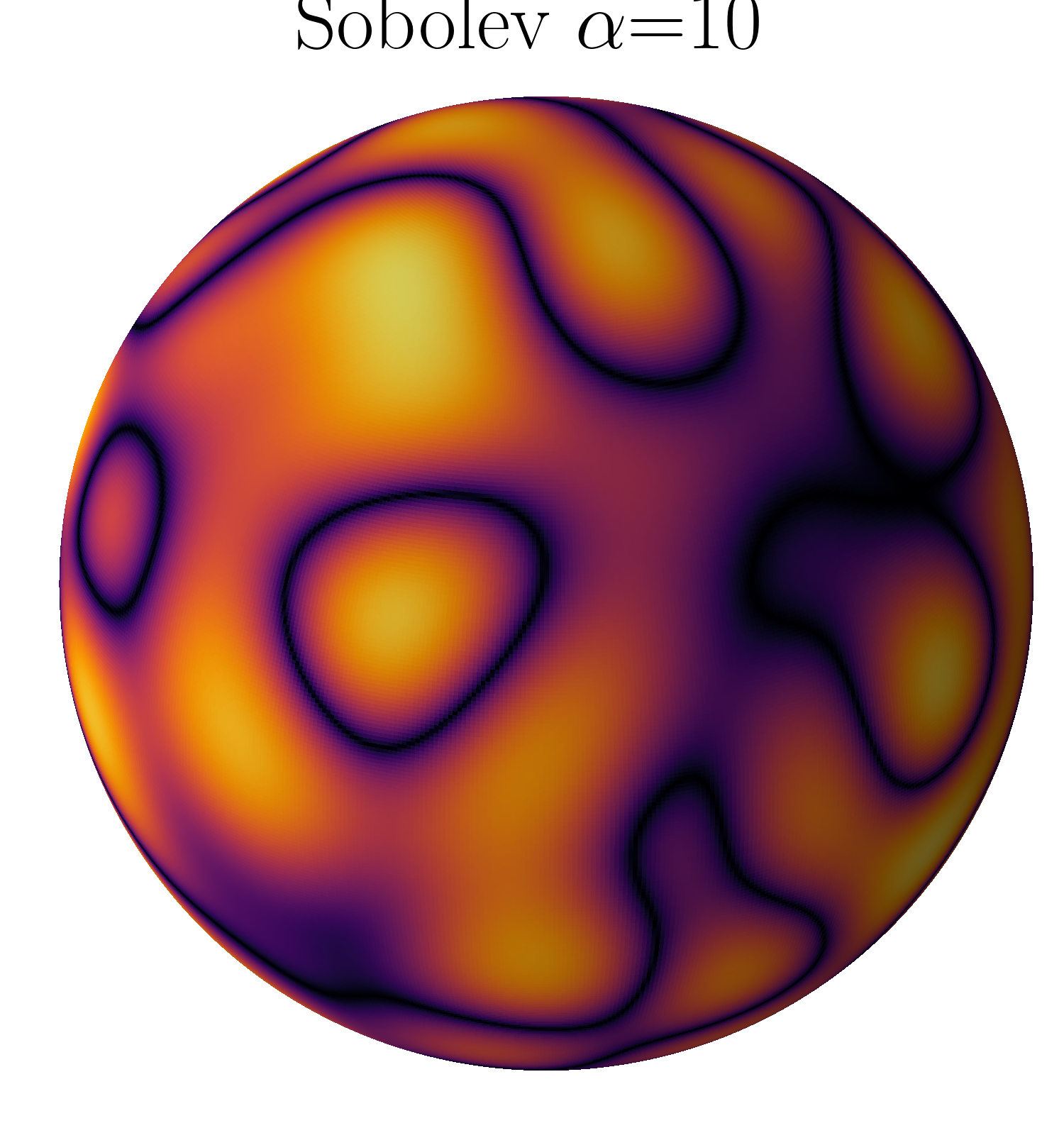}
    \end{minipage}
    \caption{Examples of distributions living on the surface of the sphere with varying Sobolev smoothness coefficients $\alpha$. As per \cref{thm:spherical_bounds}, the greater $\alpha$ is, the more global will be the impact of SIGReg for a given number of directions $M$. Practically, this represents the distribution of the encoder's output. Because the target density (isotropic Gaussian) is smooth, the $\alpha$ coeffcients of the embedding will quickly grow hereby making SIGReg (\cref{def:bcs}) immune to the curse of dimensionality.}
    \label{fig:sobolev}
\end{figure}

\begin{figure*}[t!]
    \centering
    \includegraphics[width=\linewidth]{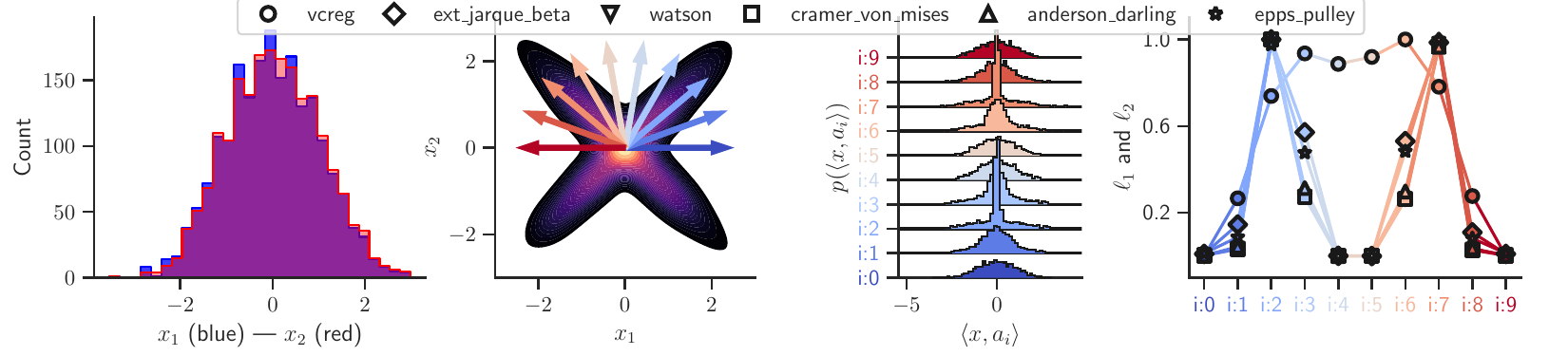}
    \caption{\small Constructed data density with ``X'' distribution whose marginals are standard Gaussian and whose covariance is identity ({\bf left densities}). Applying $M=10$ projections on the half circle directions produces $10$ univariate distributions that can be compared against a standard Gaussian ({\bf left}) using any preferred statistic from \cref{sec:tests}. The appropriate direction is able to capture the degenerate distribution of the data hereby creating a spike in the statistic value.}
    \label{fig:toy_2d}
\end{figure*}

\subsection{SIGReg: Sketching the Epps-Pulley Test is Stable and Scalable}
\label{sec:CF_better}
\label{sec:tests}

Our proposed regularizer--coined Sketched Isotropic Gaussian Regularization (SIGReg)--follows directly from \cref{thm:bcs} using any statistical test $T$ targeted towards the isotropic Gaussian, illustrated in \cref{fig:bcs_teaser,fig:toy_2d}, and formalized below.

\begin{definition}[label={def:bcs}]{SIGReg (PyTorch code in \cref{lst:epps-pulley-pytorch})}{}
SIGReg sketches a statistical test $T$ towards isotropic Gaussian
\begin{multline}
    {\rm SIGReg}_{T}(\sA,\{f_{\vtheta}(\vx_n)\}_{n=1}^{N})\triangleq\frac{1}{|\sA|}\sum_{\va\in\sA}T(\{\va^\top f_{\vtheta}(\vx_n)\}_{n=1}^{N}),\tag{SIGReg}\label{eq:BCS}
\end{multline}
where we recommend the Epps-Pulley test (\cref{sec:cf_tests}) for $T$.
\end{definition}

We replace the maximum over $\va\in\sA$ in \cref{thm:bcs} by an average in \eqref{eq:BCS} to avoid sparse gradient over the directions in $\sA$. We now delve on the choice of $T$ for which we compare well-known candidate tests in the field of statistics that are categorized into (i) moment based (\cref{sec:moment_tests}), (ii) CDF based (\cref{sec:cdf_tests}), and (iii) CF based (\cref{sec:cf_tests}) statistics--ultimately justifying our choice of the Epps-Pulley statistic.

\subsubsection{Moments are Unstable and Insufficient}
\label{sec:moment_tests}

The first family of statistics we consider are moment-based. Taking the standard Gaussian as an instanciation for the moments, we can define the Jarque-Bera \citep{jarque1980efficient} test that compares the third and fourth moments, i.e., skewness and kurtosis, as
\begin{align}
    {\rm JB}(\vu)\triangleq&\frac{N}{6}\left(\widehat{\rm skew}(\vu)^{2}+\left(\frac{\widehat{\rm kurt}(\vu)-3}{2}\right)^{2}\right)\tag{Jarque-Bera}\label{eq:jarque},
\end{align}
where $\widehat{\rm skew}$ is the skewness computed from the data as $\frac{\frac{1}{n} \sum_{i=1}^{n}\left(x_{i}-\hat{\mu}\right)^{3}}{\hat{\sigma}^{3}}$ and $\widehat{\rm kurt}$ is the kurtosis $\frac{\frac{1}{n} \sum_{i=1}^{n}\left(x_{i}-\hat{\mu}\right)^{4}}{\hat{\sigma}^4}$. Typically, the \eqref{eq:jarque} test is used to see if a density follows a Gaussian distribution of any mean and variance--hence it only looks at moments 3 and 4. In our case we aim for a standard Gaussian test and thus add the usual statistics on the first two moments, leading to the extended test
\begin{align}    
{\rm EJB}(\vu)\triangleq&\frac{N\hat{\mu}(\vu)^2}{\hat{\sigma}(\vu)^2}+\frac{(N-1)\left(\hat{\sigma}(\vu)^2-1\right)^2}{2}+{\rm JB}\tag{Extended Jarque-Bera}(\vu)\label{eq:ejarque}.
\end{align}
The \eqref{eq:ejarque} acts as a moment matching problem over the first four moments. Such moment matching methods have proven powerful not only for statistical tests but also as mean to learn parametric and nonparametric models of data.

{\bf The Stability and Identifiability Conundrum.}~We now explain why moment-based tests--albeit powerful--will not be suited for LeJEPA. The $k^{th}$ of a distribution $P$ is denoted as $m_k(P)$. The first observation is that well-behaved distributions abiding the Carleman's condition $
\sum_{k=1}^\infty m_{2k}(Q)^{-1/(2k)} = \infty$ \citep{carleman1926fonctions}, such as the Gaussian,  or for distributions with finite interval \citep{hausdorff1923momentprobleme} are uniquely determined by their moments. However, using a finite number of moments creates the following non-identifiability issue which well-known in statistics and often used as a motivation to use {\em all} moments \citep{lehmann2005testing}. 

\begin{theorem}[label={thm:moment_conendrum}]{Insufficiency of K Moments}{}
Minimizing the following objective with $c_k>0, \forall k$
\[
\sum_{k=1}^K c_k\left(m_k\left(P_{\vtheta}^{(\va)}\right)-m_k\left(Q^{(\va)}\right)\right)^2,
\]
for finite $K$ does not imply $P_{\vtheta}^{(\va)}=Q^{(\va)}$. (Proof in \cref{proof:moment_conendrum}.)
\end{theorem}
Hence \cref{thm:moment_conendrum} prescribes us with the guideline to employ as large $K$ as possible to remove collapsed shortcut solution by making sure our distribution matching is accurate. Yet, doing so leads to unstable gradient-based training due to
the gradient norm scaling as $O(k)$, and the variance of Monte Carlo gradient estimates growing as $O(k^2 m_{2(k-1)})$ 
for the $k$-th moment since $
\big\|\nabla_\theta m_k(P_{\vtheta}^{(\va)})\big\|=\|\mathbb{E}\big[k(\va^\top f_{\vtheta}(\vx))^{k-1}\va^\top J_{f_\vtheta}(\vx)\big]\|$, with $J_{f_\vtheta}(\vx)\in\mathbb{R}^{K\times P}$ the Jacobian matrix--hereby creating an impractical situation where training stability and identifiability can not be achieved simultaneously.

\begin{figure*}[t!]
    \centering
    \includegraphics[width=\linewidth]{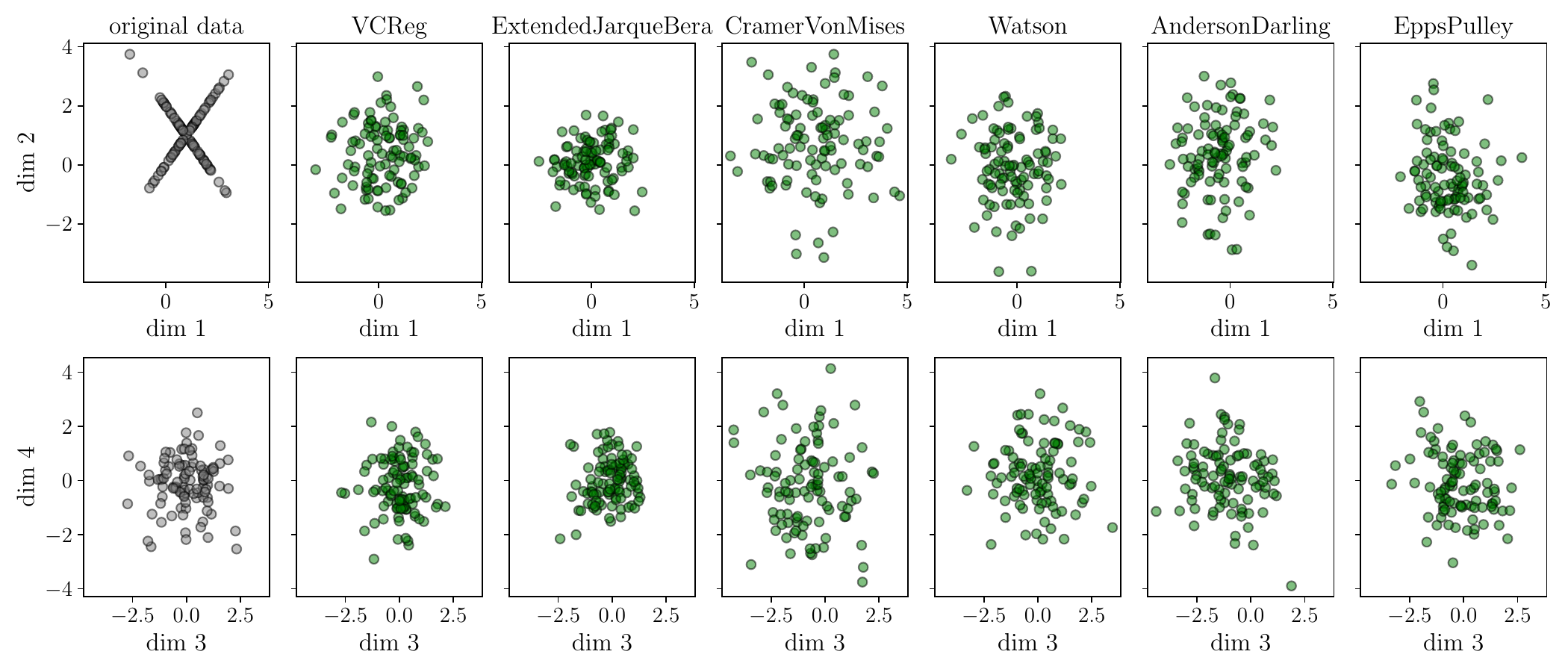}
    \caption{\small $N=100$ samples are drawn from a $1024$-dimensional standard Gaussian, and the first $2$ coordinates are altered to produce the ``X'' distribution from \cref{fig:toy_2d} ({\bf left-most column}). For each statistic ({\bf all other columns}), we perform gradient descent on the samples to minimize their value, at each iteration step with sample $M=10$ random directions to evaluate SIGReg (recall \cref{def:bcs}). We obtain that albeit this is a high-dimensional distribution with limited number of samples, SIGReg is able to capture the degenerate subspace and adapt the data accordingly to match an isotropic Gaussian distribution. Additional figures with varying dimensions and number of 1d projections are provided in \cref{fig:extra_nonparametric}.}
    \label{fig:nonparametric}
\end{figure*}

\subsubsection{Cumulative Density Functions are Impractical}
\label{sec:cdf_tests}

The second family of tests acts upon the CDF. Because those tests require sorting, let's denote the $k^{\rm th}$ order-statistics of $N$ samples by $x_{k:N}$. Two highly standard tests are quadratic Empirical Density Function statistics with different weighting known as Cramér-von Mises \citep{cramer1928composition,von1981probability} and Anderson Darling \citep{anderson1952asymptotic}, and given by
\begin{align}
    T_{w}&=N \int_{-\infty}^{\infty}\left(F_{N}(x)-F(x)\right)^{2} w(x) d F(x)\nonumber\\
    w(x)&=1,\tag{Cramér-von Mises}\label{eq:cramer}\\
    w(x)&=[F(x)(1-F(x))]^{-1}, \tag{Anderson-Darling}\label{eq:anderson_darling}
\end{align}
where $w(x)$ is a weighting function. Adding the $U^2$ statistics on top of \cref{eq:cramer} recovers the Watson test \citep{watson1961goodness}
\begin{equation}
    U^{2}=T_{w}-N\left(\bar{F}-\frac{1}{2}\right)^{2}.\tag{Watson} \label{eq:watson}
\end{equation}
We do not consider the Kolmogorov-Smirnov test \citep{kolmogorov1933} as it employs the $\ell_{\infty}$-norm instead of the $\ell_2$-norm hereby producing sparse gradients. Another common test is the Shapiro-Wilk test \citep{shapiro1965analysis} which we found to be unstable in practice--details are provided in \cref{sec:shapiro_wilk}.

{\bf Lack of Scalability and Differentiability.}~CDF-based tests require sorting that have been highly optimized, e.g., with the $\mathcal{O}(N \log (N))$ Quicksort algorithm \citep{quicksort} but that nonetheless breaks the embarrassingly parallel 
nature of SGD--especially on multi-GPU \citep{tanasic2013comparison,maltenberger2022evaluating} due to synchronization requirements. Moreover, these tests involve non-differentiable operations (sorting and 
order statistics), making them unsuitable for gradient-based optimization 
without relaxations \citep{cuturi2019differentiable,grover2019stochastic,petersen2022monotonic}. While there exists intricate sketching solutions \citep{dunning2019computing,masson2019ddsketch,dunning2021t}, each of those solutions introduce numerous additional hyper-parameters--going against our first motivation for LeJEPA.

\subsubsection{Characteristic Functions are Stable, Scalable and Identifiable}
\label{sec:cf_tests}

The third family of tests is concerned with Empirical Characteristic Functions (ECF) which are the Fourier transform of the density function. The Epps–Pulley test \citep{epps1983test} is one of the most popular test and simply compares in weighted $\ell_2$-norm the ECF of the data against a target CF
\begin{equation}
    EP = N \int_{-\infty}^{\infty} \left| \hat{\phi}_X(t) - \phi(t) \right|^2 w(t)  dt.\tag{Epps–Pulley}
    \label{eq:epps_pulley}
\end{equation}
The first crucial observation is that the ECF being defined as $\hat{\phi}_X(t) = \frac{1}{n} \sum_{j=1}^{n} e^{itX_j}$ is naturally differentiable and easily computed in distributed settings via efficient {\tt all\_reduce} operations, 
as the ECF is a simple average of complex exponentials. The weight function is typically Gaussian, such as $w(t) = e^{- t^2/\sigma^2}$ with $\sigma$ commonly set to $1$.

Other tests, e.g., based on the Entropy \citep{szekely2005new} are not considered here as they require numerous additional design choices for the univariate Entropy estimation \citep{silverman2018density,beirlant1997nonparametric}, e.g., using kernels \citep{joe1989estimation}, or M-estimators \citep{miller2003new}.

{\bf Epps-Pulley has bounded loss, gradient and curvature.}~We now consider the remaining two families of tests: moment-based and CF-based. First, recall that moments are polynomial in the data and with extreme growth rate for higher moment--assuming they even exist. Even for well-behaved distributions, raising values to a power of $k$ can quickly lead to exploding gradients. This comes in sharp contrast with the ECF which is always bounded and with bounded gradients for any input distribution for the projected samples $z_i = \va^\top f_\theta(\vx_n)$, $n=1,\ldots,N$.

\begin{theorem}[label={thm:ecf_stability}]{Stability of Epps-Pulley Test}{}
\eqref{eq:epps_pulley} 
satisfies for samples $z_1,\dots,z_N$
\[
\left|\frac{\partial EP(\mathbf{a})}{\partial z_i}\right| \le \frac{4\sigma^2}{N}, 
\quad
\left|\frac{\partial^2 EP(\mathbf{a})}{\partial z_i^2}\right| \le \frac{C\sqrt{\pi}\sigma^3}{2N},
\]
with constant $C$, and bandwidth $\sigma$. (Proof in \cref{proof:ecf_stability}.)
\end{theorem}

\begin{lstfloat}[t!]
\begin{lstlisting}[caption={SIGReg with Epps-Pulley statistic with DDP support and $\mathcal{O}(N)$ time and memory complexity. x is a (N, K) tensor, num\_slices is $|\sA|$ in \cref{def:bcs}, `global\_step` is used for sync. sampling across GPUs and can be omited for single-GPU training. An optimized implementation with caching is also provided in our official codebase, computation times provided in \cref{tab:times}.},label={lst:epps-pulley-pytorch}]
def SIGReg(x, global_step, num_slices=256):
    # slice sampling -- synced across devices --
    dev = dict(device=x.device)
    g = torch.Generator(**dev)
    g.manual_seed(global_step)
    proj_shape = (x.size(1), num_slices)
    A = torch.randn(proj_shape,  generator=g, **dev)
    A /= A.norm(p=2, dim=0)
    # -- Epps-Pulley stat. see Sec. 4.3 for alt. --
    # integration points
    t = torch.linspace(-5, 5, 17, **dev)
    # theoretical CF for N(0, 1) and Gauss. window
    exp_f = torch.exp(-0.5 * t**2)
    # empirical CF -- gathered across devices --
    x_t = (x @ A).unsqueeze(2) * t  # (N, M, T)
    ecf = (1j * x_t).exp().mean(0)
    ecf = all_reduce(ecf, op="AVG")
    # weighted L2 distance
    err = (ecf - exp_f).abs().square().mul(exp_f)
    N = x.size(0) * world_size
    T = torch.trapz(err, t, dim=1) * N
    return T
\end{lstlisting}
\end{lstfloat}

By the chain rule, \cref{thm:ecf_stability} directly gives $
\left\|\nabla_\theta EP(\mathbf{a})\right\| \le \frac{4\sigma^2}{N} \sum_{i=1}^N \|\mathbf{a}^\top \nabla_\theta f_\theta(\mathbf{x}_i)\|$, 
providing stable gradients. The limitations of moment-based and CDF-based tests coupled with \cref{thm:ecf_stability} justifies our choice of the \eqref{eq:epps_pulley}: (i) DDP-friendly and scalable, (ii) 
uniformly bounded gradients and curvature regardless of input distribution, and (iii) hyper-parameter free implementation. Lastly, we highlight that {\em our implementation has a linear memory and computational complexity of $\mathcal{O}(N)$, with $N$ the minibatch size}. The implementation of SIGReg using that statistical test is provided in \cref{lst:epps-pulley-pytorch}, along with computation times of the forward-backward pass in \cref{tab:times}.

As a last step before introducing LeJEPA, we ought to study the requirements on the number of directions ($|\sA|$) for \eqref{def:bcs} to be effective in high-dimension.

\subsection{How SIGReg Beats the Curse of Dimensionality}
\label{sec:dimension}

This last section seeks to characterize how many slices in $\sA$ one must sample for \eqref{eq:BCS} to be an effective statistical test. That design is crucial if we hope for LeJEPA to successfully converge towards  isotropic Gaussian embeddings.

\subsubsection*{Smoothness Beats the Curse of Dimensionality}

Our first argument arguing for a favorable scaling of $|\sA|$ with the embedding dimension $K$ relies on the smoothness of $P_{\vtheta}$ as measured by its Sobolev regularity $\alpha$ \citep{adams2003sobolev}. We formalize below a bound on the directional test from \cref{eq:HUV} over all possible directions $\va$ when the test statistic is minimized over $|\sA|=M$ directions. While we provide bounds on the expected discrepancy over random directions $\va$ 
when the EP test is satisfied (equals zero) on a finite set of directions, the provided proof includes the case of moment-based and CDF-based tests as well.

\begin{figure}[t!]
    \centering
    \includegraphics[width=\linewidth]{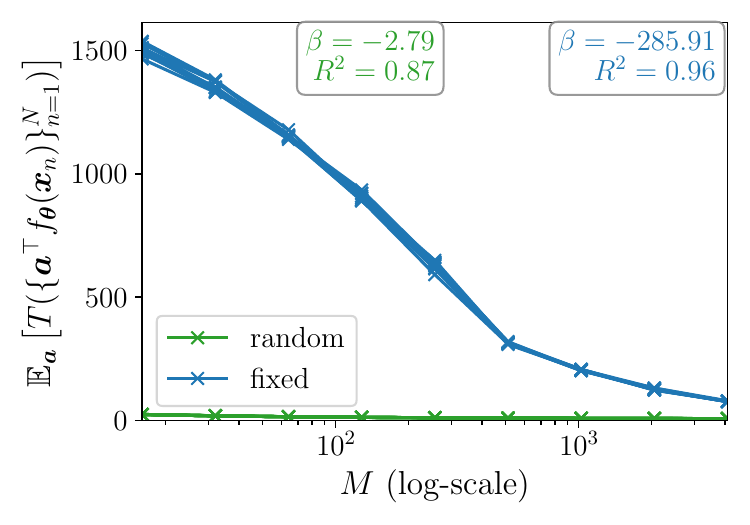}
    \caption{Expected directional statistic at the end of training ({\bf y-axis}) for varying $M$ (number of directions used at each training step, {\bf x-axis}). The $M$ directions are either resampled ({\bf green}) or kept fixed ({\bf blue}) at each training step. While for fixed directions we benefit from \cref{thm:spherical_bounds} bound where increasing $M$ reduces the overall expected loss, being able to resample at every step provides significant coverage boost for free.}
    \label{fig:resampling}
\end{figure}

\begin{theorem}[label={thm:spherical_bounds}]{Unified Error Bounds}{}
Let $p_{\vtheta}\in H^\alpha(\mathbb{R}^K)$, $\va \sim \mathcal{U}(\mathcal{S}^{K-1})$, and \eqref{eq:epps_pulley}$=0$, i.e., $P_\theta^{(\mathbf{a})} = Q^{(\mathbf{a})}, \forall \va \in \sA$, then
\begin{multline*}
        \mathbb{E}_{\va} \left[ \int_{\mathbb{R}} \left| \varphi_a(t) - \varphi_{\mathcal{N}}(t) \right|^2 dt \right]
    \leq
    C(K, \alpha) |\sA|^{-2\alpha/(K-1)} \\\times\int_0^\infty \left\| \varphi_{\cdot}(r) - \varphi_{\mathcal{N}}(r) \right\|_{H^\alpha(\mathcal{S}^{K-1})}^2 dr,
    \end{multline*}
    (Proof in \cref{proof:spherical_bounds}.)
\end{theorem}

As $|\sA| \to \infty$, the bound decays as $|\sA|^{-2\alpha/(K-1)}$, showing that 
$|\sA| = O(K)$ directions suffice for $\epsilon$-approximation when $\alpha$ is large. Some examples of embedding densities with varying $\alpha$ are provided in \cref{fig:sobolev}. The following statement characterizes how the $M$ directions actually constrain the entire space as a function of $\alpha$.
The constant $C(K, \alpha) = \frac{2^{2\alpha} \pi^{(K-1)/2} \Gamma\left(\alpha + \frac{K-1}{2}\right)}{(K-1) \Gamma(\alpha) \Gamma\left(\frac{K-1}{2}\right)}$ is visualized in \cref{fig:bound_example} (left) depicting how $\alpha$ and $|\sA|$ interact. In words, we obtain that thanks to the natural smoothness of DN--either stemming from the architecture or the implicit and explicit regularizers used during training--applying SIGReg on $|\sA|$ directions can be sufficient to tightly constrain the entire space. We note that considering the worst case over $\va$ or using low-discrepancy sequences for $\va$ does not impact the asymptotic bounds, details provided in \cref{sec:low_discrepancy}.

\subsubsection*{SGD Beats the Curse of Dimensionality}

Our second argument leverages the iterative nature of DN training. 
Although we may use only $|\sA|$ to be a few hundreds, the cumulative 
number of sampled directions grows linearly with training time. This resampling effect 
(illustrated in \cref{fig:resampling}, bottom) enables rapid convergence. Even small $|\sA|$ achieves tight distributional matching compared to keeping the set $\sA$ fixed throughout minibatches (recall \cref{thm:spherical_bounds}). Our experiments show that even with 
$|\sA|$ as low as $16$ can easily outperform a fixed set with $|\sA|$ of order of thousands thanks to the compounding effect of resampling at each minibatch.

\subsubsection*{Empirical Validation on Synthetic Data}
\label{sec:exp_tests}

We conclude this section with a controlled experiment applying \eqref{eq:BCS} with gradient-based training to produce isotropic embeddings. In this setup, we directly consider embeddings $\mZ$ which we will differentiate and optimized to minimize \eqref{eq:BCS}. By directly optimizing the embeddings we are able to observe the impact of the loss without any possible constraint and regularization that would come from the architecture.
We sample $N$ i.i.d. samples $\vx_{n}$ in a $D$-dimensional space. This sampling is based on an isotropic Gaussian distribution--but the first two dimensions are again set to the adversarial ``X'' shape. That is, among the $D$ dimensions, only two must be transformed as all the other ones already obey the isotropic Gaussian target distribution. We then make the samples $\vx_{n}$ differentiable and optimize then to minimize the value of the different statistical tests compute on $M$ random $M$ random directions. Those directions are resampled after each gradient step--which follows the procedure we will employ in LeJEPA. We present the results in \cref{fig:nonparametric} demonstrating that even in challenging case, i.e., $D=512$ and $M=16$, SIGReg is able to detect the two degenerate dimensions and unfold them back to how they should look like under the target distribution.

\section{LeJEPA: Stable and Scalable Implementation}
\label{sec:lejepa}

Having established that isotropic Gaussians are the optimal embedding distribution 
for foundation models (\cref{sec:gaussian}) and introduced SIGReg to achieve this 
distribution (\cref{def:bcs}), we now present the complete LeJEPA framework. We 
first evaluate candidate statistical tests (\cref{sec:moment_tests,sec:cdf_tests}) 
and identify characteristic function-based tests as optimal for gradient-based 
training (\cref{sec:cf_tests}). The full LeJEPA implementation follows in 
\cref{sec:lejepa_implementation}.

\subsection{LeJEPA: SIGReg + Prediction Loss}
\label{sec:lejepa_implementation}

\begin{lstfloat}[t!]
\begin{lstlisting}[caption={LeJEPA implementation--works out-of-the-box on any dataset, with DDP, with any backbone, e.g., torchvision or timm. For non-ViT architectures (e.g., ResNet), set global\_views = all\_views. We use bs for the minibatch size, SIGReg is from \cref{lst:epps-pulley-pytorch}.},label={code:lejepa}]
def LeJEPA(global_views, all_views, lambd):
    """global_views and all_views are lists of tensors, lambd is a scalar"""

    # embedding of global views
    g_emb = forward(torch.cat(glob_views))
    # embedding of local views
    # if resnet: skip with a_emb=g_emb
    a_emb = forward(torch.cat(all_views))
    
    # LeJEPA loss
    centers = g_emb.view(-1, bs, K).mean(0)
    a_emb = a_emb.view(-1, bs, K)
    sim = (centers - a_emb).square().mean()
    sigreg = mean(SIGReg(emb, global_step) for emb in a_emb)
    return (1-lambd)*sim + lambd*sigreg
\end{lstlisting}
\end{lstfloat}

We now discuss the implementation of LeJEPA starting with SIGReg and followed by the prediction and total losses.

{\bf The SIGReg Loss.}~We chose \eqref{eq:epps_pulley} for its provable boundedness (\cref{thm:ecf_stability}) and its scalability. Its implementation follows exactly the equation except for the integrate which is estimated using a quadrature approximation. We find that the simple trapezoidal quadrature rule is sufficient even with as few knots as $17$, as ablated in \cref{fig:quadrature}. In particular, we leverage the symmetry of the integrand to double the number of knots for free, see the official code. On the other hand, the use of minibatches introduces a bias vanishing at rate $\mathcal{O}(1/N)$, as formalized below.

\begin{theorem}[label={thm:gradient_bias}]{Vanishing gradient bias}{}
The expectation of \eqref{eq:epps_pulley} satisfies
\[
\mathbb{E}\left[\widehat{L}_n(\theta)\right]
=
L(\theta)+\frac{1}{N}\int_{\mathbb{R}} w_s(t)\big(1-|\varphi_P(t)|^2\big)dt,
\]
therefore both the loss and its derivative have a bias of order $O(1/n)$. (Proof in \cref{proof:gradient_bias}.)
\end{theorem}

Hence, the gradients we obtain from using \eqref{eq:epps_pulley} are biased by an explicit $\mathcal{O}(1/N)$ term. We found this bias to be minimal and not a concern even for minibatches as small as 16. Unbiased alternatives include using U-statistic debiasing of $|\phi_\theta|^2$ or sample splitting, which we do not explore in this study. Our final implementation of the SIGReg term with Epps-Pulley statistic is provided in \cref{lst:epps-pulley-pytorch}.

{\bf The Prediction Loss.}~To standardize notations, we adopt the DINO \citep{caron2021emerging} setup of generating $V_{g}$ global views and $V_{l}$ local views, leading to a total of $V=V_{g}+V_{l}$ views. We set the first $1,\dots,V_{g}$ indices of each $\vz_{n,v}$ as the global views. For the cases without local views, simply set $V_{l}=0$. The prediction loss is then given by having all views predict the global views as
\begin{align}
    \mathcal{L}_{\rm pred}(\{\vz_{n,v}\}_{v=1}^{V})=&\frac{1}{V_g}\sum_{v=1}^{V_g}\frac{1}{V}\sum_{v'=1}^{V}\left\| \vz_{n,v} - \vz_{n,v'} \right\|_2^2\label{eq:pred1}\\
    =&\frac{1}{V}\sum_{v'=1}^{V}\left\| \frac{1}{V_{\rm g}}\sum_{v=1}^{V_{\rm g}}\vz_{n,v} - \vz_{n,v'} \right\|_2^2\label{eq:pred2}\\
    \triangleq& \frac{1}{V}\sum_{v'=1}^{V}\left\| \bm{\mu}_{n} - \vz_{n,v'} \right\|_2^2,\label{eq:pred_as}
\end{align}
where we denote $\bm{\mu}_n \triangleq \frac{1}{V_g}\sum_{v=1}^{V_g}\vz_{n,v}$, the \cref{eq:pred1} to \cref{eq:pred2} derivations are detailed in \cref{proof:pred}.

{\bf LeJEPA Loss.}~The final total loss simply combines the above prediction loss along with SIGReg on each views as per
\begin{multline}
    \mathcal{L}_{\rm LeJEPA}(\{\vx_{n,v}\}_{n,v=1}^{B,V})=\frac{\lambda}{V}\sum_{v=1}^{V}{\rm SIGReg}(\{\{\vz_{n,v}\}_{n=1}^{B}\})\\
    +\frac{1-\lambda}{B}\sum_{n=1}^{B}\mathcal{L}^{(V_{\rm g})}_{\rm pred}(\{\vz_{n,v}\}_{v=1}^{V}).\tag{LeJEPA}\label{eq:lejepa}
\end{multline}

We present \eqref{eq:lejepa}'s implementation in \cref{code:lejepa}. Altogether, the entire implementation--besides the usual model definitions, optimizers, and data loaders--only takes a few dozens lines in PyTorch (\cref{lst:epps-pulley-pytorch,code:lejepa}). The absence of prototypes, stop-gradients, and teacher-student networks 
makes \eqref{eq:lejepa} appealing as it only contains one hyperparameter, $\lambda$, balancing the trade-off between the prediction and isotropic Gaussian terms.

\subsection{Relation to Prior Work}

Prior to presenting our experiments (\cref{sec:experiments}), we conclude  by discussing how our proposed LeJEPA and SIGReg objective relate to existing frameworks in the literature.

While there is no existing solution employing such slicing and distribution matching for JEPAs, there exists similar pipelines for generative models and optimal transport. Notably, the Sliced Score Matching \citep{song2020sliced} proposes to leverage univariate slicing of the space to ease the estimation of a density for generative models. In a similar vein, the sliced Wasserstein distance \citep{bonneel2015sliced,nguyen2023energy} uses such strategy to speed up and improve optimal transport. Furthermore, when the integral of the \eqref{eq:epps_pulley} test is computed exactly, as opposed to our quadrature, each slice loss value recovers the kernel MMD \citep{sriperumbudur2010hilbert,gretton2012kernel,chwialkowski2016kernel} measuring the distance between two distributions--albeit with a quadratic complexity. Lastly, it is possible to recover some existing SSL frameworks in the limit by employing LeJEPA with a particular test--instead of the preferred \eqref{eq:epps_pulley}. For example,  
Setting $T(\{x_n\}_{n=1}^{B})={\rm mean}(\{x_n\}_{n=1}^{B})^2+({\rm std}(\{x_n\}_{n=1}^{B})-1)^2$ and using that $T$ with SIGReg in LeJEPA recovers the VICReg SSL method in the limit of large number of slices. In fact, SIGReg will enforce in expectation that  $\mathbb{E}[\mathbf{Z}] = \mathbf{0}$ and $\mathrm{Cov}(\mathbf{Z}) = \mathbf{I}_d$, where $\mathbf{I}_d$ denotes the $d \times d$ identity matrix--derivations provided in \cref{proof:vcreg}. And since our invariance term is simply the $\ell_2$ distance between the views' embeddings, LeJEPA recovers VICReg for this degenerate statistical test. Based on \cref{thm:moment_conendrum}, we however strongly advocate against such a setting as it would lead to shortcut solutions--a phenomenon already observed in VICReg.

\section{LeJEPA: Empirical Validation}
\label{sec:experiments}

We now use the LeJEPA implementation described in \cref{sec:lejepa_implementation} 
to demonstrate its effectiveness through comprehensive experiments. We show that 
LeJEPA: (i) trains reliably across diverse architectures and datasets 
(\cref{sec:hparams}), (ii) provides an informative training loss for model 
selection (\cref{sec:loss_corr}), (iii) outperforms frontier vision models on 
small-scale in-domain pretraining (\cref{sec:galaxy}), (iv) scales successfully 
to nearly 1 billion parameters on 
ImageNet-1k (\cref{sec:scale}), and (v) learns rich semantic segmentation 
features without explicit supervision.

\subsection{LeJEPA's Stability Across Hyper-Parameters and Architectures}
\label{sec:hparams}

We now demonstrate LeJEPA's stability across hyperparameters, architectures, 
and experimental setups. Additional cross-domain stability results are 
presented in \cref{sec:galaxy}.

\begin{figure}[t!]
    \centering
    \includegraphics[width=\linewidth]{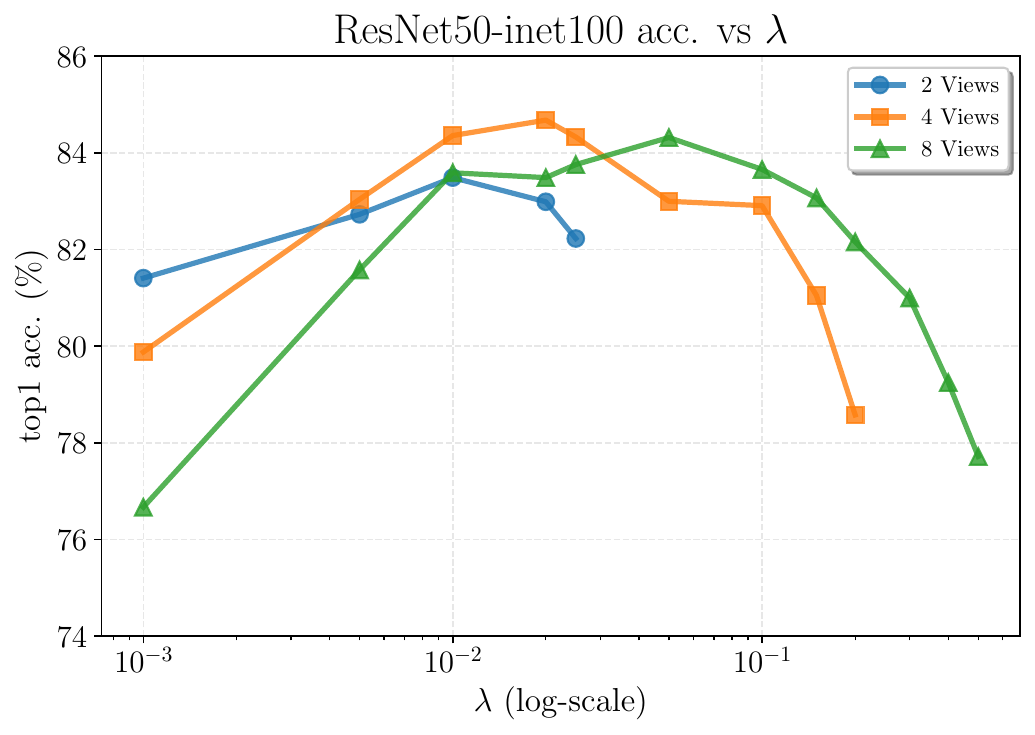}
    \caption{Inet100 with $400$ pretraining epochs and resnet50 backbone. We depict linear probe performances as a function of $\lambda$ and the number of views $V$ (recall \eqref{eq:lejepa}). We observe that performances are stable over $\lambda$--with {\bf peak performance obtain by slightly adjust $\lambda$ proportionally to the number of views}. The corresponding performance values are  provided in \cref{tab:lambda_perf}.}
    \label{fig:lambda_views}
\end{figure}
\begin{table}[t!]
\centering
\caption{ViT/Large-14, on inet1k pretraining for 100 epochs and evaluated with frozen backbone linear probing (top1 accuracy, \%).{\bf LeJEPA's performance is stable across all its hyperparameters} and while some may slightly improve performance, e.g., the number of slices $|\sA|$ and the projector sizes, none of the choices lead to a catastrophic collapse.}

\label{tab:ablations}
\vspace{0.5em}
\textbf{(a) \eqref{eq:epps_pulley} parameters}
\vspace{0.5em}
\begin{tabular}{lllll}
\toprule
integration & num\_slices & \multicolumn{3}{c}{config/bstat\_n\_points} \\
\cmidrule(lr){3-5}
 &  & 5 & 17 & 41 \\
\midrule
$[-1,1]$ & 512 & 71.82 & 72.13 & 72.04 \\
  & 2048 & 72.88 & 72.30 & 72.69 \\
$[-3,3]$ & 512 & 73.95 & 74.16 & 74.04 \\
  & 2048 & 75.02 & 74.68 & 74.77 \\
$[-5,5]$ & 512 & 73.71 & 74.21 & 74.15 \\
  & 2048 & 74.50 & 74.80 & 74.77 \\
\bottomrule
\end{tabular}
\\
\vspace{0.5em}
\textbf{(b) Number of local/global views}
\vspace{0.5em}
\begin{tabular}{llll}
\toprule
\# global\_views ($V_{\rm g}$) & 1 & 2 & 4 \\
\# views ($V=V_{\rm g}+V_{\rm l}$) &  &  &  \\
\midrule
4 & 53.06 & 72.26 & – \\
6 & 58.65 & 73.07 & 73.68 \\
8 & 64.46 & 74.24 & 73.94 \\
10 & 68.97 & 74.06 & 75.08 \\
\bottomrule
\end{tabular}
\\
\vspace{0.5em}
\textbf{(c) Mini-batch size}
\vspace{0.5em}
\begin{tabular}{lllll}
\toprule
batch\_size & 128 & 256 & 512 & 1024 \\
&  &  &  &  \\
\midrule
& 72.20 & 74.15 & 74.72 & 74.07 \\
\bottomrule
\end{tabular}
\\
\vspace{0.5em}
\textbf{(d) Embedding/Projector dim.}
\vspace{0.5em}
\begin{tabular}{lllll}
\toprule
num\_slices & \multicolumn{2}{c}{1024} & \multicolumn{2}{c}{4096} \\
emb. dim. & 512 & 2048 & 512 & 2048 \\
proj. dim. &  &  &  &  \\
\midrule
64 & 75.29 & 75.32 & 75.50 & 75.65 \\
128 & 74.77 & 75.09 & 75.26 & 75.47 \\
256 & 74.56 & 74.66 & 75.08 & 75.02 \\
512 & 73.94 & 74.11 & 74.81 & 74.65 \\
1024 & 73.65 & 73.94 & 74.71 & 74.79 \\
\bottomrule
\end{tabular}
\\
\vspace{0.5em}
\textbf{(e) Register tokens}
\vspace{0.5em}
\begin{tabular}{llllll}
\toprule
reg\_tokens & 0 & 1 & 2 & 4 & 8 \\
num\_slices &  &  &  &  &  \\
\midrule
1024 & 75.14 & 75.18 & 75.08 & 75.34 & 75.23 \\
4096 & 75.61 & 75.58 & 75.67 & 75.63 & 75.84 \\
\bottomrule
\end{tabular}
\end{table}

\begin{figure*}[t!]
    \centering
    \includegraphics[width=\linewidth]{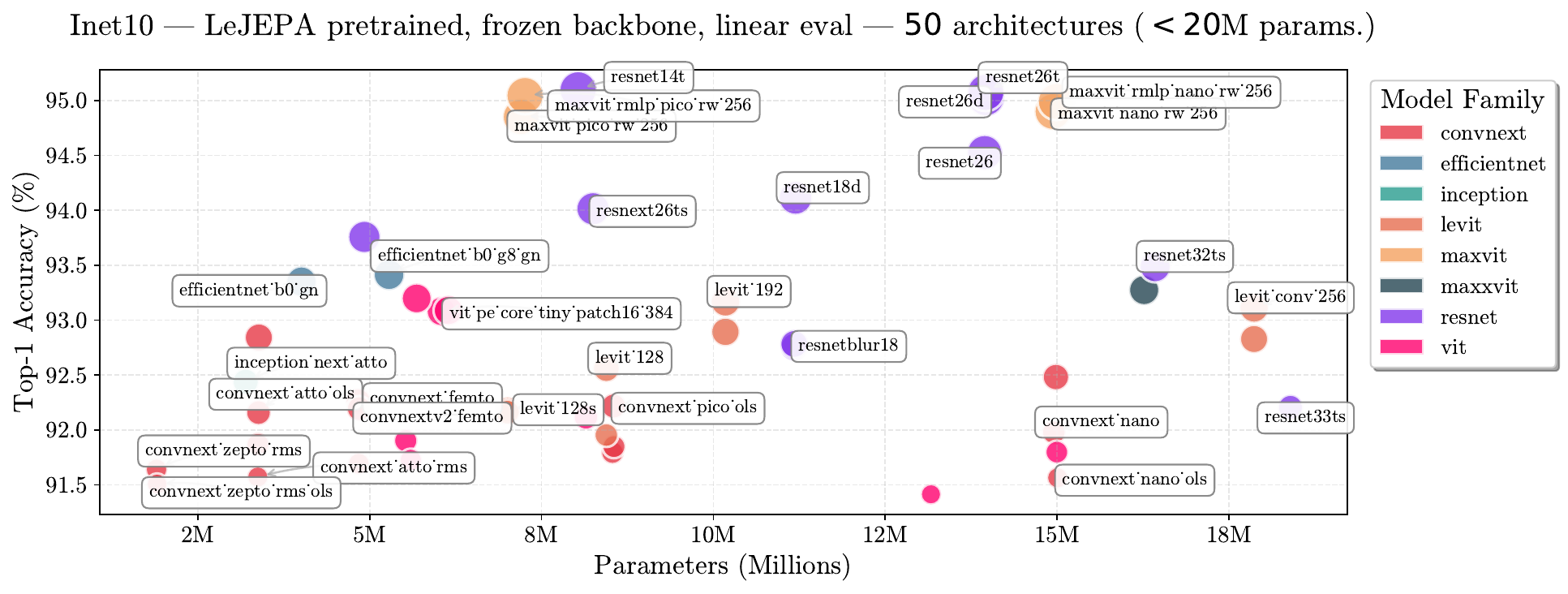}
    \caption{INet10 pretraining and frozen backbone linear evaluation across $50$ timm models using LeJEPA out of the box. We cross-validate the learning rate and weight-decay. While there is a small variation between the best and worst performing model, we clearly see that {\bf across $50$ models spanning $8$ families, LeJEPA is able to produce non-trivial representations able to solve the downstream task at SOTA levels}.}
    \label{fig:inet10_backbones}
\end{figure*}

{\bf Stability across standard hyperparameters.}~We begin by evaluating LeJEPA 
on ImageNet-100 and ImageNet-1K. On ImageNet-100, we train a ResNet-50 and 
vary the number of views and the loss weighting $\lambda$ (\cref{fig:lambda_views}). 
Performance remains stable across both dimensions, leading us to recommend 
$\lambda=0.05$ as a robust default.
On ImageNet-1K, we train a ViT-Large/14 and explore batch size, as well as 
the number of global ($V_{\rm g}$) and local ($V_{\rm l}$) views 
(\cref{tab:ablations}b). We find that the configuration commonly used in 
prior work ($V_{\rm g}=2, V_{\rm l}=8$) transfers well to LeJEPA. Notably, 
LeJEPA achieves competitive performance with batch sizes as small as 128 
on ImageNet-1K (\cref{tab:ablations}c), suggesting reduced memory requirements compared to existing 
methods. {\em We thus recommend to use  $\lambda=0.05$, $V_{\rm g}=2$, $V_{\rm l}=8$, and 
batch size $\geq 128$ as starting points}.

{\bf Stability across Epps-Pulley hyperparameters.}~We next examine 
hyperparameters specific to LeJEPA: the number of slices $|\mathcal{A}|$ 
in SIGReg, the integration domain for the Epps-Pulley test \eqref{eq:epps_pulley}, and the number of quadrature points for numerical integration. 
\cref{tab:ablations}a shows ablations on ImageNet-1K with ViT-Large/14.
Both the integration domain and number of quadrature points have negligible 
impact on performance. This is expected: since the characteristic function 
is accurate at zero, the moments of the distribution are well-characterized 
even with a modest integration range. The number of slices $|\mathcal{A}|$ 
has a modest effect—while more slices slightly improve performance, even 
512 slices yield competitive results. {\em We thus recommend to use 17 integration points, an integration domain of 
$[-5,5]$, and 1024 slices as starting points}.

{\bf Stability across architectures.}~A key advantage of LeJEPA over recent 
methods (e.g., IJEPA, DINOv2) is its architecture-agnostic design. While 
most modern self-supervised methods are tailored to Vision Transformers, 
LeJEPA works across diverse architecture families without modification.
To validate this claim, we pretrain approximately 50 architectures from 
8 different families on ImageNet-10, selecting all models in the timm 
library with fewer than 20M parameters. All models are able to learn high-quality representations reaching between 91.5\% to 95\% top 1 accuracy with frozen backbone linear probing. It seems that models performing well in supervised learning setups are also the ones to favor for LeJEPA, such as resnets and ViTs. {\em We thus recommend to use standard architectures such as ResNets and ViTs over specialized models like EfficientNet as stating point.}

{\bf Removal of popular heuristics.}~In addition to providing reliable performance across models and datasets, LeJEPA's provable construction enables us to {\em remove} many heuristics traditionally used to prevent collapse. First, prior work has shown both empirically and theoretically that predictors in image JEPA (without asymmetric information) and teacher-student architectures serve primarily to prevent collapse \citep{grill2020bootstrap,jing2021understanding,tian2021understanding,caron2021emerging,chen2021empirical}. Removing these components produces {\em collapsed} encoders, i.e., with performances at chance-level. Thanks to LeJEPA's SIGReg loss, we can remove both the predictor and teacher-student architecture without suffering from collapse, as shown in \cref{tab:proj_pred}. While a teacher-student configuration does provide a small performance boost for ViT models—consistent with observations in supervised learning via Stochastic Weight Averaging \citep{izmailov2019averagingweightsleadswider}—it is not necessary to prevent collapse. In our setup, we apply SWA on the encoder producing $\mu$ in \cref{eq:pred2}. Second, recent work demonstrated that register tokens are needed to prevent training instabilities in vision models \citep{oquab2023dinov2,simeoni2025dinov3,darcet2023vision}. We show in \cref{tab:ablations} that such instabilities likely stem from poorly conditioned training objectives. In contrast, LeJEPA {\em does not} require register tokens and achieves stable performance with or without them. {\em We thus recommend training without a predictor or register tokens, and optionally applying SWA with ViTs for a possible performance gain.}

\begin{detail}[label={def:detail}]{}{}
We strive for {\bf simplicity} and thus adopt a unified pretraining pipeline. The following parameters apply to {\em all} experiments and figures unless stated otherwise in the corresponding caption and come from \cref{sec:hparams}:
\begin{itemize}[nosep,before=\vspace{-0.1em},after=\vspace{-0.4em}]
    \item LeJEPA's implementation is given in \cref{code:lejepa} with hyperparameter $\lambda$ 
    \item All backbones are from ${\rm timm}$ and all optimizers/schedulers are from ${\rm PyTorch}$ without modifications
    \item We employ eight views ($V=8$) containing two global views ($V_{\rm g}=2$) with resolution 224x224 and 96x96 for the local views
    \item AdamW optimizer with ${\rm lr} \in \{5e-3,5e-4\}$ and ${\rm wd} \in \{1e-1,1e-2,1e-5\}$--no scheduler on weight-decay, standard linear warm-up cosine-annealing for ${\rm lr}$
    \end{itemize}
\end{detail}

\subsection{LeJEPA's Training Loss is Informative of Downstream Performance}
\label{sec:loss_corr}

\begin{figure*}[t!]
    \includegraphics[width=0.33\linewidth]{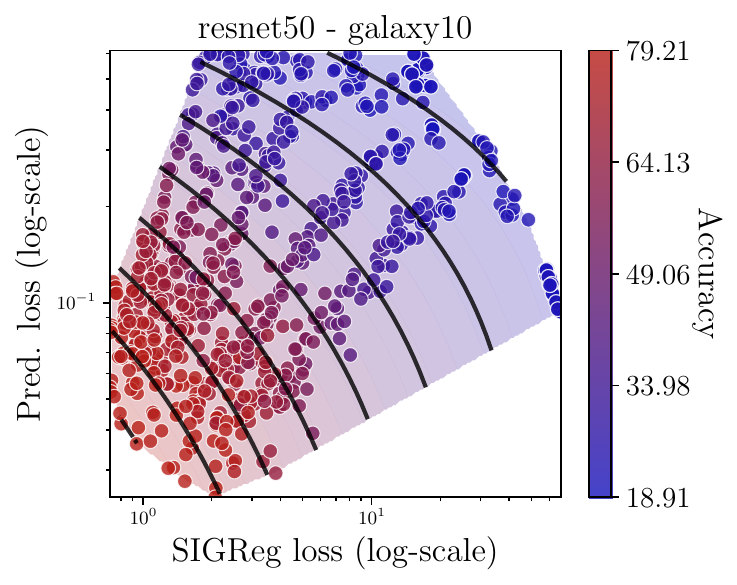}
    \includegraphics[width=0.33\linewidth]{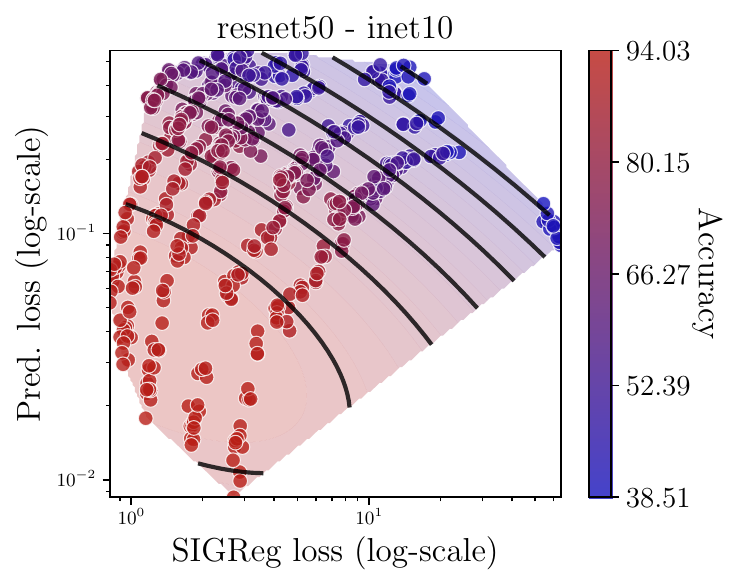}
    \includegraphics[width=0.33\linewidth]{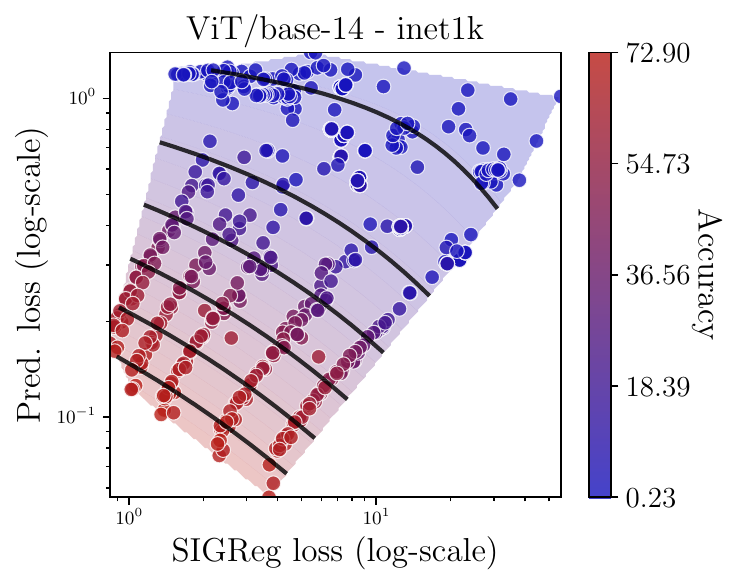}
    \caption{({\rm SIGReg, prediction loss)} $2d$-plane with downstream task accuracy shown with colors from {\bf blue} (low) to {\bf red} (high). We clearly observe that within this plane, {\bf there exists trade-off fronts between the two terms of LeJEPA producing similar downstream performance} corresponding to different values of $\lambda$. Yet, those fronts are linear and pointed towards the lower left corner, i.e., LeJEPA's training loss informs of downstream test performance across models and datasets ({\bf columns}). Additional models and datasets provided in \cref{fig:extra_heatmap_loss}.}
    \label{fig:heatmap_loss}
\end{figure*}

\begin{figure}[t!]
\includegraphics[width=\linewidth]{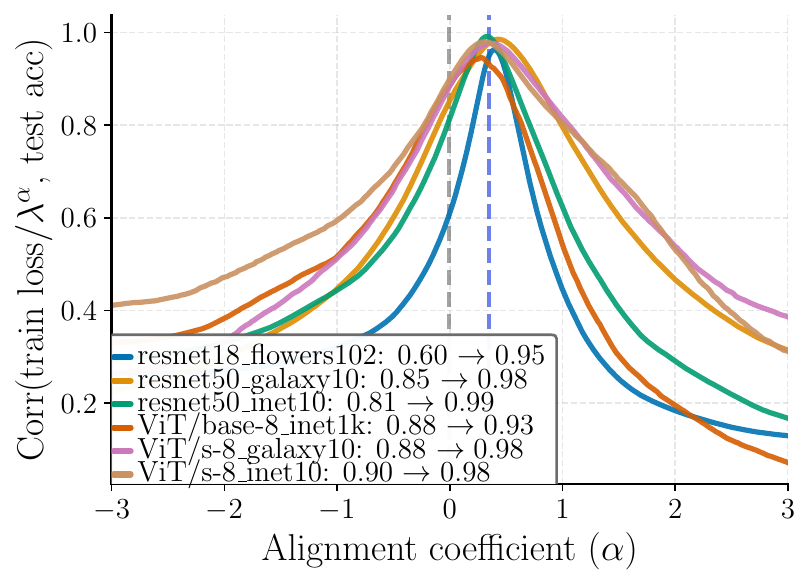}
    \caption{Spearman correlation ({\bf y-axis)} between LeJEPA's training loss and downstream accuracy on the dataset's classification task with a frozen backbone and linear evaluation. The {\bf x-axis} varies $\alpha$ in \cref{eq:corr} following our scaling law of the loss w.r.t. $\lambda$. Using $\alpha=0$ recovers the plain training loss. We clearly observe a very high correlation already for $\alpha=0$, which further increases up to $99\%$ for $\alpha=0.4$. The entire set of points is obtained across numerous hyper-parameters such as learning rate, weight decay, number of epochs, $\lambda$--demonstrating how {\bf LeJEPA's training loss is strongly predictive of downstream performance} which can be used for label-free cross-validation.}
    \label{fig:alignment_loss}
\end{figure}

A major challenge in SSL pretraining is the lack of reliable 
signals conveying the quality of the learned representation. As a result, it is common to monitor a supervised downstream task 
performance, sometimes supplemented with 
unsupervised embedding statistics \citep{agrawal2022alpha,garrido2023rankme,
thilak2023lidar}. This process is highly limiting since it requires labeled data that is costly and overly specialized. This is further exacerbated in the latest JEPA models where 
training losses exhibit low correlation with downstream performance--and 
may not even decrease monotonically during training.

In contrast, we find that LeJEPA's training loss behaves much more favorably--providing us with a meaningful signal on model quality. First, we provide in \cref{fig:heatmap_loss}, the 2D plane spanned by 
the SIGReg and prediction losses where a clear trend with downstream task accuracy can be observed. More strikingly, the combined training loss \eqref{eq:lejepa} with mixing coefficient $\lambda$
exhibits very high Spearman correlation \citep{spearman1961proof}, denoted as  $\rho_s$, of about $85\%$ with 
downstream accuracy--which is considered a strong signal. This strong relationship holds across datasets and architectures. As a result, a lower LeJEPA training loss reliably indicates a better downstream 
performance.

We can further improve this correlation through a simple scaling law based upon the trade-off weighting hyperparameter $\lambda$
\begin{align}
C^{(\alpha)} = \rho_s\left(\frac{\text{train\_loss}}{\lambda^{\alpha}}, 
\text{test\_accuracy}\right). \label{eq:corr}
\end{align}
By setting $\alpha \approx 0.4$, LeJEPA's training loss is able to achieve nearly 99\% correlation 
with downstream performance across multiple datasets and models. We depict the changes in $C^{(\alpha)}$ as a function of $\alpha$ on multiple datasets and models in 
\cref{fig:alignment_loss}, as well as the training LeJEPA loss against downstream performance in \cref{fig:corr_loss_extra}.  
{\bf The strong alignment between LeJEPA's training loss and model quality enables label-free SSL model selection and cross-validation}.

\subsection{In-Domain LeJEPA Outperforms Frontier Model Transfer Learning}
\label{sec:galaxy}

\begin{figure*}[t!]
    \centering
    \includegraphics[width=\linewidth]{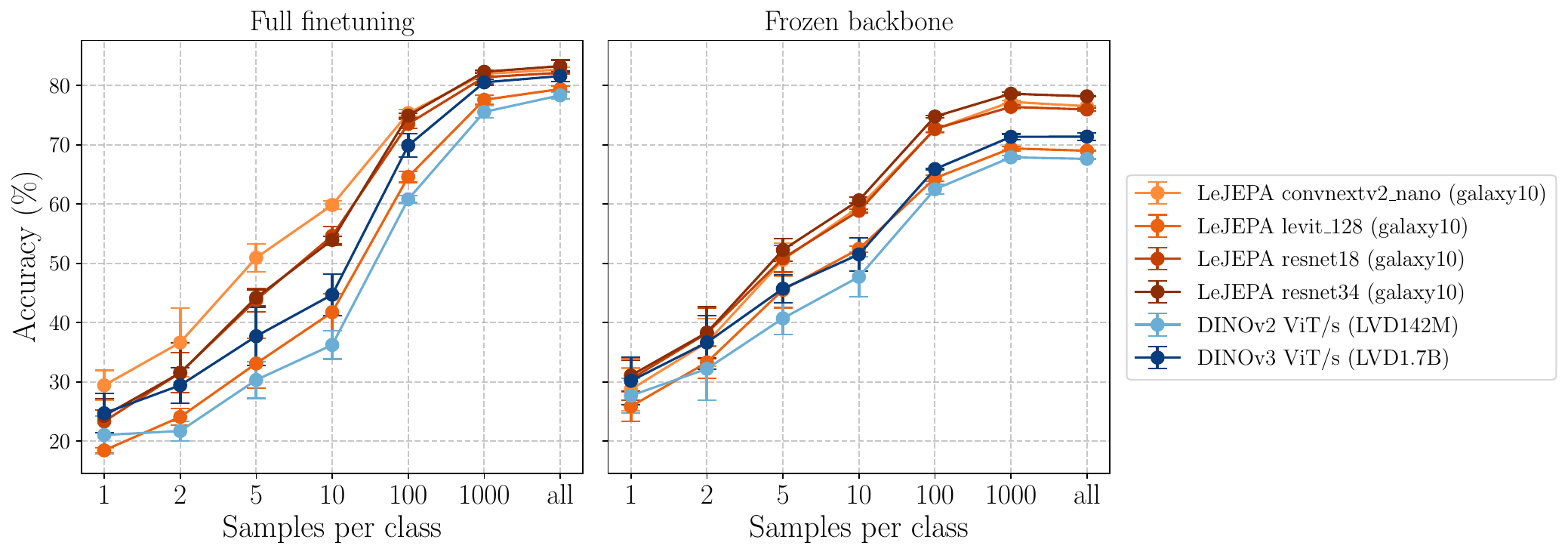}
    \caption{{\bf Small architecture in-domain (Galaxy10) LeJEPA pretraining} with linear probe evaluation using frozen backbone or full finetuning ({\bf columns}) and with varying number of samples per  class ({\bf x-axis)}. We compare against state-of-the-art foundation models (DINOv2/v3, IJEPA) over $3$ different random seeds. We observe that {\bf LeJEPA enables in-domain pretraining out of the box across architectures and able to outperform frontier foundation models}. Corresponding numbers are provided in \cref{tab:galaxy10}.}
    \label{fig:galaxy10}
\end{figure*}

\begin{table*}[t!]
\small
\centering
\setlength{\tabcolsep}{3pt}
\caption{Few-shot classification accuracy (percentages) on 8 datasets spanning textures, objects, and fine-grained categories. \textbf{Our LeJEPA achieves superior performance on fine-grained tasks (DTD, flowers102, food101) while requiring only 100 pretraining epochs compared to I-JEPA's 300 epochs—a 3× reduction in training time and computational resources without sacrificing downstream task performance.} This efficiency gain is particularly valuable for practical applications where training budget is limited. Bold indicates best performance within the IN-1K comparison group, all numbers are percentages.}
\label{tab:transfer}
\begin{tabular}{llllllrrrrrrr|r}
\toprule
 & & & & & \multicolumn{9}{c}{Dataset} \\
\cmidrule(lr){6-14}
shots & model & params & pretrain & epochs & DTD & aircr. & cars & cifar10 & cifar100 & flowers102 & food & pets & avg. \\
\midrule
\multirow[c]{5}{*}{1} & LeJEPA ViT-L & 304M & IN-1K & 100 & \textbf{33.21} & 9.37 & 3.40 & 51.65 & 27.01 & 48.53 & 17.14 & 46.11 & 29.55 \\
 & LeJEPA ConvNeXtV2-H & 660M & IN-1K & 100 & 32.15 & 8.07 & 4.28 & 50.95 & 31.48 & \textbf{48.74} & \textbf{17.95} & 58.98 & 31.58 \\
 & I-JEPA ViT-H & 632M & IN-1K & 300 & 27.71 & 9.86 & 4.33 & \textbf{56.52} & 30.58 & 44.69 & 14.53 & 53.38 & 30.20 \\
 & I-JEPA ViT-H + STOP & 632M & IN-1K & 300 & 26.60 & \textbf{11.18} & \textbf{4.75} & 56.27 & \textbf{35.20} & 47.17 & 15.75 & \textbf{59.47} & 32.05 \\
\cmidrule{2-14}
 & \textit{I-JEPA ViT-H (22K)} & \textit{632M} & \textit{IN-22K} & \textit{900} & \textit{27.98} & \textit{13.00} & \textit{3.45} & \textit{61.84} & \textit{34.70} & \textit{89.72} & \textit{19.62} & \textit{30.86} & \textit{35.15} \\
\midrule
\multirow[c]{5}{*}{10} & LeJEPA ViT-L & 304M & IN-1K & 100 & \textbf{64.72} & 35.25 & 22.25 & 85.15 & 59.77 & \textbf{92.53} & \textbf{50.90} & 77.00 & 60.95 \\
 & LeJEPA ConvNeXtV2-H & 660M & IN-1K & 100 & 61.84 & 30.67 & 24.46 & 85.74 & 63.29 & 91.78 & 49.32 & 78.53 & 60.70 \\
 & I-JEPA ViT-H & 632M & IN-1K & 300 & 57.68 & 33.82 & 21.96 & 88.77 & 66.42 & 88.24 & 43.97 & 83.23 & 60.51 \\
 & I-JEPA ViT-H + STOP & 632M & IN-1K & 300 & 57.00 & \textbf{39.77} & \textbf{25.21} & \textbf{90.09} & \textbf{70.32} & 90.16 & 45.68 & \textbf{85.13} & 62.92 \\
\cmidrule{2-14}
 & \textit{I-JEPA ViT-H (22K)} & \textit{632M} & \textit{IN-22K} & \textit{900} & \textit{58.74} & \textit{43.52} & \textit{18.27} & \textit{94.83} & \textit{75.23} & \textit{98.94} & \textit{49.06} & \textit{67.66} & 63.28 \\
\midrule
\multirow[c]{5}{*}{all} & LeJEPA ViT-L & 304M & IN-1K & 100 & \textbf{78.30} & 57.01 & 57.28 & 96.50 & 83.71 & \textbf{91.21} & \textbf{82.05} & 89.74 & 79.48 \\
 & LeJEPA ConvNeXtV2-H & 660M & IN-1K & 100 & 76.60 & 52.99 & 54.88 & 96.15 & 81.34 & 91.11 & 77.64 & 89.76 & 77.56 \\
 & I-JEPA ViT-H & 632M & IN-1K & 300 & 73.32 & 56.61 & 54.47 & 97.54 & 86.42 & 86.47 & 81.02 & 92.11 & 78.50 \\
 & I-JEPA ViT-H + STOP & 632M & IN-1K & 300 & 73.87 & \textbf{61.95} & \textbf{61.27} & \textbf{98.02} & \textbf{87.78} & 88.08 & 81.72 & \textbf{92.88} & 80.70 \\
\cmidrule{2-14}
 & \textit{I-JEPA ViT-H (22K)} & \textit{632M} & \textit{IN-22K} & \textit{900} & \textit{75.67} & \textit{65.39} & \textit{49.79} & \textit{98.46} & \textit{89.95} & \textit{98.54} & \textit{81.58} & \textit{87.19} & \textit{80.82} \\
\bottomrule
\end{tabular}
\end{table*}

\begin{figure*}[t!]
    \centering
    \includegraphics[width=0.33\linewidth]{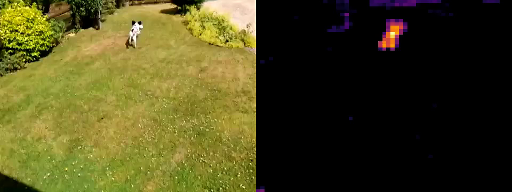}
    \includegraphics[width=0.33\linewidth]{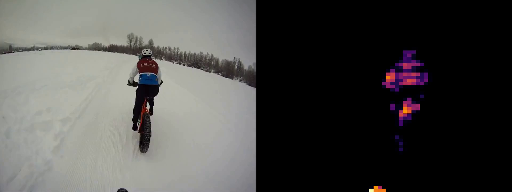}
    \includegraphics[width=0.33\linewidth]{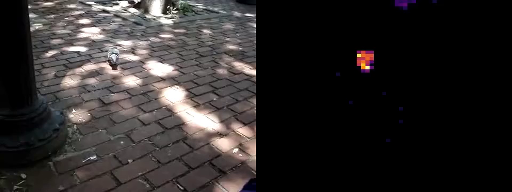}
    \caption{{\bf Emergent Object Segmentation via Last Layer Thresholding.} LeJEPA naturally learns to segment and track salient objects (shown in attention maps on the right of each video) without explicit supervision. The results display impressive visual quality and strong temporal consistency across video frames ({\em videos provided on our \href{https://rbalestr-lab.github.io/lejepa/}{project page}}). This emergent capability demonstrates the rich semantic representations learned through our self-supervised approach.}
    \label{fig:video_seg}
\end{figure*}

\begin{figure}[t!]
    \centering
    \includegraphics[width=0.49\linewidth]{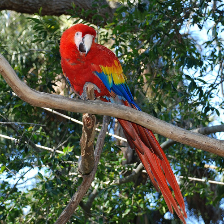}
    \includegraphics[width=0.49\linewidth]{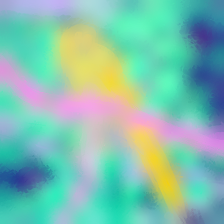}
    \\
    \includegraphics[width=0.49\linewidth]{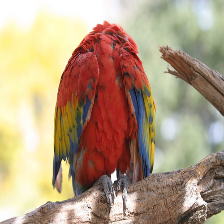}
    \includegraphics[width=0.49\linewidth]{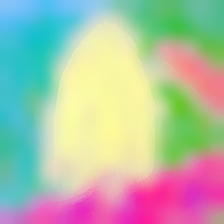}
    \\
    \includegraphics[width=0.49\linewidth]{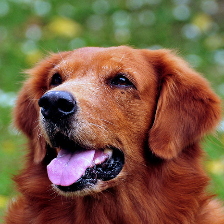}
    \includegraphics[width=0.49\linewidth]{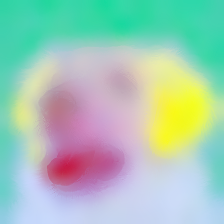}
    \caption{\textbf{LeJEPA learns rich semantic representations through self-supervised learning.} PCA visualization of last-layer features from LeJEPA (ViT-Large, 100 epochs on ImageNet-1K). For each image, features are independently projected to RGB using the first 3 principal components. Without any supervision, LeJEPA spontaneously develops semantically meaningful representations: notice how warm colors (red/magenta/pink) consistently capture foreground objects (parrot bodies, dog face), while cool colors (cyan/green/yellow) represent backgrounds and foliage. This emergent object-background separation and perceptual grouping discovered the visual structure of the world purely from unlabeled data.}
    \label{fig:attention_vis}
\end{figure}

A key promise of self-supervised learning is to learn universal representations that generalize across tasks and domains. However, current frontier foundation models (e.g., DINOv2/v3, IJEPA) are pretrained on natural images forcing practitioners in specialized domains to collect large amount of labels for supervised finetuning. In fact, most frontier models can not be trained directly on those domains as the number of samples may be small and searching again for the hyper-parameters would be cumbersome yet necessary \citep{assran2022hidden}.

To demonstrate LeJEPA's versatility and ability to resolve that current pain-point, we propose to pretrain directly on a new domain without any change in the loss or the pretraining pipeline. We select the Galaxy10 dataset, a galaxy morphology classification task that differs significantly from natural images in both visual structure and statistical properties \citep{balestriero2025gaussian}. The dataset contains 11,000 training samples across 10 galaxy types. For LeJEPA, we use the default hyper-parameters and pretrain for 400 epochs a variety of backbones. We compare against the latest DINOv2, DINOv3 and IJEPA. We report in \cref{fig:galaxy10} the top1 accuracy for linear probing both with frozen backbone and full-finetuning. We observe that {\bf in-domain pretraining with LeJEPA substantially outperforms state-of-the-art frontier models (DINOv2, DINOv3) on both linear probing and full finetuning}. Additional datasets and backbones are provided in \cref{tab:in_domain} depicting LeJEPA's ability to train in-domain, even with a dataset with $1000$ samples (flowers102). Coupling this result with the stability of LeJEPA across architectures and hyper-parameters should offer a promising alternatives in domains not yet accounted for by the latest frontier models.

\subsection{LeJEPA Scales Across Data and Models}
\label{sec:scale}

We now propose to apply LeJEPA over a larger pretraining dataset, i.e., Imagenet-1k, and over larger backbones such as ViT/Large (0.3B), ConvNextV2-Huge (0.6B). For those two models, we reach an online linear probe accuracy on inet1k of 77.1\% and 78.5\%  respectively. Beyond in-distribution performances, we also explore transfer learning. For those experiments, our baselines are IJEPA with a ViT-Huge (0.6B) which is the closest to our setup, and we also include a recent improved version of IJEPA including additional stochastic prediction tasks \citep{bar2023stochastic} that is coined IJEPA + STOP. For LeJEPA, we employ the same recipe as described in \cref{sec:hparams} and report transfer learning performances with frozen backbone in \cref{tab:transfer}. We observe that we consistently outperform IJEPA while employed a smaller model and shorted training schedule. Beyond top1 accuracy, we also echo our findings from \cref{sec:loss_corr} about LeJEPA's training loss quality. In our setup, we observe a very stable and smooth training curve indicating a stable optimization landscape removing the need for careful hyperparameter selection (recall \cref{thm:ecf_stability}). We provide an example on a ViT-gigantic (1.8B parameters) in \cref{fig:teaser}.

\subsection{Emergent Semantic Structure in LeJEPA Representations}

A hallmark of 
successful self-supervised learning is the emergence of semantically 
meaningful attention patterns without explicit supervision 
\citep{caron2021emerging}. To assess whether LeJEPA learns such structure, 
we visualize the attention maps of the learned representations. Following 
DINO \citep{caron2021emerging}, we apply PCA to the embeddings and visualize 
the first principal components, which reveal clear correspondence to object 
boundaries and salient regions (\cref{fig:attention_vis}). Furthermore, we 
explore whether these attention patterns can enable unsupervised video 
segmentation—a challenging task requiring temporal consistency and object 
understanding. By thresholding the self-attention maps of the [CLS] token, 
we obtain binary masks that track objects across frames without any 
segmentation labels during training. As shown in \cref{fig:video_seg}, 
{\bf LeJEPA's attention naturally segments foreground objects from background 
with remarkable temporal coherence}, suggesting that the learned 
representations capture both spatial semantics and temporal structure. 
This emergent capability demonstrates that LeJEPA's stability-focused 
objective does not sacrifice the semantic richness of learned features.

\section{Conclusion}
We have established a principled theoretical framework for JEPA-based self-supervised 
learning that fundamentally resolves its core pathologies. Our contributions span 
theory and practice: we proved that isotropic Gaussian embeddings uniquely minimize 
worst-case downstream risk, introduced SIGReg as a tractable and provably correct 
method to enforce this distribution, and demonstrated that this approach eliminates 
representational collapse by design--and not through ad-hoc combinations of teacher-student 
networks, stop-gradients, or asymmetric architectures.

We validate LeJEPA across domains and over $60$ architectures including gigantic versions with 1.8B parameters. In spite of its simplicify , LeJEPA matches state-of-the-art 
performance while requiring fewer than 50 lines of core implementation. Critically, 
our approach provides what SSL has long needed: a mathematically rigorous foundation 
that directly informs practical algorithm design.

\section*{Acknowledgments}
We would like to thank Mike Rabbat and
Lucas Maes for providing valuable feedbacks on the manuscript.

\bibliography{bibliography}
\bibliographystyle{plainnat}

\clearpage
\onecolumn 
\appendix

\newcommand{\FancyAppendixTitle}[1]{%
    \thispagestyle{empty}
    \begin{center}
        \noindent
        \begin{tikzpicture}
            \draw[line width=2pt] (0,0) -- (\textwidth,0);
        \end{tikzpicture}
        \\[0.5em]
        \begin{tikzpicture}
            \draw[line width=0.5pt] (0,0) -- (\textwidth,0);
        \end{tikzpicture}
        \\[2.5em]
        {\fontsize{22pt}{26pt}\selectfont\bfseries #1}\\[1.2em]
        {\fontsize{16pt}{20pt}\selectfont\bfseries Appendix}\\[2.5em]
        \begin{tikzpicture}
            \draw[line width=0.5pt] (0,0) -- (\textwidth,0);
        \end{tikzpicture}
        \\[0.5em]
        \begin{tikzpicture}
            \draw[line width=2pt] (0,0) -- (\textwidth,0);
        \end{tikzpicture}
    \end{center}
    \vspace{2cm}
}
\FancyAppendixTitle{LeJEPA}

\section{Additional Details on Nonlinear Probing}
\label{sec:additional_nonlinear}

\subsection{kNN Probing}

To allow for more flexible evaluation of the pretrained encoder $f_{\vtheta}$, it is standard to work with a $k$-NN prober \citep{taunk2019brief}, both for regression and classification. We rely on the radial $k$-NN variation that leverages a sample-dependent $k$--improving performance for non uniform distributions of samples \citep{sun2010adaptive,zhang2017efficient,abu2019effects}.

We denote the underlying embedding density as $p_{z}\in C^3$ with derivatives of order up to $3$ bounded, and finite Fisher information and covariance. This regularity condition is fulfilled by current encoders. The {\em unknown} labels come from the target function $\eta:\R^K\to\R$, assumed $C^2$. We handle classification tasks by setting $\eta(\vz)=\mathbb{P}(Y=1\mid \vz)$. The training consists of the $N$ embeddings along with their training labels $\{(\vz_n,\eta(\vz_n))\}_{n=1}^{N}$, where we will denote $\vy_{n}\triangleq \eta(\vz_n)$. The prediction for a query vector $\vq$ is formed as
\begin{align}
\widehat{\vy}(\vq)
:=
\frac{1}{\vy(\vq)}\sum_{n:\norm{\vz_{n}-\vq}\le r_0}\vy_n,\tag{kNN}\label{eq:kNN}
\end{align}
with $\vy(\vq)\triangleq\#\{n:\norm{\vz_{n}-\vq}\le r_0\}$ counting the number of samples within a $r$-radius ball around $\vq$. The radius $r$ controls how many neighbors predictions are averaged to form the query's prediction. As per the linear probing's \cref{thm:linear_probe_bias}, we can characterize the bias of the estimator \cref{eq:kNN} at a particular query point, as formalized below.

\begin{lemma}[label={thm:knn_bias}]{k-NN Pointwise Bias}{}
The \eqref{eq:kNN} estimator has bias at query $\vq$ given by
\begin{multline*}
    \mathrm{Bias}(\vq)=
\frac{r_0^2}{d+2}\Big(\grad\eta(\vq)^\top\grad\log p_{z}(\vq)+\tfrac{1}{2}\Delta\eta(\vz)\Big)\\+o(r_0^2),
\end{multline*}
where the remainder $o(r_0^2)$ is uniform in $\vq$. (Proof in \cref{proof:knn_bias}.)
\end{lemma}

To obtain the integrated bias, i.e., over the distribution of query points, we consider the following two properties. First, the distribution of query points follow the training distribution, i.e., $\vq \sim p_{z}$, second, target function $\eta$ has gradient which is mean-zero and isotropic with $\E\big[\grad\eta(\vz)\grad\eta(\vz)^\top\big]=\tau_g^2I_d$ with $\tau_g^2\in(0,\infty)$
uniformly in $\vz$. We also have any finite scalar-constraint on the covariance of the embeddings such as $\Tr(\Sigma)=c$ or $\|\Sigma\|_F=c$ for a finite constant $c$.

\begin{theorem}[label={thm:knn_optimal}]{k-NN isotropic Gaussian Optimality}{}
The integrated squared bias of \eqref{eq:kNN} satisfies
\[
\E_{\vz}\big[\mathrm{Bias}(\vz)^2\big]
=
\frac{r_0^4}{(K+2)^2}\tau_g^2J(p)\\+O(r_0^4),
\]
and the isotropic Gaussian  is the unique minimizer of the integrated square bias. (Proof in \cref{proof:knn_optimal}.)
\end{theorem}

As a result, we now have a unique minimizer for the optimal embedding density for both the linear and k-NN probes.

\subsection{Kernel Probing}

As an alternative to \eqref{eq:kNN}, it is also common to leverage kernel methods, which we consider in this section.

Consider a kernel $K:\mathbb{R}^K\to\mathbb{R}$ with the following standard properties
\begin{align}
    &\int_{\mathbb{R}^d} K(u)du=1,\tag{normalized}\\
    &\int_{\mathbb{R}^d} uK(u)du=0,\tag{symmetric}\\
    &\int_{\mathbb{R}^d} u u^\top K(u)du=\mu_2(K)I_d,\tag{isotropic}\\
    &R(K)\triangleq\int_{\mathbb{R}^d} K(u)^2du<\infty,\tag{finite roughness}
\end{align}
for some $\mu_2(K)\in(0,\infty)$, some bandwidth $h>0$ and denoting $K_h(t)\triangleq h^{-d}K(t/h)$, we remind the reader that the Nadaraya-Watson estimator, introduced in \citet{nadaraya1964estimating,watson1964smooth}, at a query $\vq\in\mathbb{R}^d$ is
\begin{align}
\widehat \vy(\vq)\triangleq \frac{\sum_{n=1}^N K_h(\vq-\vx_n)\vy_n}{\sum_{n=1}^N K_h(\vq-\vx_n)}.\tag{NW}\label{eq:NW}
\end{align}
Similarly to \eqref{eq:kNN}, we will see that the performance of \eqref{eq:NW} depends crucially on the distribution of the training points. We have access to our dataset of inputs from $p_z$ and for each sample $\vz_n$ the corresponding target is given from $\eta(\vz_n)=\mathbb{E}[Y_n\mid \vz_n]$. We also denote the corresponding conditional variance of the target function at that point as $v(x)=\mathrm{Var}(Y_i\mid X_i=x)$. We follow the regularity conditions of the k-NN probing derivations and additionally assume that $p$ has sufficiently light tails so that for each coordinate $j$, $\lim_{\|x\|\to\infty} p(x)=0$ and $\lim_{\|x\|\to\infty} x_jp(x)=0$. We first derive the pointwise bias and variance for $\widehat \vy(\vq)$.

\begin{lemma}[label={thm:kernel_bias}]{Kernel Bias and Variance}{}
For any fixed $\vq\in\mathbb{R}^d$ with $p(\vq)>0$, as $h\to 0$ and $n h^d\to\infty$,
\begin{align*}
\mathrm{Bias}\big[\widehat \vy(\vq)\big]
&=\frac{h^2\mu_2(K)}{2}\Big(\Delta \vy(\vq)+2\nabla \vy(\vq)^\top \nabla\log p(\vq)\Big)+o(h^2),\\
\mathrm{Var}\big[\widehat \vy(\vq)\big]
&=\frac{R(K)}{n h^d}\frac{v(\vq)}{p(\vq)}+o\big((n h^d)^{-1}\big).
\end{align*}
The $o(\cdot)$ terms are uniform over compact sets where $p$ is bounded away from zero.
(Proof in \cref{proof:kernel_bias}.)
\end{lemma}

We now show that, under a fixed mean and total-covariance constraint on $p_z$, the isotropic Gaussian distribution uniquely minimizes the bias and variance of the kernel regression estimator at any test point. We restrict the smoothness class of the target function using 
\begin{multline*}
    \mathcal{M}(L,B)\triangleq\Big\{m\in C^2(\mathbb{R}^d):\|\nabla \vy(\vq)\|\le L,\\|\Delta \vy(\vq)|\le B, \forall  \vq\in\mathbb{R}^d\Big\},
\end{multline*}
allowing us to formalize below the worst case integrated bias and the optimal density for $z$.

\begin{theorem}[label={thm:kernel_optimal}]{Kernel isotropic Gaussian Optimality}{thm:kernel_optimal}
The integrated squared bias of \eqref{eq:NW} satisfies
\begin{multline*}
\sup_{m\in\mathcal{M}(L,B)}\mathbb{E}_{z}\left[\mathrm{Bias}\big[\widehat \vy(\vz)\big]\right]
\le \Big(\frac{h^2\mu_2(K)}{2}\Big)^2 \Big(2 B^2 + 8 L^2J(p)\Big)+o(h^4),
\end{multline*}
and the integrated variance is independent of $p$. 
Among all densities $p$ on $\mathbb{R}^d$ with total-variance constrained, e.g., $\Tr(\Sigma)=c$, the isotropic Gaussian is the unique minimizer. (Proof in \cref{proof:kernel_optimal}.)
\end{theorem}

\section{Proofs}

\subsection{Proof of \cref{thm:linear_probe_bias}}
\label{proof:linear_probe_bias}

\begin{proof}
Our proof follows standard derivations when it comes to studying the bias of an estimator. Let's consider the ridge regression problem (Tikhonov regularized least squares estimator) with close form estimator
\begin{equation}
\hat{\boldsymbol{\beta}} = (\mathbf{X}^T \mathbf{X} + \lambda_{\rm wd} \mathbf{I})^{-1} \mathbf{X}^T \mathbf{Y}.
\end{equation}
The labels are formed from the ground truth parameter $\beta_{\rm true}$ with centered error, as per $\mathbf{Y} = \mathbf{X}\boldsymbol{\beta}_{\text{true}} + \boldsymbol{\varepsilon}$ where $\mathbb{E}[\boldsymbol{\varepsilon}] = \mathbf{0}$.
We can now look at the bias of our estimator given by
\begin{align*}
\text{Bias}(\hat{\boldsymbol{\beta}}) &= \mathbb{E}[\hat{\boldsymbol{\beta}}] - \boldsymbol{\beta}_{\text{true}} \\
&=(\mathbf{X}^T \mathbf{X} + \lambda_{\rm wd} \mathbf{I})^{-1} \mathbf{X}^T \mathbf{X}\boldsymbol{\beta}_{\text{true}}-\boldsymbol{\beta}_{\text{true}}\\
&= -\lambda_{\rm wd}(\mathbf{X}^T \mathbf{X} + \lambda_{\rm wd} \mathbf{I})^{-1} \boldsymbol{\beta}_{\text{true}}\\
&= -\lambda_{\rm wd} \mathbf{Q}(\boldsymbol{\Lambda} + \lambda \mathbf{I})^{-1}\mathbf{Q}^T \boldsymbol{\beta}_{\text{true}}
\end{align*}
We will now compare that bias when $\mX$ has isotropic and anisotropic covariance with same total variance:
\begin{equation}
\frac{\lambda_1 + \lambda_2 + \cdots + \lambda_p}{p} = \bar{\lambda}.
\end{equation}
For any anisotropic covariance matrix of $\mX$, denote by $\vq_1$ the eigenvector with smallest eigenvalue, and let's denote by $\kappa>0$ a positive constant. We now define
\begin{equation}
\boldsymbol{\beta}_{\text{true}} = \kappa \cdot \mathbf{q}_p,
\end{equation}
leading to
\begin{align*}
\|\text{Bias}(\hat{\boldsymbol{\beta}})\|_{\text{isotropic}} = \frac{\lambda_{\rm wd}}{\bar{\lambda} + \lambda_{\rm wd}} \|\boldsymbol{\beta}_{\text{true}}\|,\\
\|\text{Bias}(\hat{\boldsymbol{\beta}})\|_{\text{non-isotropic}} = \frac{\lambda_{\rm wd}}{\lambda_p + \lambda_{\rm wd}} \|\boldsymbol{\beta}_{\text{true}}\|
\end{align*}
Since $\lambda_p < \bar{\lambda}$ (strict inequality when not isotropic):
\begin{equation*}
\frac{\lambda_{\rm wd}}{\lambda_p + \lambda_{\rm wd}} > \frac{\lambda_{\rm wd}}{\bar{\lambda} + \lambda_{\rm wd}}
\end{equation*}
we obtain that
\begin{equation*}
\|\text{Bias}(\hat{\boldsymbol{\beta}})\|_{\text{non-isotropic}} > \|\text{Bias}(\hat{\boldsymbol{\beta}})\|_{\text{isotropic}}
\end{equation*}

As a result, whenever the covariance matrix of $\mX$ is anisotropic, there will be downstream tasks for which the estimator bias is increased compared to having isotropic covariance matrix.
Anisotropic covariance structure thus amplifies regularization bias when the true parameter vector aligns unfavorably with the data's covariance structure.
\end{proof}

\subsection{Proof of \cref{thm:linear_probe_variance}}
\label{proof:linear_probe_variance}
\begin{proof}
We use the same formula as in \cref{proof:linear_probe_bias} with $\lambda_{\rm wd}=0$. We first see that the estimator is unbiased. We will now leverage that result to compute the covariance matrix of the estimator 
\begin{align*}
\text{Var}(\hat{\boldsymbol{\beta}}|\mathbf{X}) &= \mathbb{E}[(\hat{\boldsymbol{\beta}} - \boldsymbol{\beta})(\hat{\boldsymbol{\beta}} - \boldsymbol{\beta})^T|\mathbf{X}]\\
&= \mathbb{E}[(\mathbf{X}^T\mathbf{X})^{-1}\mathbf{X}^T\boldsymbol{\varepsilon}\boldsymbol{\varepsilon}^T\mathbf{X}(\mathbf{X}^T\mathbf{X})^{-1}|\mathbf{X}]\\
&= (\mathbf{X}^T\mathbf{X})^{-1}\mathbf{X}^T\mathbb{E}[\boldsymbol{\varepsilon}\boldsymbol{\varepsilon}^T|\mathbf{X}]\mathbf{X}(\mathbf{X}^T\mathbf{X})^{-1}\\
&= (\mathbf{X}^T\mathbf{X})^{-1}\mathbf{X}^T(\sigma^2\mathbf{I}_n)\mathbf{X}(\mathbf{X}^T\mathbf{X})^{-1}\\
&= \sigma^2(\mathbf{X}^T\mathbf{X})^{-1}
\end{align*}
leading to the total variance
$$\text{tr}(\text{Var}(\hat{\boldsymbol{\beta}})) = \sigma^2\text{tr}(\mathbf{G}^{-1})=\sigma^2 \sum_{j=1}^p \frac{1}{\lambda_j}$$
where we used the eigendecomposition:
$$\mathbf{G} = \mathbf{Q}\boldsymbol{\Lambda}\mathbf{Q}^T$$

The function $f(x) = \frac{1}{x}$ is strictly convex on $(0, \infty)$ allowing us to leverage Jensen's Inequality:
\begin{align*}
    \frac{1}{K}\sum_{k=1}^K \frac{1}{\lambda_k} > \frac{1}{\frac{1}{K}\sum_{j=1}^K \lambda_k}\\
    \iff \frac{1}{K}\sum_{k=1}^K \frac{1}{\lambda_k} > \frac{1}{K}\sum_{k=1}^{K}\frac{1}{\frac{1}{K}\sum_{j=1}^K \lambda_k}\\
    \iff \sum_{k=1}^K \frac{1}{\lambda_k} > \sum_{k=1}^{K}\frac{1}{\frac{1}{K}\sum_{j=1}^K \lambda_k}\\
    \iff \text{tr}(\text{Var}(\hat{\boldsymbol{\beta}}))_{\text{aniso}} > \text{tr}(\text{Var}(\hat{\boldsymbol{\beta}}))_{\text{iso}}
\end{align*}
The inequality is strict whenever the eigenvalues $\{\lambda_j\}_{j=1}^p$ are not all equal.
\end{proof}

\subsection{Proof of \cref{thm:knn_bias}}
\label{proof:knn_bias}
\begin{proof}

Under PPP, conditional expectations of $\widehat{\eta}(x)$ coincide with the normalized ball average
\[
\E\big[\widehat{\eta}(x)\big]
\;=\;
\frac{\int_{\Ball(0,r_0)} \eta(x+z)p(x+z)dz}{\int_{\Ball(0,r_0)} p(x+z)dz}
\quad\text{to second order in }r_0,
\]
which is the key surrogate used below.
\medskip
\noindent\textbf{Ball integrals.} For computations we use (by symmetry) for any $r>0$:
\[
\int_{\Ball(0,r)} zdz=0,\qquad
\int_{\Ball(0,r)} zz^\top dz=\frac{{\rm Vol}^{d+2}}{d+2}I_d,\qquad
\int_{\Ball(0,r)} \norm{z}^2dz=\frac{d{\rm Vol}^{d+2}}{d+2}.
\]

Fix $x\in\R^d$ and write $z\in\Ball(0,r_0)$ for local displacements. 
Assume $p\in C^3$, $\eta\in C^2$ with bounded derivatives on the region of interest, and expand a second-order Taylor expansion:
\begin{align*}
p(x+z)&=p(x)+\nabla p(x)^\top z+\tfrac12 z^\top H p(x)z + O(\|z\|^3),\\
\eta(x+z)&=\eta(x)+\nabla\eta(x)^\top z+\tfrac12 z^\top H\eta(x)z + O(\|z\|^3),
\end{align*}
with remainders satisfying $|R_\eta(x;z)|\le C_\eta\norm{z}^3$ and $|R_p(x;z)|\le C_p\norm{z}^3$ uniformly for $\norm{z}\le r_0$. Using the ball identities
$\int_{B(0,r)} zdz=0$ and $\int_{B(0,r)} zz^\top dz=\frac{v_d r^{d+2}}{d+2}I_d$ and collecting terms up to order $r_0^{d+2}$, we simplify the denominator as
\begin{align*}
    \mathcal{D}(x)&\triangleq\int_{\Ball(0,r_0)} p(x+z)dz\\
    &= \int_{\Ball(0,r_0)} \Big[p(x) + \grad p(x)^\top z + \tfrac{1}{2}z^\top \Hess p(x)z + R_p(x;z)\Big]dz\\
    &= {\rm Vol}_0^dp(x)\;+\;\frac{{\rm Vol}_0^{d+2}}{2(d+2)}\tr\big(\Hess p(x)\big)\;+\;O(r_0^{d+3}),
\end{align*}
since $\int zdz=0$ and $\int z^\top \Hess pzdz=\tr(\Hess p)\frac{\vol r_0^{d+2}}{d+2}$ and the denominator as
\begin{align*}
\mathcal{N}(x)&\triangleq \int_{\Ball(0,r_0)} \eta(x+z)p(x+z)dz\\
    &= \int \Big[\eta(x)+\grad\eta(x)^\top z+\tfrac{1}{2}z^\top \Hess\eta(x)z\Big]
\Big[p(x)+\grad p(x)^\top z+\tfrac{1}{2}z^\top \Hess p(x)z\Big]dz+O(r_0^{d+3})\\
&= \eta(x)p(x)\vol r_0^d+\eta(x)\frac{\vol r_0^{d+2}}{2(d+2)}\tr\big(\Hess p(x)\big)+\frac{\vol r_0^{d+2}}{d+2}\grad\eta(x)\cdot\grad p(x)+ \frac{\vol r_0^{d+2}}{2(d+2)}p(x)\tr\big(\Hess\eta(x)\big)
+O(r_0^{d+3}).
\end{align*}

Cubic terms vanish by symmetry, and quartic terms are $O(r_0^{d+4})$. 
Subtract $\eta(x)\mathcal{D}(x)$ to obtain the bias numerator:
\[
\mathcal{N}(x)-\eta(x)\mathcal{D}(x)
= \frac{v_dr_0^{d+2}}{d+2}\Big(\nabla\eta(x)\cdot\nabla p(x) + \tfrac{1}{2}p(x)\Delta\eta(x)\Big) + O(r_0^{d+3}).
\]
Write $\mathcal{D}(x)=v_d r_0^d p(x)\big(1+\alpha(x)r_0^2+O(r_0^3)\big)$ where $\alpha(x):=\frac{1}{2(d+2)p(x)}\mathrm{tr}(H p(x))$. Then
\begin{align*}
\frac{\mathcal{N}(x)}{\mathcal{D}(x)}-\eta(x)
&= \frac{ \frac{v_d r_0^{d+2}}{d+2}\left(\nabla\eta\cdot\nabla p + \frac{1}{2}p\Delta\eta\right) + O(r_0^{d+3})}{ v_d r_0^d p \left(1+\alpha r_0^2+O(r_0^3)\right)}\\[0.5ex]
&= \frac{r_0^2}{d+2}\left(\frac{\nabla\eta\cdot\nabla p}{p} + \frac{1}{2}\Delta\eta\right)\Big(1-\alpha r_0^2+O(r_0^3)\Big)\ +\ O(r_0^3)\\[0.5ex]
&= \frac{r_0^2}{d+2}\Big(\nabla\eta(x)\cdot\nabla\log p(x) + \tfrac{1}{2}\Delta\eta(x)\Big)\ +\ o(r_0^2),
\end{align*}
uniformly on $\mathcal{K}$. This gives the bias formula
\[
\mathbb{E}\big[\widehat{\eta}(x)\big]-\eta(x)
=
\frac{r_0^2}{d+2}\Big(\nabla\eta(x)\cdot\nabla\log p(x) + \tfrac{1}{2}\Delta\eta(x)\Big)\ +\ o(r_0^2),
\]

completing the proof.
\end{proof}

\subsection{Proof of \cref{thm:knn_optimal}}
\label{proof:knn_optimal}
\begin{proof}

Recall from \cref{proof:knn_bias} that the bias term as sample $\vx$ is given by
\begin{align*}
\mathrm{Bias}(\vx)
=&\frac{r_0^2}{d+2}\Big(\grad\eta(x)\cdot\grad\log p(x)\Big)\;+\;\frac{r_0^2}{2(d+2)}\Delta\eta(x)\;+\;o(r_0^2)\\
=&\frac{r_0^2}{d+2}\big(A(x)+C(x)\big)+o(r_0^2),
\end{align*}
where we defined $A(x)\triangleq \nabla\eta(x)\cdot\nabla\log p(x)$ and $ 
C(x)\triangleq\frac{1}{2}\Delta\eta(x)$.
We now square and take expectation of $X\sim p$ and the isotropic gradient prior 
\begin{align}
\mathbb{E}\big[\mathrm{Bias}(X)^2\big]
&=\mathbb{E}\big[\left(\frac{r_0^2}{d+2}\right)^2\big(A(x)^2 + 2A(x)C(x) + C(x)^2\big)
+ o(r_0^4)\big]\\
&=\left(\frac{r_0^2}{d+2}\right)^2
\Big\{\underbrace{\mathbb{E}\big[A(X)^2\big]}_{\text{score-gradient term}}
+ \underbrace{2\mathbb{E}\big[A(X)C(X)\big]}_{\text{cross term}}
+ \underbrace{\mathbb{E}\big[C(X)^2\big]}_{\text{curvature term}}\Big\}
+ o(r_0^4).
\label{eq:three-terms}
\end{align}
We will derive each term separately, recalling that we assume an isotropic gradient prior for $\eta$, i.e., $\mathbb{E}\big[\nabla\eta(x)\big]=0$ and $\mathbb{E}\big[\nabla\eta(x)\nabla\eta(x)^\top\big]=\tau_g^2I_d$,
for some $\tau_g^2\in(0,\infty)$.

\paragraph{1) The score-gradient term $\mathbb{E}[A(X)^2]$.}

Using $v(x):=\nabla\log p(x)$ for brevity:
\begin{align*}
\mathbb{E}\big[A(X)^2\big]
=&\mathbb{E}_X\big[\mathbb{E}_\eta[A(X)^2]\big]\\
=&\mathbb{E}_X\big[\mathbb{E}_\eta[\big(\nabla\eta(x)^\top v(x)\big)^2]\big]\\
=&\mathbb{E}_X\big[\mathbb{E}_\eta[\nabla\eta(x)^\top\Big(v(x)v(x)^\top\Big)\nabla\eta(x)]\big]\\
=&\mathbb{E}_X\big[\mathbb{E}_\eta[\mathrm{tr}\Big(v(x)v(x)^\top\nabla\eta(x)\nabla\eta(x)^\top\Big)]\big]\\
=&\mathbb{E}_X\big[\mathrm{tr}\Big(v(x)v(x)^\top\mathbb{E}_\eta[\nabla\eta(x)\nabla\eta(x)^\top]\Big)\big]\\=&\mathbb{E}_X\big[\tau_g^2\|v(x)\|^2\big]\\
=&\tau_g^2\mathbb{E}_X\big[\|v(X)\|^2\big]\\
=&\tau_g^2\int_{\mathbb{R}^d} \|\nabla\log p(x)\|^2p(x)dx
\end{align*}
recovering the Fisher-information functional $J(p)$, scaled by $\tau_g^2$

\paragraph{2) The cross term $2\mathbb{E}[A(X)C(X)]$.}
We have
\[
A(x)C(x)=\frac{1}{2}\big(\nabla\eta(x)^\top v(x)\big)\Delta\eta(x).
\]
Under the prior, $\nabla\eta$ is mean-zero and isotropic; if, additionally, $\Delta\eta$ is uncorrelated with $\nabla\eta$ and has zero mean (or is bounded and mean-zero after centering), then $\mathbb{E}_\eta[A(x)C(x)]=0$.  If one does \emph{not} assume the orthogonality/vanishing covariance above, then $\mathbb{E}[A(X)C(X)]$ is a finite constant (depending on the joint law of derivatives of $\eta$), and the cross term contributes
\[
\left(\frac{r_0^2}{d+2}\right)^2\cdot 2\mathbb{E}[A(X)C(X)]
=\ O(r_0^4),
\]
not $o(r_0^4)$. In that general case, the leading $p$-dependent term of $\mathbb{E}[\mathrm{Bias}(X)^2]$ is still the \emph{score-gradient} $\tau_g^2J(p)$.

\paragraph{3) The curvature term $\mathbb{E}[C(X)^2]$.}

\begin{align*}
\mathbb{E}\big[C(X)^2\big]
=&\mathbb{E}_X\big[\mathbb{E}_\eta[C(X)^2]\big]\\
=&\frac{1}{4}\mathbb{E}_X\big[\mathbb{E}_\eta[(\Delta\eta(X))^2\big]
\end{align*}
which is independent of $p$, hence $\mathbb{E}\big[C(X)^2\big]=O(1)$

\paragraph{Putting it together.}
Substituting into \eqref{eq:three-terms}:
\begin{align*}
\mathbb{E}\big[\mathrm{Bias}(X)^2\big]
&=\left(\frac{r_0^2}{d+2}\right)^2\Big\{\tau_g^2J(p) + O(1)\Big\} + o(r_0^4)\\
&=\frac{r_0^4}{(d+2)^2}\tau_g^2J(p)\;+\;O(r_0^4),
\end{align*}

We show that, among all mean-zero distributions $p$ on $\mathbb{R}^d$ with a given \emph{scalar} constraint on the covariance (trace, determinant, Frobenius norm, or spectral radius), the density that minimizes the Fisher-information functional
\[
J(p)\;:=\;\int_{\mathbb{R}^d}\|\nabla\log p(x)\|^2p(x)dx
\]
is the Gaussian with \emph{isotropic} covariance satisfying the same scalar constraint.
\medskip
\noindent
We proceed in two steps: (i) for fixed covariance matrix $\Sigma\succ 0$, $J(p)$ is minimized by the Gaussian $\mathcal{N}(0,\Sigma)$ and attains the value $\mathrm{tr}(\Sigma^{-1})$; (ii) for each scalar constraint, $\mathrm{tr}(\Sigma^{-1})$ is minimized by $\Sigma=sI_d$ for the appropriate scalar $s>0$.

\begin{lemma}[label={lem:fixed-Sigma}]{Special case: Recovery of VCReg}{}
Let $p$ be a mean-zero probability density on $\mathbb{R}^d$ with covariance $\Sigma=\mathbb{E}[X X^\top]\succ 0$. Then
\[
J(p)\;\ge\;\mathrm{tr}(\Sigma^{-1}),
\]
with equality if and only if $p=\mathcal{N}(0,\Sigma)$.
\end{lemma}
\begin{proof}
Consider the location family $p_\theta(x):=p(x-\theta)$, $\theta\in\mathbb{R}^d$. Its Fisher-information matrix at $\theta$ is
\[
\mathcal{I}(\theta)
\;=\;
\mathbb{E}\big[\nabla_\theta\log p_\theta(X)\nabla_\theta\log p_\theta(X)^\top\big]
\;=\;
\mathbb{E}\big[\nabla\log p(X)\nabla\log p(X)^\top\big],
\]
so that $J(p)=\mathrm{tr}\mathcal{I}(\theta)$. The estimator $T(X)\equiv X$ is unbiased for $\theta$ under $p_\theta$, with $\mathrm{Cov}(T)=\Sigma$. The matrix Cramér--Rao bound gives $\mathrm{Cov}(T)\succeq \mathcal{I}(\theta)^{-1}$, i.e., $\mathcal{I}(\theta)\succeq \Sigma^{-1}$. Taking traces yields $J(p)\ge \mathrm{tr}(\Sigma^{-1})$.
Equality in the matrix Cramér--Rao bound holds if and only if the score is an \emph{affine} function of $X-\theta$, i.e., $\nabla\log p_\theta(X)=A(X-\theta)$ a.s.\ for some matrix $A$; integrating this identity shows $p_\theta$ is Gaussian with precision matrix $-A$, hence $p=\mathcal{N}(0,\Sigma)$.
\end{proof}
\subsection*{Step 2: Optimizing over covariance shapes under scalar constraints}
Write the eigenvalues of $\Sigma$ as $\lambda_1,\ldots,\lambda_d>0$. Then
\[
\mathrm{tr}(\Sigma^{-1})=\sum_{i=1}^d \frac{1}{\lambda_i}.
\]
We now solve $\min \sum_i 1/\lambda_i$ under each scalar constraint; in every case the minimum is attained when all $\lambda_i$ are equal, i.e., $\Sigma=sI_d$.
\paragraph{(a) Trace constraint.}
Given $\mathrm{tr}(\Sigma)=\sum_i \lambda_i=t>0$, by Cauchy--Schwarz,
\[
\left(\sum_{i=1}^d \frac{1}{\lambda_i}\right)\left(\sum_{i=1}^d \lambda_i\right)\;\ge\;\left(\sum_{i=1}^d 1\right)^2=d^2,
\]
with equality if and only if $\lambda_1=\cdots=\lambda_d$. Hence
\[
\min_{\Sigma\succ 0:\ \mathrm{tr}(\Sigma)=t}\ \mathrm{tr}(\Sigma^{-1})
\;=\;\frac{d^2}{t},
\quad\text{attained at}\quad
\Sigma=\frac{t}{d}I_d.
\]
\paragraph{(b) Determinant constraint.}
Given $\det(\Sigma)=\prod_i \lambda_i=\delta>0$, set $\mu_i:=1/\lambda_i$ so that $\prod_i \mu_i=\delta^{-1}$. By the AM--GM inequality,
\[
\frac{1}{d}\sum_{i=1}^d \mu_i\;\ge\;\left(\prod_{i=1}^d \mu_i\right)^{1/d}=\delta^{-1/d},
\]
with equality iff $\mu_1=\cdots=\mu_d$, i.e., $\lambda_1=\cdots=\lambda_d$. Thus
\[
\min_{\Sigma\succ 0:\ \det(\Sigma)=\delta}\ \mathrm{tr}(\Sigma^{-1})
\;=\;d\delta^{-1/d},
\quad\text{attained at}\quad
\Sigma=\delta^{1/d}I_d.
\]
\paragraph{(c) Frobenius-norm constraint.}
Given $\|\Sigma\|_F^2=\sum_i \lambda_i^2=c^2>0$, minimize $f(\lambda):=\sum_i 1/\lambda_i$ over $\lambda_i>0$ subject to $g(\lambda):=\sum_i \lambda_i^2=c^2$. The Lagrangian
\[
\mathcal{L}(\lambda,\nu)=\sum_{i=1}^d \frac{1}{\lambda_i} + \nu\left(\sum_{i=1}^d \lambda_i^2 - c^2\right)
\]
has first-order conditions $-\lambda_i^{-2}+2\nu\lambda_i=0$ for all $i$, i.e., $\lambda_i^3=\frac{1}{2\nu}$, so all $\lambda_i$ are equal. Imposing $\sum \lambda_i^2=c^2$ yields $\lambda_i=c/\sqrt{d}$, hence
\[
\min_{\Sigma\succ 0:\ \|\Sigma\|_F=c}\ \mathrm{tr}(\Sigma^{-1})
\;=\;\sum_{i=1}^d \frac{1}{\lambda_i}
\;=\;\frac{d^{3/2}}{c},
\quad\text{attained at}\quad
\Sigma=\frac{c}{\sqrt{d}}I_d.
\]
\paragraph{(d) Spectral-radius constraint.}
Let the spectral radius be constrained by $\rho(\Sigma)=\max_i \lambda_i\le r$ for some $r>0$. Since $x\mapsto 1/x$ is strictly decreasing on $(0,\infty)$,
\[
\sum_{i=1}^d \frac{1}{\lambda_i}
\;\ge\;\sum_{i=1}^d \frac{1}{r}
\;=\;\frac{d}{r},
\]
with equality if and only if $\lambda_i=r$ for all $i$. Therefore
\[
\min_{\Sigma\succ 0:\ \rho(\Sigma)\le r}\ \mathrm{tr}(\Sigma^{-1})
\;=\;\frac{d}{r},
\quad\text{attained at}\quad
\Sigma=rI_d.
\]
(The same conclusion holds if the constraint is $\rho(\Sigma)=r$, since one may take all eigenvalues equal to $r$.)
\subsection*{Conclusion: Isotropic Gaussian is optimal}
Combining Lemma~\ref{lem:fixed-Sigma} with the solutions (a)--(d), we obtain:
\begin{theorem}[label={thm:isotropic-optimal}]{Special case: Recovery of VCReg}{}
Fix one of the following scalar covariance constraints for a mean-zero distribution $p$ on $\mathbb{R}^d$:
\begin{itemize}
\item trace: $\mathrm{tr}(\mathrm{Cov}(X))=t$,
\item determinant: $\det(\mathrm{Cov}(X))=\delta$,
\item Frobenius norm: $\|\mathrm{Cov}(X)\|_F=c$,
\item spectral radius upper bound: $\rho(\mathrm{Cov}(X))\le r$.
\end{itemize}
Then the Fisher-information functional $J(p)$ is minimized over all such $p$ by the isotropic Gaussian $p_G=\mathcal{N}(0,sI_d)$ with $s$ chosen to satisfy the constraint. The minimal values are:
\[
\begin{array}{ll}
\text{trace } t: & J_{\min}=\dfrac{d^2}{t},\quad s=\dfrac{t}{d},\\[1.2ex]
\text{determinant } \delta: & J_{\min}=d\delta^{-1/d},\quad s=\delta^{1/d},\\[1.2ex]
\text{Frobenius } c: & J_{\min}=\dfrac{d^{3/2}}{c},\quad s=\dfrac{c}{\sqrt{d}},\\[1.2ex]
\text{spectral radius } r: & J_{\min}=\dfrac{d}{r},\quad s=r.
\end{array}
\]
In each case, $p_G$ is the unique minimizer (up to null sets).
\end{theorem}
\begin{proof}
For any admissible $p$ with covariance $\Sigma$, Lemma~\ref{lem:fixed-Sigma} gives $J(p)\ge \mathrm{tr}(\Sigma^{-1})$. Minimizing the right-hand side under the stated scalar constraint yields $\Sigma=sI_d$ by the calculations in (a)--(d). Equality in Lemma~\ref{lem:fixed-Sigma} holds if and only if $p$ is Gaussian with that covariance, hence $p_G$ uniquely attains the bound.
\end{proof}
\end{proof}

\subsection{Proof of \cref{thm:kernel_bias}}
\label{proof:kernel_bias}
\begin{proof}
Write the numerator and denominator of $\widehat m(x)$ as
\[
B_n(x):=\sum_{i=1}^n K_h(x-X_i)Y_i,\qquad A_n(x):=\sum_{i=1}^n K_h(x-X_i),
\]
so that $\widehat m(x)=\frac{B_n(x)}{A_n(x)}$.
\emph{Bias.} Compute expectations using independence and change of variables. For the denominator,
\begin{align*}
\mathbb{E}[A_n(x)]
&=n\mathbb{E}\big[K_h(x-X)\big]\\
&=n\int_{\mathbb{R}^d} h^{-d}K\Big(\frac{x-u}{h}\Big)p(u)du\\
&=n\int_{\mathbb{R}^d} K(t)p(x-h t)dt\qquad (t:=(x-u)/h)\\
&=n\int_{\mathbb{R}^d} K(t)\Big(p(x)-ht^\top \nabla p(x)+\frac{h^2}{2}t^\top \nabla^2 p(x)t+o(h^2)\Big)dt\\
&=n\Big(p(x)+\frac{h^2}{2}\underbrace{\int t^\top \nabla^2 p(x)tK(t)dt}_{=\mu_2(K)\Delta p(x)}+o(h^2)\Big),
\end{align*}
where we used symmetry $\int t K(t)dt=0$ and isotropy $\int t t^\top K(t)dt=\mu_2(K) I_d$, which implies $\int t^\top \nabla^2 p(x)tK(t)dt=\mu_2(K)\mathrm{tr}(\nabla^2 p(x))=\mu_2(K)\Delta p(x)$.
Similarly, for the numerator,
\begin{align*}
\mathbb{E}[B_n(x)]
&=n\mathbb{E}\big[K_h(x-X)Y\big]
=n\int K(t)(m p)(x-h t)dt\\
&=n\int K(t)\Big((mp)(x)-ht^\top \nabla(mp)(x)+\frac{h^2}{2}t^\top \nabla^2(mp)(x)t+o(h^2)\Big)dt\\
&=n\Big(m(x)p(x)+\frac{h^2}{2}\mu_2(K)\mathrm{tr}\big(\nabla^2(mp)(x)\big)+o(h^2)\Big)\\
&=n\Big(m(x)p(x)+\frac{h^2\mu_2(K)}{2}\big(p\Delta m + m\Delta p + 2\nabla m^\top \nabla p\big)(x)+o(h^2)\Big),
\end{align*}
where the last step uses the fact that  $\mathrm{tr}\big(\nabla^2(mp)\big)=p\Delta m + m\Delta p + 2\nabla m^\top \nabla p$
by the product rule and symmetry of mixed derivatives.

Now expand the ratio $\frac{\mathbb{E}[B_n(x)]}{\mathbb{E}[A_n(x)]}$ using the identity
\[
\frac{a_0+h^2 a_2+o(h^2)}{b_0+h^2 b_2+o(h^2)}
=\frac{a_0}{b_0}
+h^2\frac{a_2 b_0-a_0 b_2}{b_0^2}
+o(h^2),
\]
with $a_0=m(x)p(x)$, $a_2=\frac{\mu_2(K)}{2}\big(p\Delta m + m\Delta p + 2\nabla m^\top \nabla p\big)(x)$, $b_0=p(x)$, and $b_2=\frac{\mu_2(K)}{2}\Delta p(x)$. This yields
\begin{align*}
\frac{\mathbb{E}[B_n(x)]}{\mathbb{E}[A_n(x)]}
&=m(x)
+\frac{h^2\mu_2(K)}{2}\frac{\big(p\Delta m + m\Delta p + 2\nabla m^\top \nabla p\big)p - m p\Delta p}{p^2}\Big|_{x}+o(h^2)\\
&=m(x)
+\frac{h^2\mu_2(K)}{2}\Big(\Delta m(x)+2\nabla m(x)^\top \frac{\nabla p(x)}{p(x)}\Big)+o(h^2),
\end{align*}
which recovers our statement.
\emph{Variance.} Linearize $\widehat m(x)=B_n(x)/A_n(x)$ around $(\mathbb{E}[B_n(x)],\mathbb{E}[A_n(x)])$ and use independence. To leading order,
\[
\mathrm{Var}[\widehat m(x)]\approx \frac{\mathrm{Var}[B_n(x)]}{(\mathbb{E}[A_n(x)])^2}.
\]
Compute
\begin{align*}
\mathrm{Var}[B_n(x)]
&= \sum_{i=1}^n \mathrm{Var}\big(K_h(x-X_i)Y_i\big)\quad\text{(independence)}\\
&= n\mathbb{E}\big[K_h(x-X)^2\mathrm{Var}(Y\mid X)\big]
= n\mathbb{E}\big[K_h(x-X)^2v(X)\big]\\
&= n\int h^{-2d}K\Big(\frac{x-u}{h}\Big)^2 v(u)p(u)du\\
&= n h^{-d}\int K(t)^2v(x-h t)p(x-h t)dt
= n h^{-d}\Big(R(K)v(x)p(x)+o(1)\Big),
\end{align*}
while
\[
\mathbb{E}[A_n(x)]=n\big(p(x)+o(1)\big).
\]
Therefore,
\[
\mathrm{Var}[\widehat m(x)]
\approx \frac{n h^{-d}R(K)v(x)p(x)}{n^2p(x)^2}
=\frac{R(K)}{n h^d}\frac{v(x)}{p(x)}+o\big((n h^d)^{-1}\big),
\]
completing the proof.
\end{proof}

\subsection{Proof of \cref{eq:pred1} to \cref{eq:pred2}}
\label{proof:pred}
\begin{proof}
Let $\bar{\mathbf{z}} = \frac{1}{V_g}\sum_{v=1}^{V_g}\mathbf{z}_{n,v}$ denote the mean of the first $V_g$ vectors.

We prove that:
\begin{equation}
\frac{1}{V_g}\sum_{v=1}^{V_g}\frac{1}{V}\sum_{v'=1}^{V}\| \mathbf{z}_{n,v} - \mathbf{z}_{n,v'} \|_2^2 = \frac{1}{V}\sum_{v'=1}^{V}\left\| \bar{\mathbf{z}} - \mathbf{z}_{n,v'} \right\|_2^2
\end{equation}

Expanding the left-hand side:
\begin{align}
\text{LHS} &= \frac{1}{V_g V}\sum_{v=1}^{V_g}\sum_{v'=1}^{V}\| \mathbf{z}_{n,v} - \mathbf{z}_{n,v'} \|_2^2 \\
&= \frac{1}{V_g V}\sum_{v=1}^{V_g}\sum_{v'=1}^{V}\left(\|\mathbf{z}_{n,v}\|_2^2 - 2\mathbf{z}_{n,v}^T\mathbf{z}_{n,v'} + \|\mathbf{z}_{n,v'}\|_2^2\right) \\
&= \frac{1}{V_g}\sum_{v=1}^{V_g}\|\mathbf{z}_{n,v}\|_2^2 - \frac{2}{V_g V}\sum_{v=1}^{V_g}\sum_{v'=1}^{V}\mathbf{z}_{n,v}^T\mathbf{z}_{n,v'} + \frac{1}{V}\sum_{v'=1}^{V}\|\mathbf{z}_{n,v'}\|_2^2 \\
&= \frac{1}{V_g}\sum_{v=1}^{V_g}\|\mathbf{z}_{n,v}\|_2^2 - \frac{2}{V}\bar{\mathbf{z}}^T\sum_{v'=1}^{V}\mathbf{z}_{n,v'} + \frac{1}{V}\sum_{v'=1}^{V}\|\mathbf{z}_{n,v'}\|_2^2
\end{align}

Expanding the right-hand side:
\begin{align}
\text{RHS} &= \frac{1}{V}\sum_{v'=1}^{V}\left(\|\bar{\mathbf{z}}\|_2^2 - 2\bar{\mathbf{z}}^T\mathbf{z}_{n,v'} + \|\mathbf{z}_{n,v'}\|_2^2\right) \\
&= \|\bar{\mathbf{z}}\|_2^2 - \frac{2}{V}\bar{\mathbf{z}}^T\sum_{v'=1}^{V}\mathbf{z}_{n,v'} + \frac{1}{V}\sum_{v'=1}^{V}\|\mathbf{z}_{n,v'}\|_2^2
\end{align}

To complete the proof, we verify that:
\begin{equation}
\frac{1}{V_g}\sum_{v=1}^{V_g}\|\mathbf{z}_{n,v}\|_2^2 = \|\bar{\mathbf{z}}\|_2^2
\end{equation}

Expanding the right-hand side:
\begin{align}
\|\bar{\mathbf{z}}\|_2^2 &= \left\|\frac{1}{V_g}\sum_{v=1}^{V_g}\mathbf{z}_{n,v}\right\|_2^2 \\
&= \frac{1}{V_g^2}\sum_{v=1}^{V_g}\sum_{v''=1}^{V_g}\mathbf{z}_{n,v}^T\mathbf{z}_{n,v''} \\
&= \frac{1}{V_g}\sum_{v=1}^{V_g}\|\mathbf{z}_{n,v}\|_2^2
\end{align}

Therefore, LHS = RHS, completing the proof.
\end{proof}

\subsection{Proof of \cref{thm:kernel_optimal}}
\label{proof:kernel_optimal}
\begin{proof}
For each $x$,
\[
\mathrm{Bias}[\widehat m(x)]=\frac{h^2\mu_2(K)}{2}\Big(\Delta m(x)+2\nabla m(x)^\top \nabla\log p(x)\Big)+o(h^2).
\]
Square and integrate against $p(x)$:
\begin{align*}
\mathcal{B}^2(h;p,m)
&=\Big(\frac{h^2\mu_2(K)}{2}\Big)^2\int \Big(\Delta m(x)+2\nabla m(x)^\top \nabla\log p(x)\Big)^2p(x)dx+o(h^4)\\
&\le \Big(\frac{h^2\mu_2(K)}{2}\Big)^2
\int \Big(2(\Delta m(x))^2+2(2\nabla m(x)^\top \nabla\log p(x))^2\Big)p(x)dx+o(h^4)\\
&=\Big(\frac{h^2\mu_2(K)}{2}\Big)^2\Big(2\int (\Delta m(x))^2p(x)dx+8\int (\nabla m(x)^\top \nabla\log p(x))^2p(x)dx\Big)+o(h^4),
\end{align*}
where we used $(a+b)^2\le 2 a^2+2 b^2$ pointwise.
Since $|\Delta m(x)|\le B$ for all $x$, we have
\[
\int (\Delta m)^2p \le \int B^2p = B^2.
\]
For the second term, first use Cauchy--Schwarz 
and then integrate against $p(x)$ to obtain
\begin{align*}
(\nabla m(x)^\top \nabla\log p(x))^2 \le \|\nabla m(x)\|^2\|\nabla\log p(x)\|^2 \le L^2\|\nabla\log p(x)\|^2\\
\implies \int (\nabla m(x)^\top \nabla\log p(x))^2p(x)dx
\le L^2\int \|\nabla\log p(x)\|^2p(x)dx
= L^2J(p).
\end{align*}
which can be combined with the bounds above to obtain the desired result. We similarly have for the integrated variance
\begin{align*}
\mathcal{V}(h;p)
&=\int \Big(\frac{R(K)}{n h^d}\frac{v(x)}{p(x)}+o\big((n h^d)^{-1}\big)\Big)p(x)dx=\frac{R(K)}{n h^d}\int v(x)dx+o\big((n h^d)^{-1}\big),
\end{align*}
which is independent of $p$.
\end{proof}

\subsection{Proof of \cref{thm:spherical_cramer}}
\label{proof:spherical_cramer}
\begin{proof}
We first start by reminding the reader about the original Cramér-Wold theorem that is a function of all possible directions (not unit-norm ones).

\begin{theorem}[label={def:cramer}]{Cramér-Wold \cite{cramer1936some}}{}
Let $X$ and $Y$ be random vectors in $\mathbb{R}^D$:
\begin{align}
    X \overset{d}{=} Y \iff \langle X, a \rangle \overset{d}{=} \langle Y, a \rangle, \forall \va \in \mathbb{R}^D.
\end{align}
\end{theorem}

Our proof will follow the same proof as for \cref{def:cramer}. Necessity is immediate: if $X \stackrel{d}{=} Y$, then every measurable function of $X$ has the same distribution as the corresponding function of $Y$, from which the linear mapping $x \mapsto \langle u,x\rangle$ for $u \in \mathbb{S}^{d-1}$ is a special case.
For sufficiency, assume $\langle u,X\rangle \stackrel{d}{=} \langle u,Y\rangle$ for all $u \in \mathbb{S}^{d-1}$. Let $\varphi_X(t) := \mathbb{E}\big[e^{i\langle t,X\rangle}\big]$ and $\varphi_Y(t) := \mathbb{E}\big[e^{i\langle t,Y\rangle}\big]$ denote the characteristic functions of $X$ and $Y$. Fix an arbitrary $t \in \mathbb{R}^d$; if $t=0$, then $\varphi_X(0)=\varphi_Y(0)=1$. If $t \neq 0$, write $t = s u$ with $s := \|t\| > 0$ and $u := t/\|t\| \in \mathbb{S}^{d-1}$. By the assumption, $\langle u,X\rangle \stackrel{d}{=} \langle u,Y\rangle$, hence for this $u$ and $s$ we have
\[
\varphi_X(t) \;=\; \mathbb{E}\big[e^{i\langle t,X\rangle}\big]
\;=\; \mathbb{E}\big[e^{i s \langle u,X\rangle}\big]
\;=\; \mathbb{E}\big[e^{i s \langle u,Y\rangle}\big]
\;=\; \mathbb{E}\big[e^{i\langle t,Y\rangle}\big]
\;=\; \varphi_Y(t).
\]
Thus $\varphi_X(t) = \varphi_Y(t)$ for all $t \in \mathbb{R}^d$, i.e., $\varphi_X \equiv \varphi_Y$ on $\mathbb{R}^d$. By the uniqueness theorem for characteristic functions, this implies $X \stackrel{d}{=} Y$.
(ii) Define $\psi_{n,t} := \mathbb{E}\big[e^{i\langle t,X_n\rangle}\big]$ and $\psi_{t} := \mathbb{E}\big[e^{i\langle t,X\rangle}\big]$. Fix $t \in \mathbb{R}^d$ and decompose $t = s u$ with $s := \|t\| \ge 0$ and $u \in \mathbb{S}^{d-1}$ (take, e.g., $u = t/\|t\|$ if $t \neq 0$, and any $u$ if $t=0$). The map $g_s:\mathbb{R}\to\mathbb{R}$, $g_s(x)=s x$, is continuous. By the continuous mapping theorem applied to the real-valued random variables $\langle u,X_n\rangle \xrightarrow{d} \langle u,X\rangle$, we obtain
\[
\langle t,X_n\rangle \;=\; s \langle u,X_n\rangle \xrightarrow{d} s \langle u,X\rangle \;=\; \langle t,X\rangle.
\]
Hence, for every fixed $t \in \mathbb{R}^d$, the one-dimensional projections satisfy $\langle t,X_n\rangle \xrightarrow{d} \langle t,X\rangle$, which in turn yields pointwise convergence of characteristic functions:
\[
\psi_{n,t} \;=\; \mathbb{E}\big[e^{i\langle t,X_n\rangle}\big] \;\longrightarrow\; \mathbb{E}\big[e^{i\langle t,X\rangle}\big] \;=\; \psi_{t}, \qquad \text{for all } t \in \mathbb{R}^d.
\]
Therefore, by L\'evy's continuity theorem, $X_n \xrightarrow{d} X$.
This completes the proof.
\end{proof}

\subsection{Proof of \cref{thm:bcs}}
\label{proof:bcs}
\begin{proof}
We first formulate the following assumptions required for the proof--all of this are satisfied by typical univariate statistical tests.

$P=Q$ if and only if $P_a=Q_a$ for all $a\in S^{d-1}$ (population-level equivalence of laws).

$A_n$ are finite sets with mesh $\Delta(A_n):=\sup_{u\in S^{d-1}} \min_{a\in A_n}\|u-a\| \to 0$ as $n\to\infty$.

If $P\neq Q$, there exists a separating direction $a^\star\in S^{d-1}$ and a neighborhood $U$ of $a^\star$ such that
\[
\inf_{a\in U}\lim_{n\to\infty}\Pr\big(T_{a,n} \ge u_n(\alpha)\big)=1.
\]
(Intuitively: near a truly separating direction, the 1D statistic eventually exceeds the global null threshold with probability $\to 1$.)

(i) Under $H_0:P=Q$, our assumption implies no separating direction exists at the population level, and the calibration of $u_n(\alpha)$ ensures
$\Pr(M_n \ge u_n(\alpha)) \le \alpha$ for all $n$, hence
$\limsup_{n\to\infty}\Pr(\Psi_n=1)\le \alpha$.
(ii) Suppose $P\neq Q$.
Our assumption guarantees that there exists at least one separating direction $a^\star$ with $P_{a^\star}\neq Q_{a^\star}$.
Our assumption  guarantees a neighborhood $U$ of $a^\star$ in which the projection statistics exceed the global null threshold with probability tending to 1:
\[
\inf_{a\in U}\lim_{n\to\infty}\Pr\big(T_{a,n} \ge u_n(\alpha)\big) \;=\; 1.
\]
By assumption, for all large $n$ the set $A_n$ contains at least one direction $a_n\in U$ (dense coverage).
Therefore,
\[
\Pr(\Psi_n=1) \;=\; \Pr\big(M_n \ge u_n(\alpha)\big)
\;\ge\; \Pr\big( T_{a_n,n} \ge u_n(\alpha) \big)
\;\longrightarrow\; 1,
\]
which proves consistency.
\end{proof}

\subsection{Proof of \cref{thm:spherical_bounds}}
\label{proof:spherical_bounds}
\begin{proof}
For each case, consider the function $g(a)$ on $\mathbb{S}^{D-1}$ defined by the quantity of interest (CF, CDF, or moment) at a fixed $t$ or $k$. Since $f \in H^\alpha(\mathbb{R}^D)$, the mapping $a \mapsto g(a)$ is in $H^\alpha(\mathbb{S}^{D-1})$ for each fixed $t$ or $k$.

Given $M$ samples $\{a_i\}_{i=1}^M$ on the sphere, the best possible reconstruction of $g$ from its values at these points is given by spherical interpolation. By classical results on Sobolev spaces and spherical harmonics (see, e.g., \cite{narcowich2006localized}), the $L^2$ interpolation error for functions in $H^\alpha(\mathbb{S}^{D-1})$ using $M$ points is bounded by
\[
\mathbb{E}_b \left[ |g(b) - g^*(b)|^2 \right] \leq C(D, \alpha) M^{-2\alpha/(D-1)} \| g \|_{H^\alpha(\mathbb{S}^{D-1})}^2,
\]
where $g^*$ is the interpolant matching $g$ at the $M$ sampled points.
The interpolation error bound on the sphere follows from the theory of spherical harmonics and Marcinkiewicz–Zygmund (MZ) inequalities . Any $f \in H^\alpha(\mathbb{S}^d)$ admits a spherical harmonics expansion, and the best $L^2$ approximation by harmonics of degree at most $L$ satisfies
\[
\|f - P_L f\|_{L^2(\mathbb{S}^d)} \leq (1 + L^2)^{-\alpha/2} \|f\|_{H^\alpha(\mathbb{S}^d)},
\]
where $P_L f$ is the projection onto harmonics of degree $\leq L$ \cite[Lemma~2.1]{narcowich2006localized}. If $M$ points are distributed quasi-uniformly on $\mathbb{S}^d$, then for $L \sim c M^{1/d}$, the set forms a Marcinkiewicz–Zygmund (MZ) set for degree $L$ \cite[Theorem 1.1]{mhaskar2001spherical}. This allows reconstruction of any function in the space of harmonics of degree at most $L$ from its values at these points, and the $L^2$ interpolation error for $f$ is bounded by
\[
\|f - I_M f\|_{L^2(\mathbb{S}^d)} \leq C (1 + L^2)^{-\alpha/2} \|f\|_{H^\alpha(\mathbb{S}^d)},
\]
where $I_M f$ is any interpolant matching $f$ at the $M$ points \cite[Theorem 3.1]{narcowich2006localized}. Substituting $L \sim c M^{1/d}$ yields the rate $M^{-\alpha/d}$, and thus
\[
\mathbb{E}_{\omega} |f(\omega) - I_M f(\omega)|^2 \leq C(d, \alpha) M^{-2\alpha/d} \|f\|_{H^\alpha(\mathbb{S}^d)}^2,
\]
with explicit $C(d, \alpha)$ as in the main theorem.
Integrating (or summing) over $t$ (for CF and CDF) or $k$ (for moments, with weights $w_k$) yields the stated bounds. The explicit constant $C(D, \alpha)$ arises from the theory of spherical Sobolev spaces and is given above.

For the moment case, the sum over $k$ is weighted to ensure convergence, as higher moments may grow rapidly. The weights $w_k$ can be chosen, for example, as $w_k = 1/k!$.

This completes the proof.
\end{proof}

\subsection{Proof of \cref{thm:moment_conendrum}}
\label{proof:moment_conendrum}

Pick distinct $x_0,\dots,x_{K+1}\in\mathbb{R}$ and consider the linear map
$A:\mathbb{R}^{K+2}\to\mathbb{R}^{K+1}$, $(Ap)_r=\sum_{j=0}^{K+1} p_j x_j^r$ for $r=0,\dots,K$.
Then $\mathrm{rank}(A)\le K+1$, so $\ker(A)\neq\{0\}$. Let $v\in\ker(A)\setminus\{0\}$; from $(Ap)_0=\sum_j p_j$, we get $\sum_j v_j=0$, hence $v$ has positive and negative entries. Choose a strictly positive probability vector $p$ and $\varepsilon>0$ small such that $p^\pm:=p\pm\varepsilon v$ remain probability vectors. Then $Ap^+=Ap^-$, so the distributions supported on $\{x_j\}$ with masses $p^\pm$ are distinct yet match moments up to order $K$.

\subsection{Proof of \cref{thm:ecf_stability}}
\label{proof:ecf_stability}
\begin{proof}

Fix the Gaussian weight
\[
w_s(t)=e^{-s^2 t^2},\qquad s>0,
\]
and define the population CF distance
\[
D(P,G)=\int_{\mathbb{R}} w_s(t)\big|\varphi_P(t)-\varphi_G(t)\big|^2dt.
\]
Let the empirical CF be
\[
\widehat{\varphi}_N(t)=\frac{1}{N}\sum_{i=1}^N e^{itX_i},
\]
and consider the V-statistic estimator
\[
\widehat{D}_V=\int_{\mathbb{R}} w_s(t)\big|\widehat{\varphi}_N(t)-\varphi_G(t)\big|^2dt.
\]
We use only that $|e^{itX}|=1$, $|\varphi_P(t)|\le 1$, $|\varphi_G(t)|\le 1$, and integrability of $w_s$.
For each $i$ differentiate under the integral (dominated convergence applies because the integrand and its derivative are bounded)
\begin{align*}
\frac{\partial \widehat{D}_V}{\partial X_i}
=& \int_{\mathbb{R}} w_s(t)2\Re\!\Big(\big(\widehat{\varphi}_N(t)-\varphi_G(t)\big)\overline{\frac{\partial \widehat{\varphi}_N(t)}{\partial X_i}}\Big)dt,\\
\frac{\partial \widehat{\varphi}_N(t)}{\partial X_i}
=& \frac{1}{N}i te^{itX_i},\end{align*}
since $|\widehat{\varphi}_N(t)|\le 1$ and $|\varphi_G(t)|\le 1$,
\begin{align*}
\left|\frac{\partial \widehat{D}_V}{\partial X_i}\right|
&\le \frac{2}{N}\int w_s(t)|t|\big(|\widehat{\varphi}_N(t)|+|\varphi_G(t)|\big)dt\\
&\le \frac{4}{N}\int w_s(t)|t|dt\\
&= \frac{4}{Ns^2},
\end{align*}
using $\int_{\mathbb{R}} e^{-s^2 t^2}|t|dt=1/s^2$. 
\[
\Bigg|\frac{\partial \widehat{D}_V}{\partial X_i}\Bigg|
\;\le\;\frac{4}{N}\int_{\mathbb{R}} w_s(t)|t|dt
\;=\; \frac{4}{Ns^2}.
\]
Moreover, differentiating once more in $X_i$ and using $|\widehat{\varphi}_N(t)|\le 1$, $|\varphi_G(t)|\le 1$ gives a global Lipschitz bound
\[
\Bigg|\frac{\partial^2 \widehat{D}_V}{\partial X_i^2}\Bigg|
\;\le\; \frac{C}{N}\int_{\mathbb{R}} w_s(t)t^2dt
\;=\; \frac{C}{N}\cdot \frac{\sqrt{\pi}}{2s^3},
\]
for some absolute constant $C$ arising from bounded factors and product rule. Hence ECF gradients are uniformly bounded and Lipschitz, with scale controlled only by $(N,s)$.
\item[(B)] (Moment sample-gradients are polynomial in $X_i$ and unbounded for $k\ge 2$.) Let $\widehat{D}_V$ be as above. Define the moment objective
\[
\widehat{D}_k
\;=\;
(\bar{\phi}-\mu)^\top W(\bar{\phi}-\mu),\qquad
\bar{\phi}:=\frac{1}{N}\sum_{i=1}^N \phi(X_i),\quad
\phi(x)=(x,x^2,\dots,x^k)^\top,
\]
for a symmetric positive semidefinite $W\in\mathbb{R}^{k\times k}$ and Gaussian target moments $\mu=\mathbb{E}_G[\phi(Y)]$. For each $i$,
\begin{align*}
\frac{\partial \widehat{D}_k}{\partial X_i}
=& \frac{2}{N}(\bar{\phi}-\mu)^\top W\frac{\partial \phi(X_i)}{\partial X_i},\\
\frac{\partial \phi(X)}{\partial X}
=& \big(1,2X,3X^2,\dots,k X^{k-1}\big)^\top.
\end{align*}
The gradient formula follows by the chain rule and linearity of $\bar{\phi}$. Let $c:=W(\bar{\phi}-\mu)$ and write $c_r$ for its $r$-th coordinate. Then
\[
\frac{\partial \widehat{D}_k}{\partial X_i}
= \frac{2}{N}\sum_{r=1}^k c_rrX_i^{r-1},
\]
which is a polynomial in $X_i$ of degree $\deg=\max\{r-1: c_r\neq 0\}\le k-1$. In particular, if $c_k\neq 0$ (the generic case when the top-weighted deviation is nonzero), then
\[
\left|\frac{\partial \widehat{D}_k}{\partial X_i}\right|
\;\xrightarrow[|X_i|\to\infty]{}\; \infty
\quad\text{as}\quad |X_i|^{k-1}.
\]
The expression is a nonconstant polynomial in $X_i$ of degree $\deg\le k-1$ whenever some $c_r\neq 0$ with $r\ge 2$. Thus the gradient cannot be uniformly bounded on $\mathbb{R}$. If $c_k\neq 0$, the leading term dominates and the magnitude grows like $|X_i|^{k-1}$, proving unboundedness for $k\ge 2$.
\end{proof}

\subsection{Proof of \cref{thm:gradient_bias}}
\label{proof:gradient_bias}
\begin{proof}


A direct calculation shows
Fix $t \in \mathbb{R}^d$ and abbreviate $Z_j \coloneqq e^{\mathrm{i} t^\top X_j}$, so that $\phi_n(t) = \frac{1}{n}\sum_{j=1}^n Z_j$. Note that $|Z_j|=1$ almost surely (since $t^\top X_j\in\mathbb{R}$), and $\mathbb{E}[Z_j]=\phi_\theta(t)$ for all $j$.
We start from the algebraic identity
\[
\big|\phi_n(t) - \psi(t)\big|^2
=
\phi_n(t)\overline{\phi_n(t)}
-
\psi(t)\overline{\phi_n(t)}
-
\overline{\psi(t)}\phi_n(t)
+
\big|\psi(t)\big|^2.
\]
Taking expectations term by term gives
\begin{align}
\mathbb{E}\left[ \big|\phi_n - \psi\big|^2 \right]
=&
\mathbb{E}\left[ |\phi_n|^2 \right]
-
\psi\mathbb{E}\left[ \overline{\phi_n} \right]
-
\overline{\psi}\mathbb{E}\left[ \phi_n \right]
+
|\psi|^2,\\
=&
\mathbb{E}\left[ |\phi_n|^2 \right]
- \psi\overline{\mathbb{E}[\phi_n]}-\overline{\psi}\frac{1}{n}\sum_{j=1}^n \mathbb{E}[Z_j] +
|\psi|^2,\\
=&
\mathbb{E}\left[ |\phi_n|^2 \right]
-\psi\overline{\phi_\theta} - \overline{\psi}\phi_\theta+
|\psi|^2,\\
=&\mathbb{E}\left[ |\phi_n|^2 \right]
-2\mathrm{Re}\big( \overline{\psi}\phi_\theta \big)+
|\psi|^2,\\
=&\mathbb{E}\left[ \left|\frac{1}{n}\sum_{j=1}^n Z_j\right|^2 \right]-2\mathrm{Re}\big( \overline{\psi}\phi_\theta \big)+
|\psi|^2,\\
=&\frac{1}{n^2}\sum_{j=1}^n \sum_{l=1}^n \mathbb{E}\left[ Z_j\overline{Z_l} \right]-2\mathrm{Re}\big( \overline{\psi}\phi_\theta \big)+
|\psi|^2,\\
\end{align}
Since the $Z_j$ are i.i.d.,
\[
\mathbb{E}\left[ Z_j\overline{Z_l} \right]
=
\begin{cases}
\mathbb{E}\left[ |Z_1|^2 \right] = 1, & \text{if } j=l,\\[4pt]
\mathbb{E}[Z_j]\overline{\mathbb{E}[Z_l]} = \phi_\theta\overline{\phi_\theta} = |\phi_\theta|^2, & \text{if } j\neq l,
\end{cases}
\]
hence
\begin{align*}
    \mathbb{E}\left[ |\phi_n|^2 \right]
=&
\frac{1}{n^2}\Big( n + n(n-1) |\phi_\theta|^2 \Big)\\
=&
\frac{1}{n} + \left(1-\frac{1}{n}\right)|\phi_\theta|^2\\
=&
|\phi_\theta|^2 + \frac{1-|\phi_\theta|^2}{n}
\end{align*}
Plugging these, we obtain
\begin{align*}
\mathbb{E}\left[ \big|\phi_n - \psi\big|^2 \right]
&=
\left( |\phi_\theta|^2 + \frac{1-|\phi_\theta|^2}{n} \right)
-
2\mathrm{Re}\big( \overline{\psi}\phi_\theta \big)
+
|\psi|^2
\\[4pt]
&=
\big( |\phi_\theta|^2 - 2\mathrm{Re}\big( \overline{\psi}\phi_\theta \big) + |\psi|^2 \big)
+
\frac{1-|\phi_\theta|^2}{n}
\\[4pt]
&=
\big|\phi_\theta - \psi\big|^2 + \frac{1-|\phi_\theta|^2}{n}.
\end{align*} 
Under Dominated convergence, $\E[\nabla_\theta D_n(t)] = \nabla_\theta \E[D_n(t)]$, hence
\[
\E\left[\nabla_\theta D_n(t)\right]
= \nabla_\theta \big|\phi_\theta(t)-\psi(t)\big|^2
+ \nabla_\theta \frac{1-|\phi_\theta(t)|^2}{n},
\]
concluding the proof.

In practice one replaces $\int_{\mathbb{R}} w(t)(\cdot)dt$ by a deterministic quadrature on a uniform grid $t_k\in[-T,T]$ with weights $\omega_k$ (e.g.\ trapezoidal rule) and a Gaussian window $w(t)=e^{-\alpha t^2}$. All statements above remain valid with the integral replaced by $\sum_k \omega_k (\cdot)$:
\[
L(\theta) \approx \sum_{k} \omega_k\big|\phi_\theta(t_k)-\psi(t_k)\big|^2,
\quad
\widehat{L}_n(\theta) \approx \sum_{k} \omega_k\big|\phi_n(t_k)-\psi(t_k)\big|^2,
\]
and the bias term becomes
\[
\text{Bias}(\theta) = -\frac{1}{n}\sum_k \omega_k\nabla_\theta \big|\phi_\theta(t_k)\big|^2.
\]
Since the grid and weights are deterministic, they do not affect unbiasedness with respect to sampling; they only introduce a deterministic approximation error to the target functional $L(\theta)$.

\end{proof}

\subsection{Proof of VICReg's Recovery}
\label{proof:vcreg}
\begin{proof}
We prove this result in two parts.
\paragraph{Part I: $\mathbb{E}[\mathbf{X}] = \mathbf{0}$}
Given that $\mathbb{E}[\langle \mathbf{X}, \mathbf{a} \rangle] = 0$ for all unit vectors $\mathbf{a}$, and noting that $\langle \mathbf{X}, \mathbf{a} \rangle = \mathbf{a}^T \mathbf{X}$, we have:
\begin{equation}\label{eq:mean_condition}
\mathbb{E}[\mathbf{a}^T \mathbf{X}] = 0 \quad \text{for all } \mathbf{a} \in \mathbb{R}^d \text{ with } \|\mathbf{a}\| = 1
\end{equation}
By linearity of expectation:
\begin{equation}
\mathbf{a}^T \mathbb{E}[\mathbf{X}] = 0 \quad \text{for all unit vectors } \mathbf{a}
\end{equation}
Let $\boldsymbol{\mu} = \mathbb{E}[\mathbf{X}]$. We claim that $\boldsymbol{\mu} = \mathbf{0}$.
Suppose, for the sake of contradiction, that $\boldsymbol{\mu} \neq \mathbf{0}$. Then $\|\boldsymbol{\mu}\|_2 > 0$. Define the unit vector:
\begin{equation}
\mathbf{a}^* = \frac{\boldsymbol{\mu}}{\|\boldsymbol{\mu}\|_2}
\end{equation}
Since $\mathbf{a}^*$ is a unit vector, equation \eqref{eq:mean_condition} implies:
\begin{equation}
(\mathbf{a}^*)^T \boldsymbol{\mu} = 0
\end{equation}
However, substituting the definition of $\mathbf{a}^*$:
\begin{equation}
(\mathbf{a}^*)^T \boldsymbol{\mu} = \left(\frac{\boldsymbol{\mu}}{\|\boldsymbol{\mu}\|_2}\right)^T \boldsymbol{\mu} = \frac{\boldsymbol{\mu}^T \boldsymbol{\mu}}{\|\boldsymbol{\mu}\|_2} = \frac{\|\boldsymbol{\mu}\|_2^2}{\|\boldsymbol{\mu}\|_2} = \|\boldsymbol{\mu}\|_2 > 0
\end{equation}
This contradiction establishes that $\boldsymbol{\mu} = \mathbf{0}$.
\paragraph{Part II: $\mathrm{Cov}(\mathbf{X}) = \mathbf{I}_d$}
Since $\mathbb{E}[\mathbf{X}] = \mathbf{0}$, we have:
\begin{equation}
\mathrm{Var}(\langle \mathbf{X}, \mathbf{a} \rangle) = \mathbb{E}[(\langle \mathbf{X}, \mathbf{a} \rangle)^2] = \mathbb{E}[(\mathbf{a}^T \mathbf{X})^2]
\end{equation}
Expanding the quadratic form:
\begin{equation}
\mathbb{E}[(\mathbf{a}^T \mathbf{X})^2] = \mathbb{E}[\mathbf{a}^T \mathbf{X} \mathbf{X}^T \mathbf{a}] = \mathbf{a}^T \mathbb{E}[\mathbf{X} \mathbf{X}^T] \mathbf{a}
\end{equation}
Since $\mathbb{E}[\mathbf{X}] = \mathbf{0}$, the covariance matrix is $\mathrm{Cov}(\mathbf{X}) = \mathbb{E}[\mathbf{X} \mathbf{X}^T]$. Let $\boldsymbol{\Sigma} = \mathrm{Cov}(\mathbf{X})$. The variance condition gives us:
\begin{equation}\label{eq:quadratic_form}
\mathbf{a}^T \boldsymbol{\Sigma} \mathbf{a} = 1 \quad \text{for all unit vectors } \mathbf{a}
\end{equation}
We now show that $\boldsymbol{\Sigma} = \mathbf{I}_d$.
\emph{Step 1: Diagonal entries.} For $i \in \{1, 2, \ldots, d\}$, let $\mathbf{e}_i$ denote the $i$-th standard basis vector. Setting $\mathbf{a} = \mathbf{e}_i$ in equation \eqref{eq:quadratic_form}:
\begin{equation}
\mathbf{e}_i^T \boldsymbol{\Sigma} \mathbf{e}_i = \Sigma_{ii} = 1
\end{equation}
Therefore, all diagonal entries of $\boldsymbol{\Sigma}$ equal 1.
\emph{Step 2: Off-diagonal entries.} For distinct indices $i, j \in \{1, 2, \ldots, d\}$, consider the unit vector:
\begin{equation}
\mathbf{a} = \frac{\mathbf{e}_i + \mathbf{e}_j}{\|\mathbf{e}_i + \mathbf{e}_j\|_2} = \frac{\mathbf{e}_i + \mathbf{e}_j}{\sqrt{2}}
\end{equation}
Applying equation \eqref{eq:quadratic_form}:
\begin{equation}
\mathbf{a}^T \boldsymbol{\Sigma} \mathbf{a} = \frac{1}{2}(\mathbf{e}_i + \mathbf{e}_j)^T \boldsymbol{\Sigma} (\mathbf{e}_i + \mathbf{e}_j) = 1
\end{equation}
Expanding the quadratic form and using the symmetry of $\boldsymbol{\Sigma}$:
\begin{align}
\frac{1}{2}(\mathbf{e}_i^T \boldsymbol{\Sigma} \mathbf{e}_i + 2\mathbf{e}_i^T \boldsymbol{\Sigma} \mathbf{e}_j + \mathbf{e}_j^T \boldsymbol{\Sigma} \mathbf{e}_j) &= 1\\
\frac{1}{2}(\Sigma_{ii} + 2\Sigma_{ij} + \Sigma_{jj}) &= 1\\
\frac{1}{2}(1 + 2\Sigma_{ij} + 1) &= 1\\
1 + \Sigma_{ij} &= 1\\
\Sigma_{ij} &= 0
\end{align}
Therefore, all off-diagonal entries of $\boldsymbol{\Sigma}$ equal zero, establishing that $\boldsymbol{\Sigma} = \mathbf{I}_d$.
\end{proof}

\section{Background}

\textbf{Foundation: The Linear Regression Model}
We start with the standard linear regression model:
$$\mathbf{y} = \mathbf{X}\boldsymbol{\beta} + \boldsymbol{\varepsilon}$$
where:
\begin{itemize}
    \item $\mathbf{y} = [y_1, y_2, \ldots, y_n]^T \in \mathbb{R}^n$ is the response vector
    \item $\mathbf{X} \in \mathbb{R}^{n \times p}$ is the design matrix with $\mathbf{X}_{ij} = x_{ij}$
    \item $\boldsymbol{\beta} = [\beta_1, \beta_2, \ldots, \beta_p]^T \in \mathbb{R}^p$ is the parameter vector
    \item $\boldsymbol{\varepsilon} = [\varepsilon_1, \varepsilon_2, \ldots, \varepsilon_n]^T \sim \mathcal{N}(\mathbf{0}, \sigma^2\mathbf{I}_n)$ is the error vector
\end{itemize}
The error assumption means:
$$\mathbb{E}[\varepsilon_i] = 0, \quad \text{Var}(\varepsilon_i) = \sigma^2, \quad \text{Cov}(\varepsilon_i, \varepsilon_j) = 0 \text{ for } i \neq j$$
\textbf{Step 1: Deriving the OLS Estimator}
To find the OLS estimator, we minimize the sum of squared residuals:
$$\text{SSR}(\boldsymbol{\beta}) = \sum_{i=1}^n (y_i - \mathbf{x}_i^T\boldsymbol{\beta})^2 = (\mathbf{y} - \mathbf{X}\boldsymbol{\beta})^T(\mathbf{y} - \mathbf{X}\boldsymbol{\beta})$$
Expanding this quadratic form:
\begin{align}
\text{SSR}(\boldsymbol{\beta}) &= \mathbf{y}^T\mathbf{y} - 2\boldsymbol{\beta}^T\mathbf{X}^T\mathbf{y} + \boldsymbol{\beta}^T\mathbf{X}^T\mathbf{X}\boldsymbol{\beta}
\end{align}
Taking the derivative with respect to $\boldsymbol{\beta}$:
$$\frac{\partial \text{SSR}}{\partial \boldsymbol{\beta}} = -2\mathbf{X}^T\mathbf{y} + 2\mathbf{X}^T\mathbf{X}\boldsymbol{\beta}$$
Setting equal to zero and solving:
$$-2\mathbf{X}^T\mathbf{y} + 2\mathbf{X}^T\mathbf{X}\boldsymbol{\beta} = \mathbf{0}$$
$$\mathbf{X}^T\mathbf{X}\boldsymbol{\beta} = \mathbf{X}^T\mathbf{y}$$
Assuming $\mathbf{X}^T\mathbf{X}$ is invertible:
$$\boxed{\hat{\boldsymbol{\beta}} = (\mathbf{X}^T\mathbf{X})^{-1}\mathbf{X}^T\mathbf{y}}$$

\begin{figure*}[t!]
    \centering
    \includegraphics[width=0.49\linewidth]{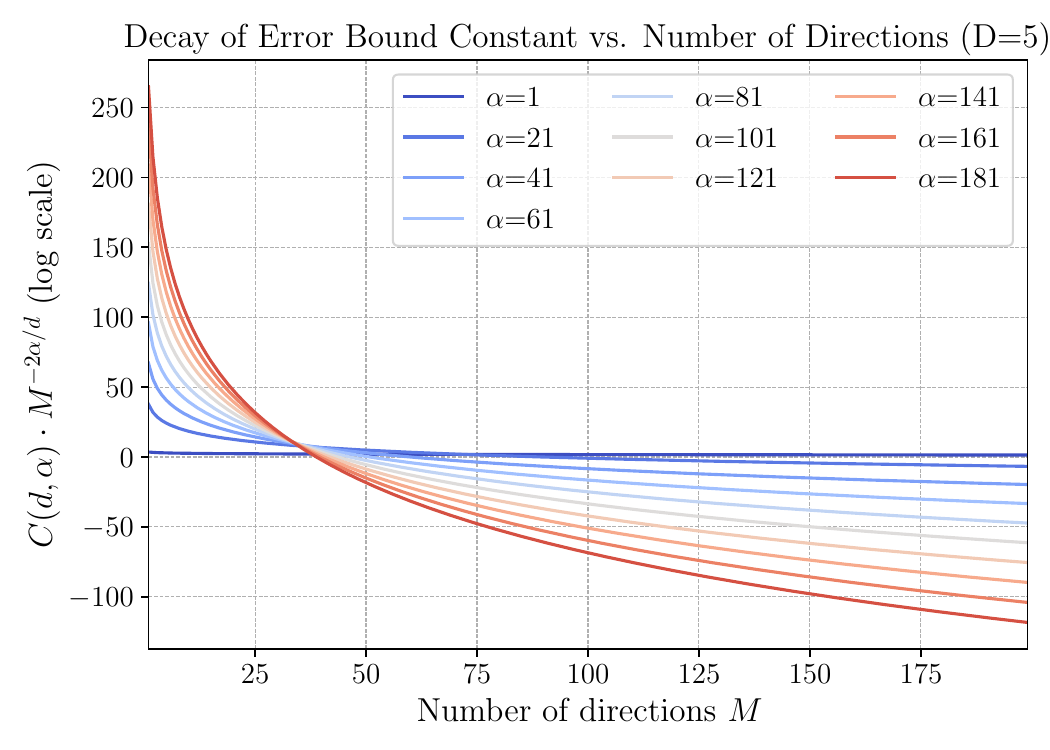}
    \caption{\small Depiction of the expected BCS loss upper bound (\cref{thm:spherical_bounds}) for various smoothness values $\alpha$. We clearly see that as the smoothness increases ({\bf blue to red}), as the upper bound decreases more and more rapidly with $M$.}
    \label{fig:bound_example}
\end{figure*}

\begin{figure}[t!]
\centering
    \centering
    dimension=128, slices=10\\
    \includegraphics[width=0.9\linewidth]{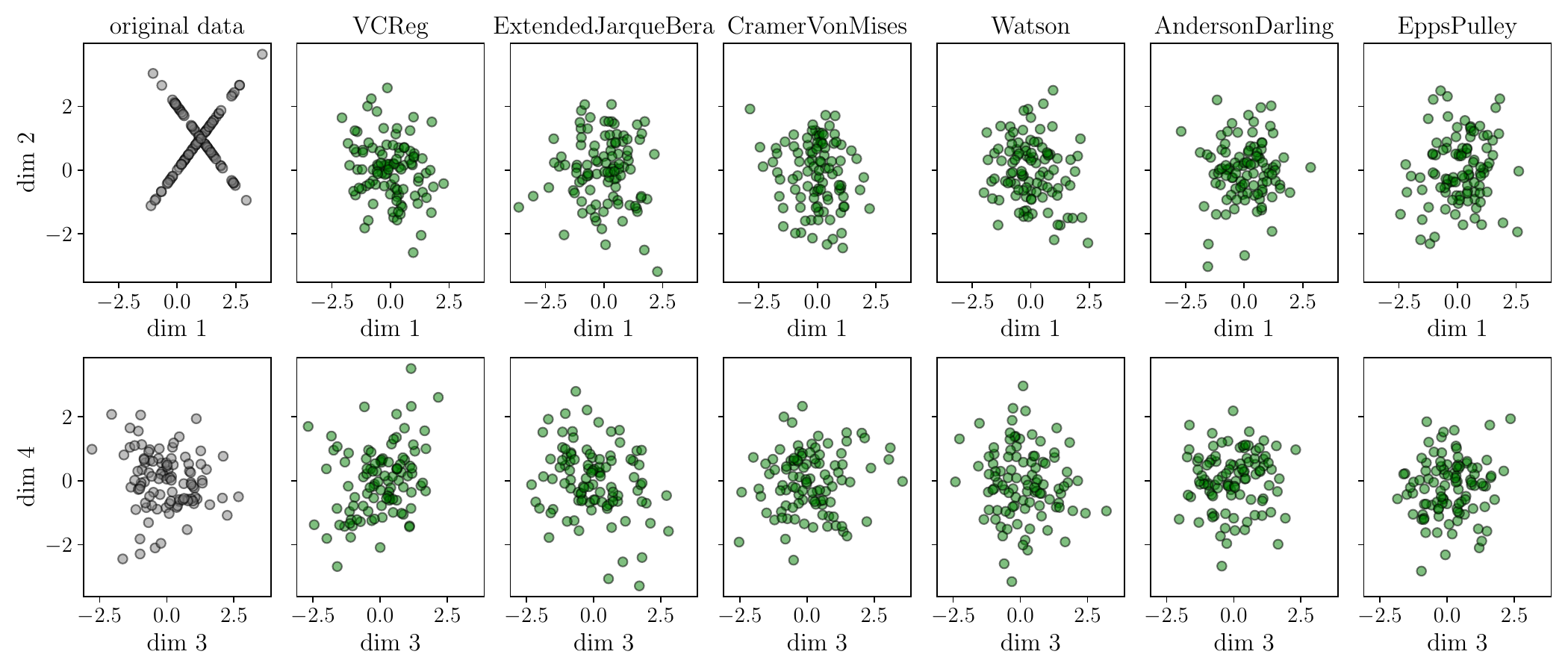}
    dimension=128, slices=100\\
    \includegraphics[width=0.9\linewidth]{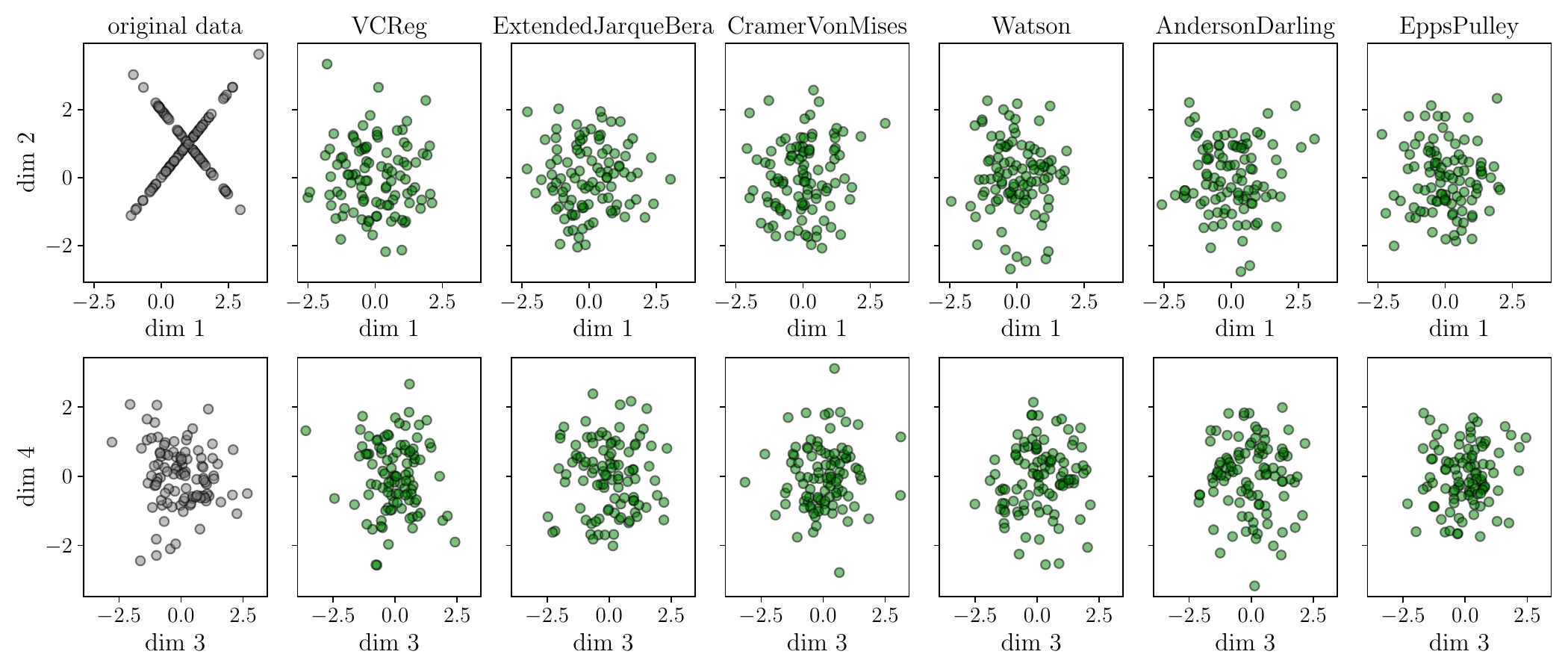}
    dimension=1024, slices=100\\
    \includegraphics[width=0.9\linewidth]{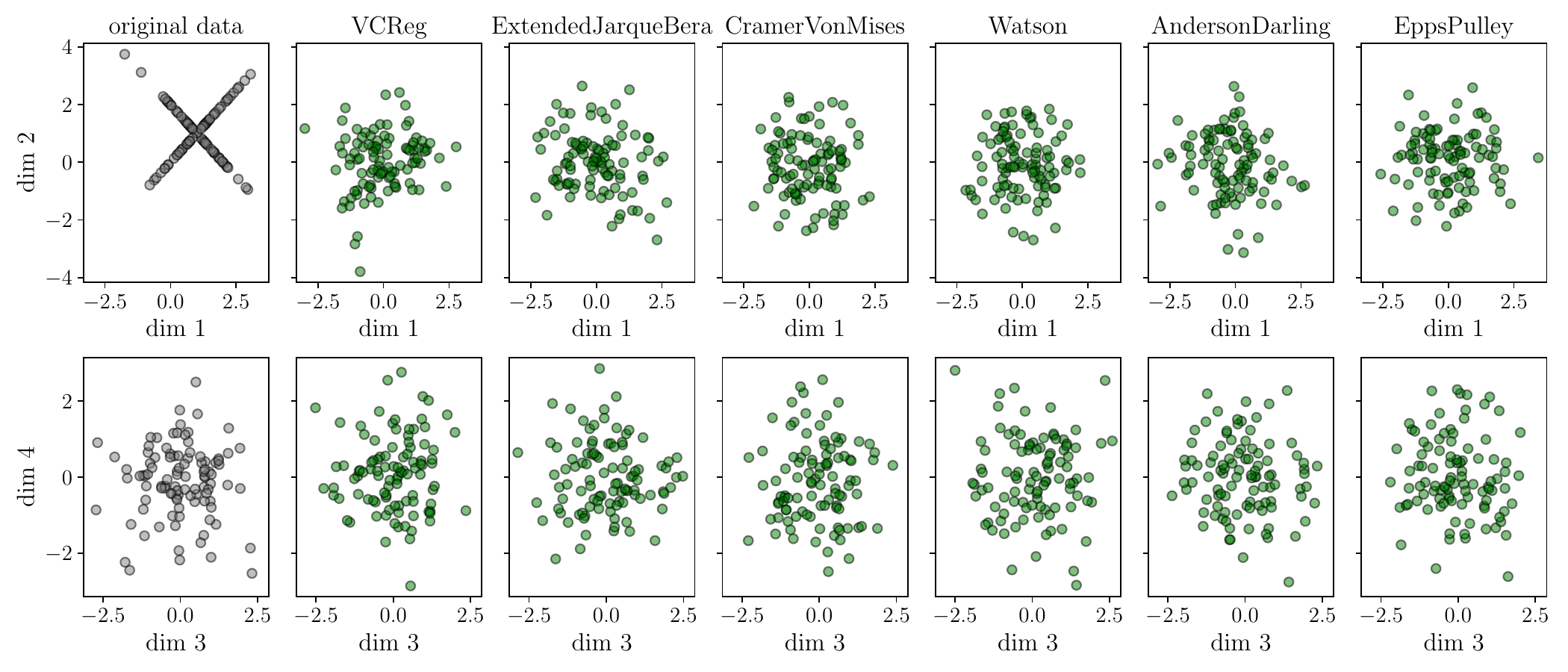}
    \caption{Reprise of \cref{fig:nonparametric} for additional dimensions and number of 1d projections.}
    \label{fig:extra_nonparametric}
\end{figure}

\begin{figure}[t!]
    \centering
    \includegraphics[width=0.8\linewidth]{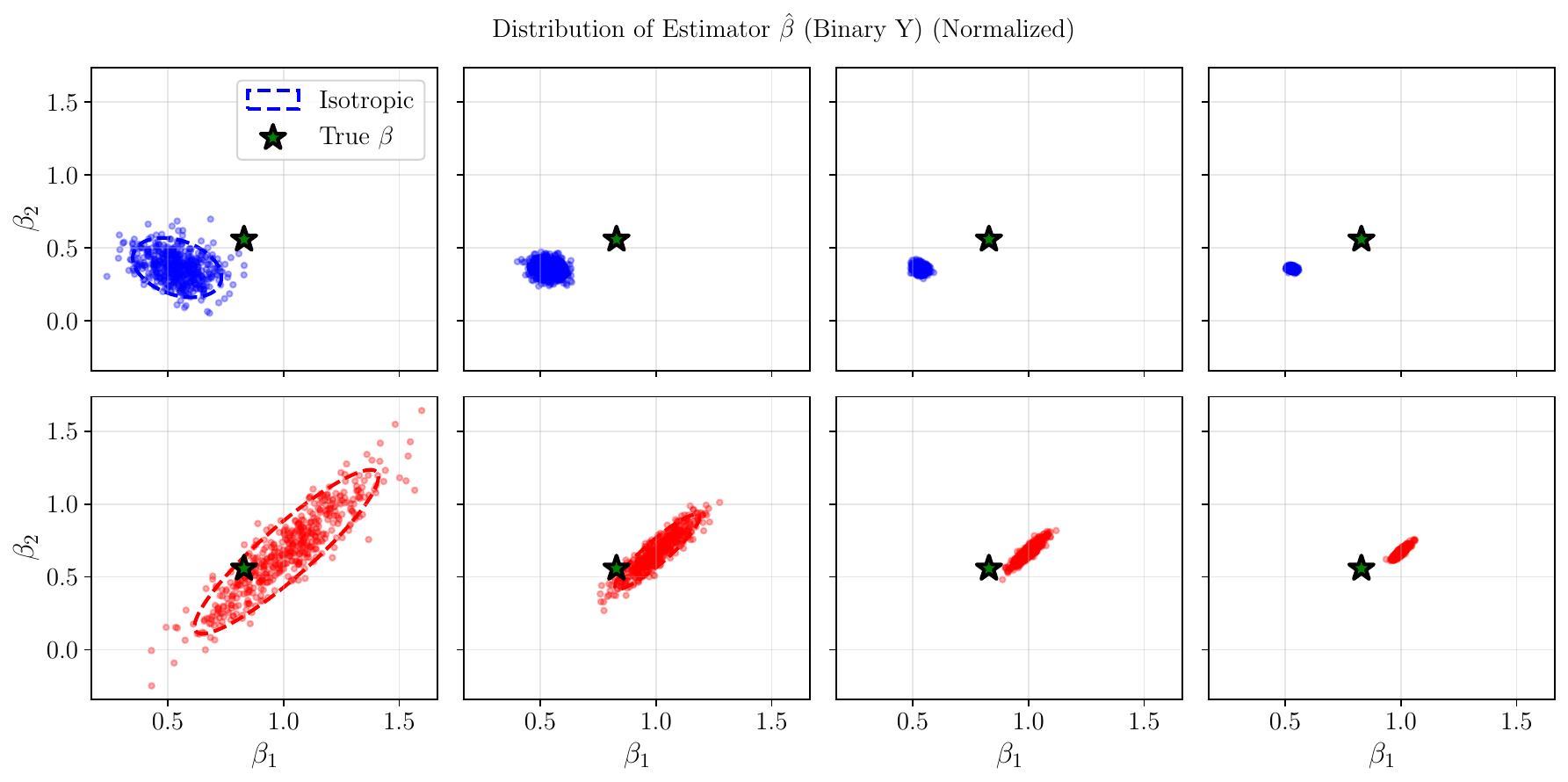}
    \includegraphics[width=0.8\linewidth]{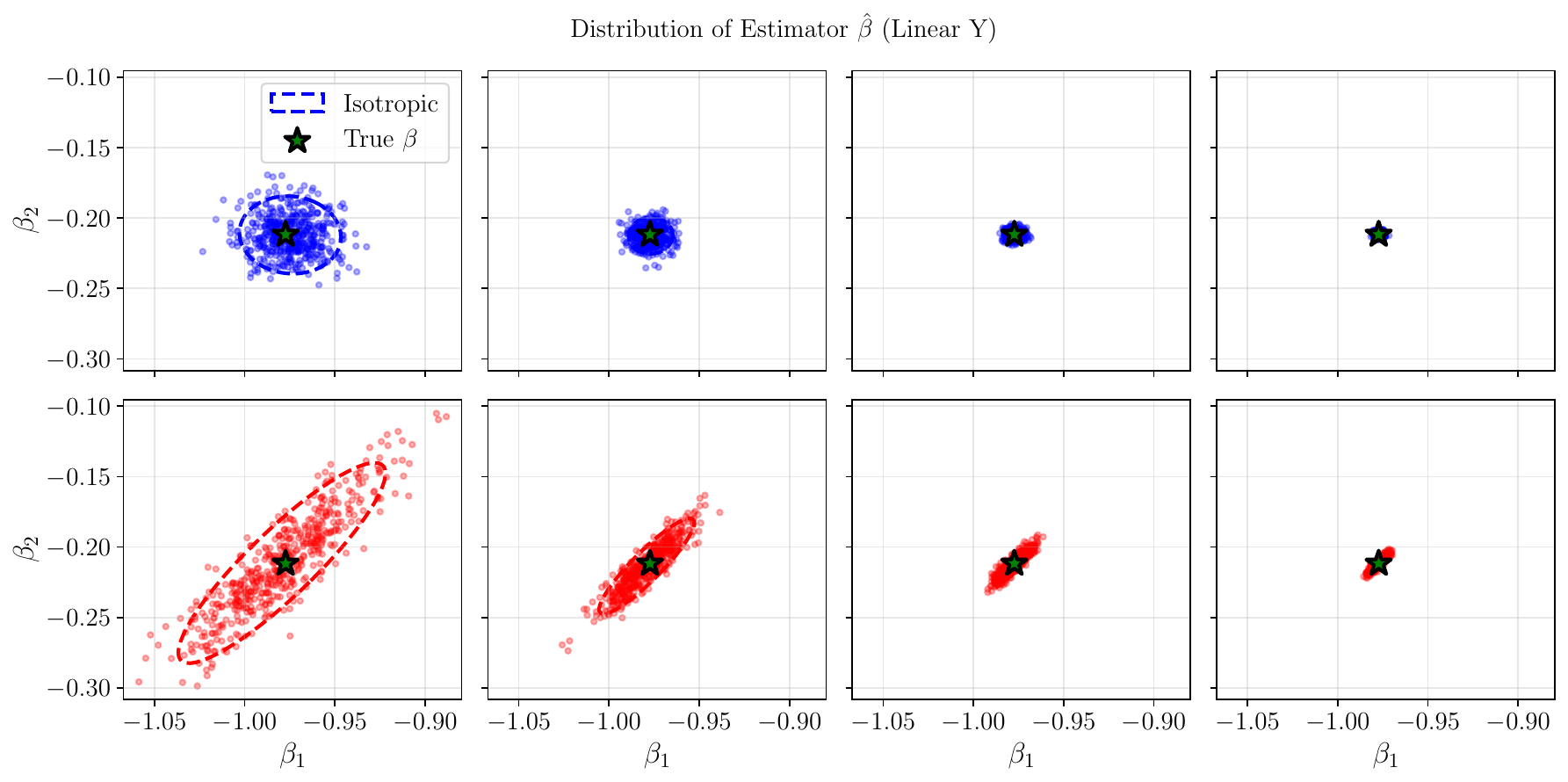}
    \includegraphics[width=0.8\linewidth]{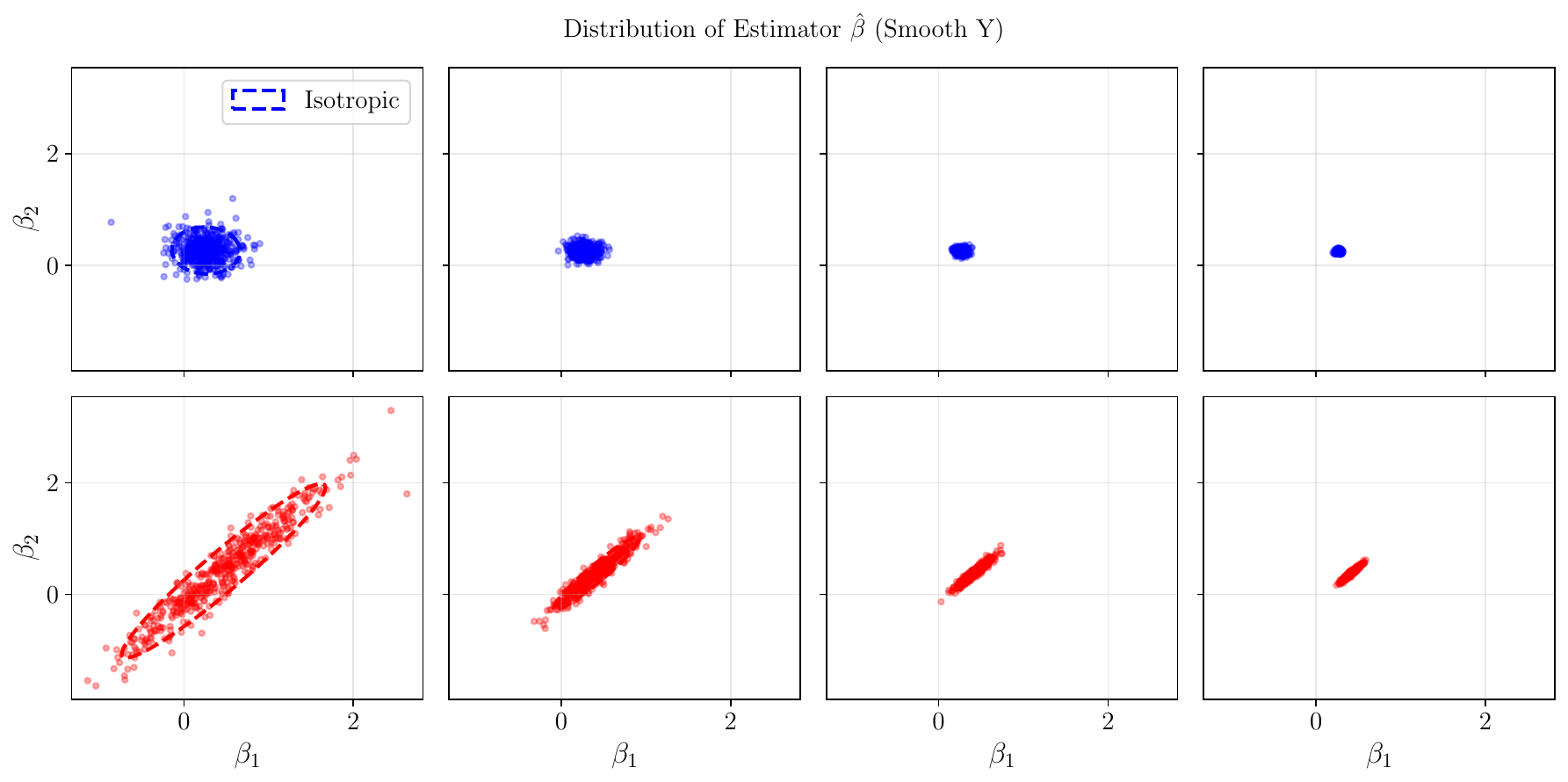}
    \caption{Depiction of the distribution of optimized $\beta$ values from OLS when comparing $\mZ_{\rm iso}$ and $\mZ_{\rm aniso}$ from \cref{thm:linear_probe_bias,thm:linear_probe_variance}. We clearly observe that the anisotropic version ({\bf blue}) provides much lower variance compared to the isotropic case ({\bf red}). We consider a binary classification (linear separable class) ({\bf top row}), a linear regression task ({\bf middle row}), and a nonlinear regression task with smooth targets ({\bf bottom row}). For each case, we resample the training samples numerous times and produce an estimate for $\beta$ each time. Because the data is $2$-dimensional, we can visualize the $\beta$ distribution directly.}
    \label{fig:beta_distributions}
\end{figure}

\begin{figure}
    \centering
    \includegraphics[width=0.32\linewidth]{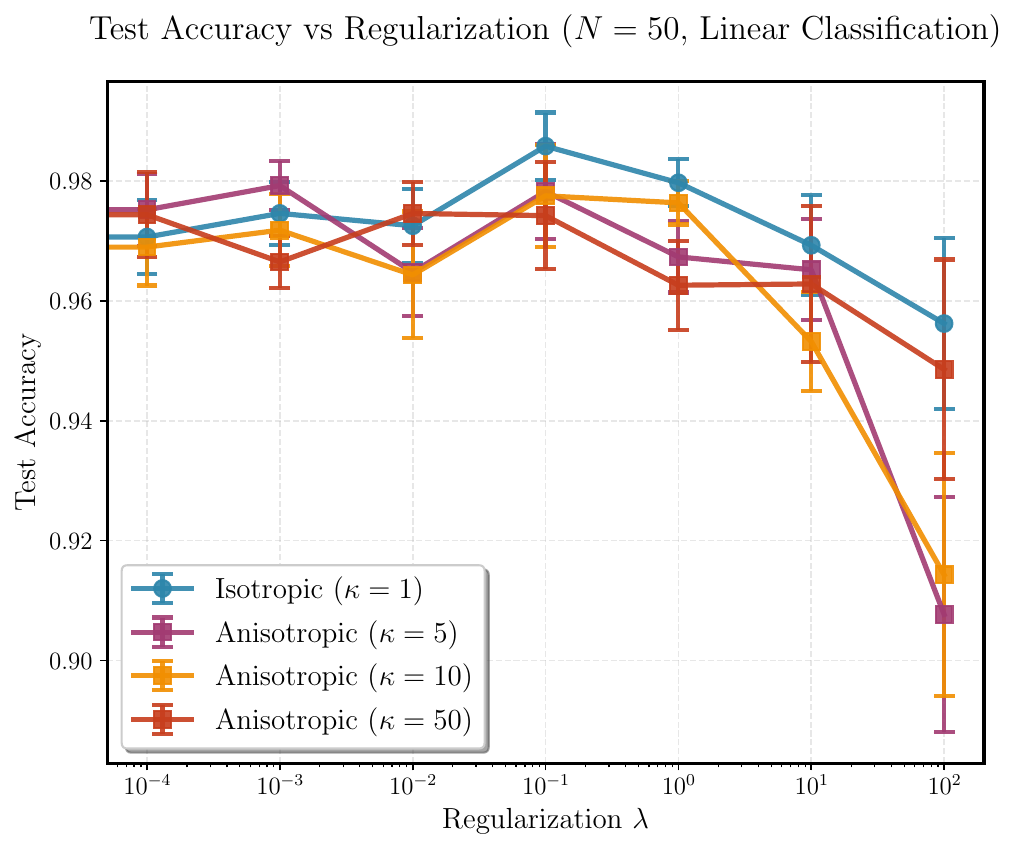}
    \includegraphics[width=0.32\linewidth]{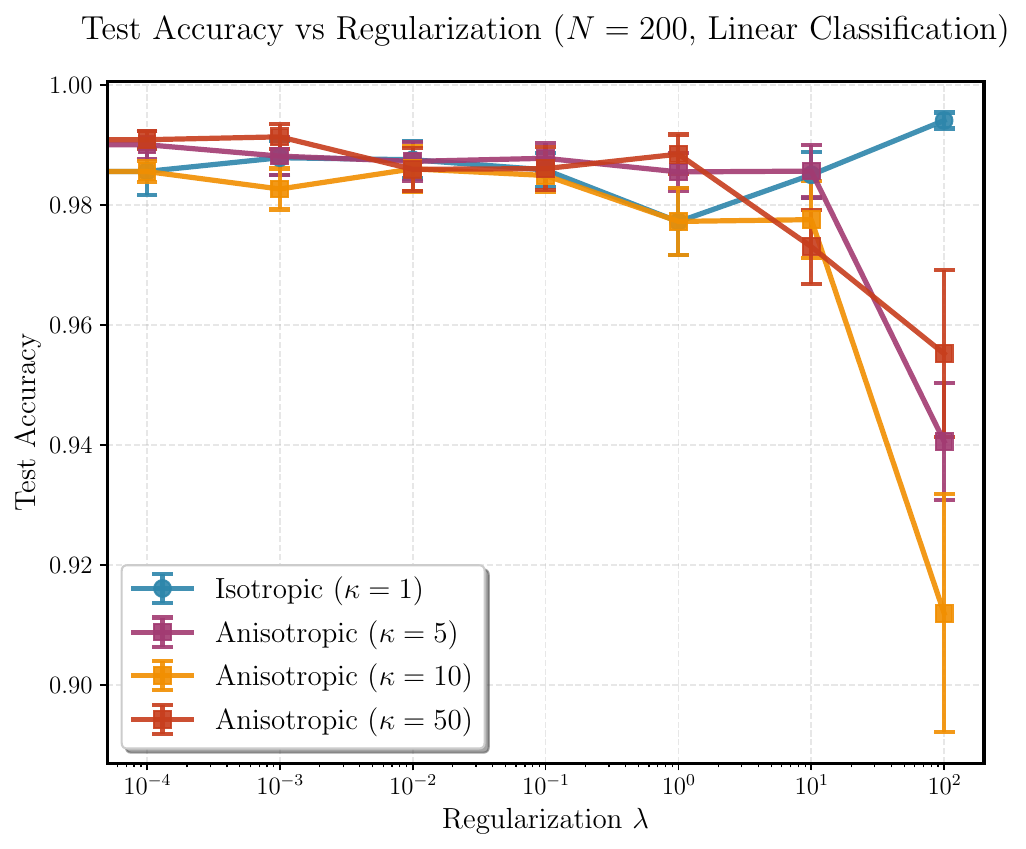}
    \includegraphics[width=0.32\linewidth]{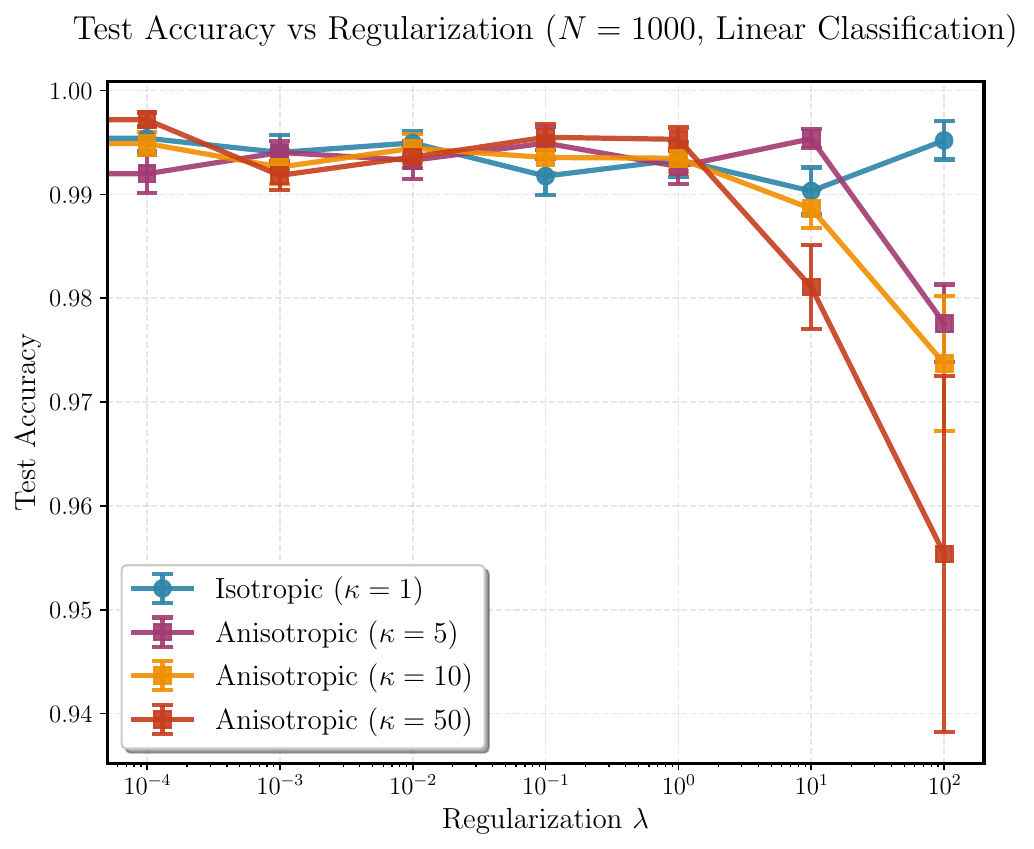}\\
    \includegraphics[width=0.32\linewidth]{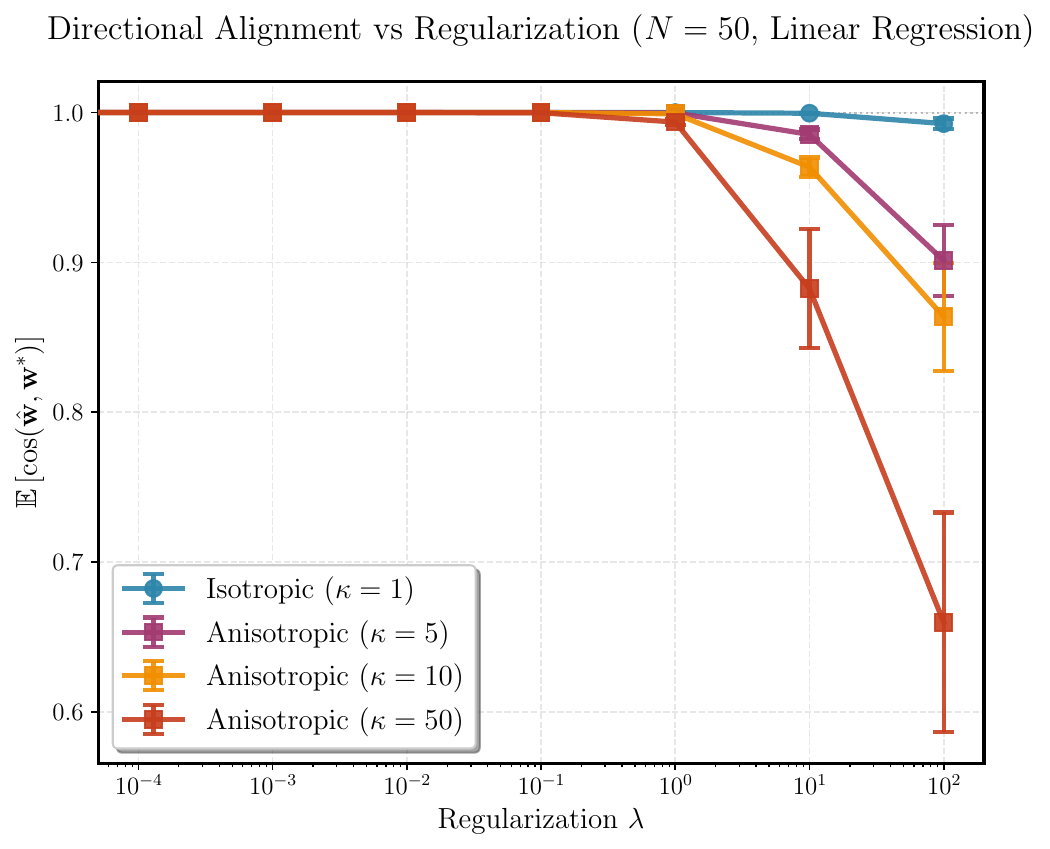}
    \includegraphics[width=0.32\linewidth]{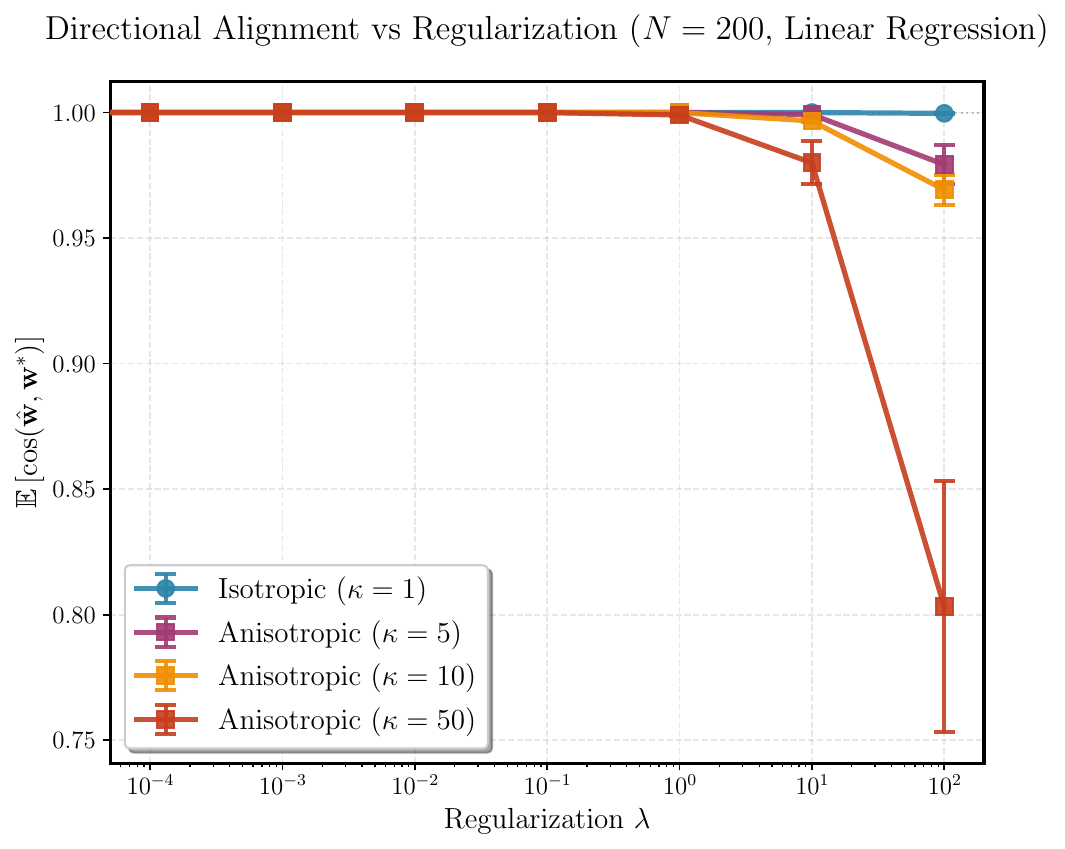}
    \includegraphics[width=0.32\linewidth]{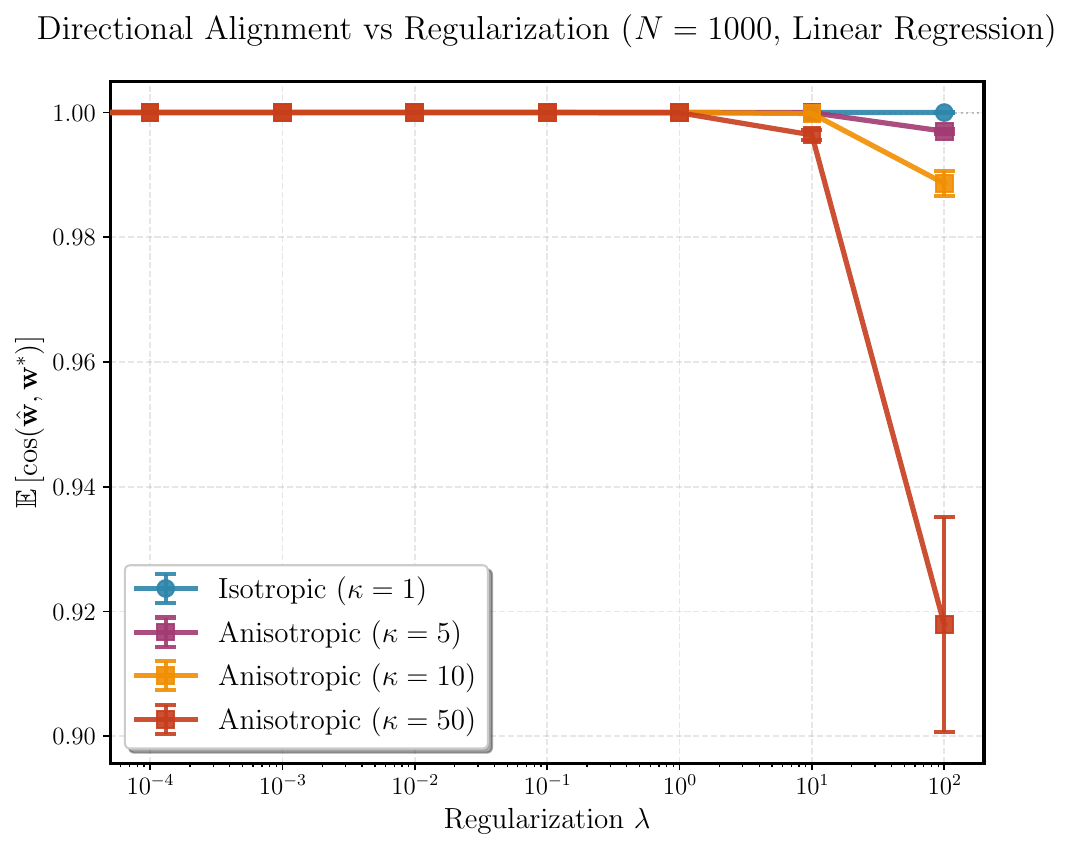}
    \caption{Depiction of accuracy ({\bf top}) and cosine similarity between estimated and true estimator ({\bf bottom}) for the OLS setting with varying strength of Tikhonov regularization ({\bf x-axis)} comparing isotropic and anisotropic embeddings. As per \cref{thm:gradient_bias}, the anisotropic distribution creates a bias in the OLS estimation for nonzero regularization.}
    \label{fig:ols_bias}
\end{figure}

\begin{figure}[t!]
    \includegraphics[width=0.33\linewidth]{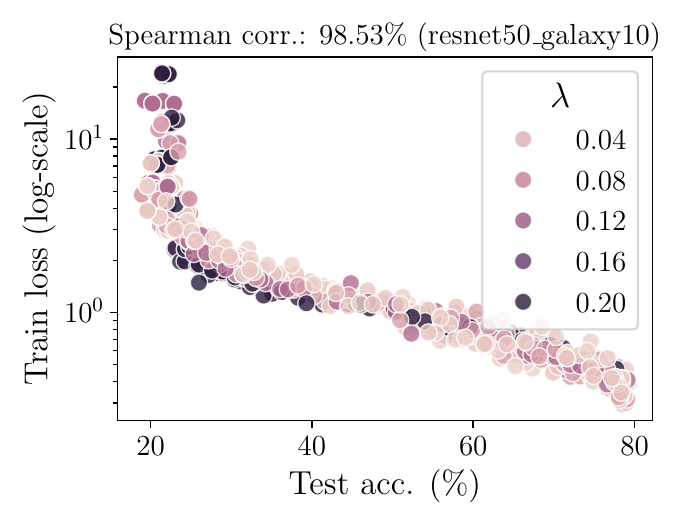}
    \includegraphics[width=0.33\linewidth]{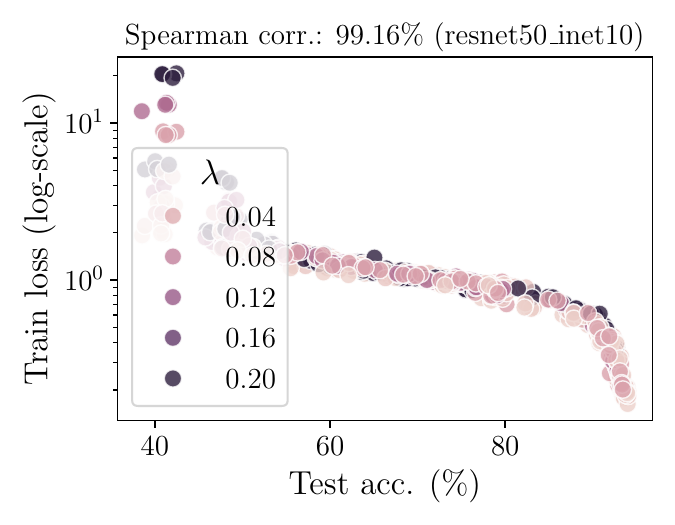}
    \includegraphics[width=0.33\linewidth]{toy_figures/exps/loss_corr/loss_corr_ViT-base-8_inet1k.pdf}
    \includegraphics[width=0.33\linewidth]{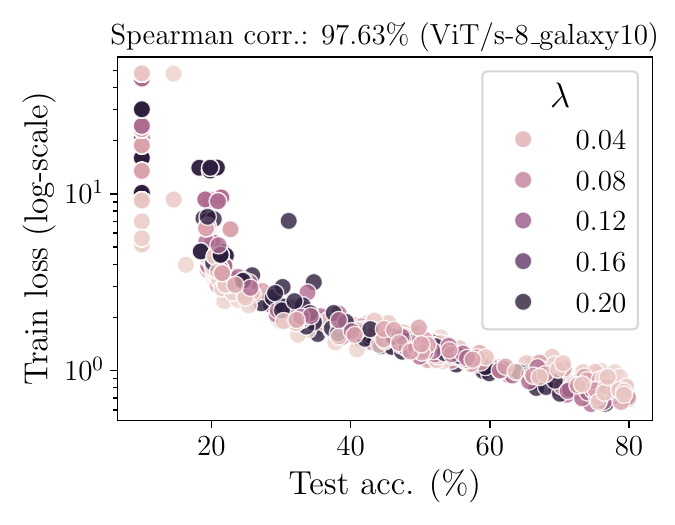}
    \includegraphics[width=0.33\linewidth]{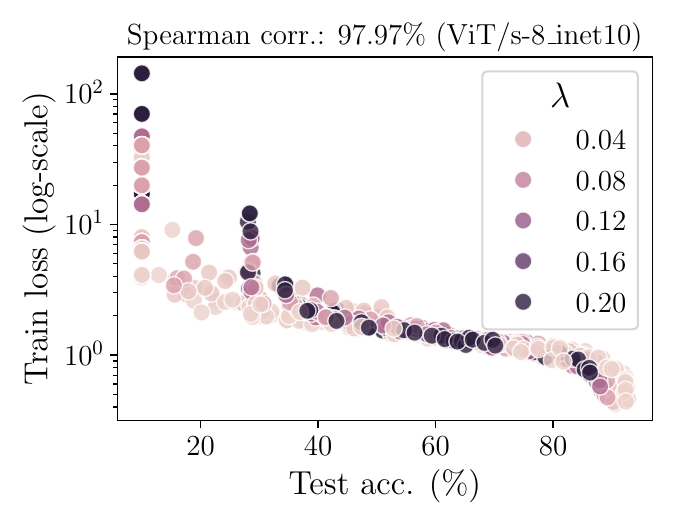}
    \includegraphics[width=0.33\linewidth]{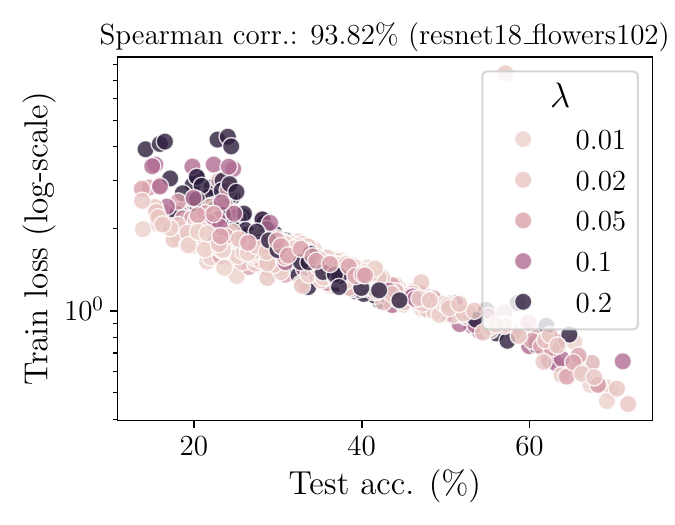}
    \caption{Additional figures provides in \cref{fig:corr_loss_extra}}
    \label{fig:corr_loss_extra}
\end{figure}

\begin{table}[t!]
\caption{Performance metrics across different sample sizes from \cref{fig:galaxy10}}
\label{tab:galaxy10}
\centering
\begin{tabular}{llrrrrrrr}
\toprule
\textbf{Freeze Backbone} & \textbf{Model Name} & \multicolumn{7}{c}{\textbf{Samples per Class}} \\
\cmidrule(lr){3-9}
 &  & \textbf{All} & \textbf{1} & \textbf{2} & \textbf{5} & \textbf{10} & \textbf{100} & \textbf{1000} \\
\midrule
\multirow[c]{6}{*}{\textbf{No}} 
 & \multicolumn{8}{l}{\textit{LeJEPA (Ours)}} \\
 & ConvNeXt-V2 Nano & 82.72 & \textbf{29.42} & \textbf{36.65} & \textbf{50.94} & \textbf{59.85} & \textbf{75.34} & 81.97 \\
 & LeViT-128 & 79.41 & 18.45 & 24.08 & 33.11 & 41.76 & 64.59 & 77.59 \\
 & ResNet-18 & 82.15 & 23.34 & 31.56 & 43.82 & 54.64 & 73.53 & 81.41 \\
 & ResNet-34 & \textbf{83.28} & 24.27 & 31.51 & 44.23 & 53.95 & 74.93 & \textbf{82.32} \\
\cmidrule(lr){2-9}
 & \multicolumn{8}{l}{\textit{Baselines}} \\
 & DINOv2 Small & 78.34 & 21.05 & 21.71 & 30.33 & 36.23 & 60.81 & 75.55 \\
 & DINOv3 ViT-S/16 & 81.60 & 24.71 & 29.43 & 37.71 & 44.71 & 69.87 & 80.54 \\
\midrule
\multirow[c]{6}{*}{\textbf{Yes}} 
 & \multicolumn{8}{l}{\textit{LeJEPA (Ours)}} \\
 & ConvNeXt-V2 Nano & 76.52 & 28.74 & 36.65 & 50.60 & 59.50 & 72.62 & 77.24 \\
 & LeViT-128 & 69.00 & 25.85 & 33.30 & 45.52 & 52.43 & 64.37 & 69.39 \\
 & ResNet-18 & 75.95 & 30.48 & 38.22 & 50.85 & 58.86 & 72.70 & 76.39 \\
 & ResNet-34 & \textbf{78.17} & \textbf{31.08} & \textbf{38.33} & \textbf{52.26} & \textbf{60.63} & \textbf{74.77} & \textbf{78.62} \\
\cmidrule(lr){2-9}
 & \multicolumn{8}{l}{\textit{Baselines}} \\
 & DINOv2 Small & 67.62 & 27.68 & 32.22 & 40.72 & 47.72 & 62.49 & 67.89 \\
 & DINOv3 ViT-S/16 & 71.38 & 30.17 & 36.65 & 45.74 & 51.51 & 65.90 & 71.35 \\
\bottomrule
\end{tabular}
\end{table}

\begin{figure}
    \centering
    \includegraphics[width=\linewidth]{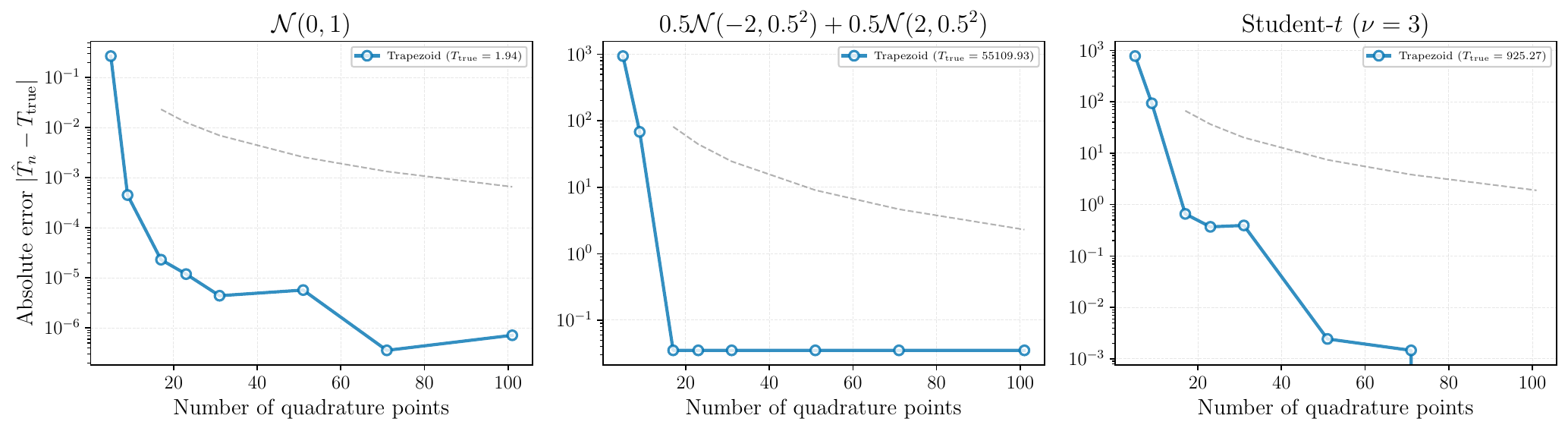}
    \caption{Proposed trapezoid quadrature for the Epps-Pulley statistic as implemented in \cref{lst:epps-pulley-pytorch}. We depict the approximation error of the integral for various distributions, demonstrate rapid convergence (faster than quadratic show in {\bf grey line}) across possible embedding distributions.}
    \label{fig:quadrature}
\end{figure}

\begin{figure*}[t!]
    
    \includegraphics[width=0.33\linewidth]{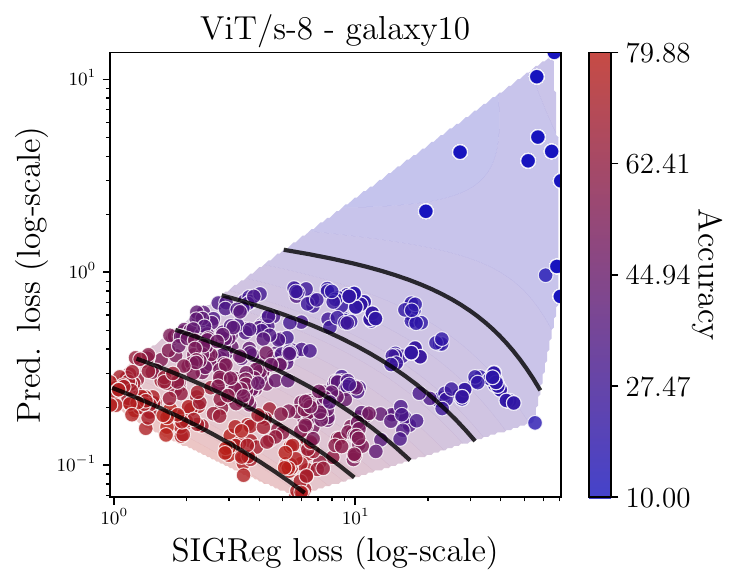}
    \includegraphics[width=0.33\linewidth]{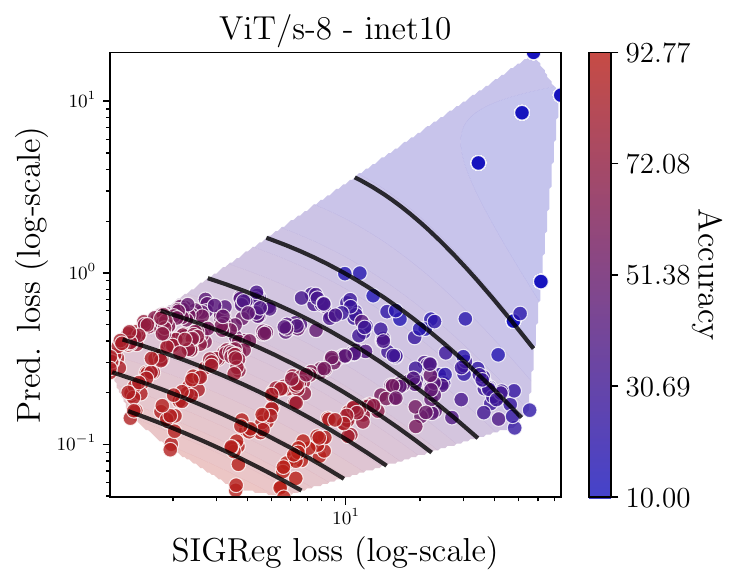}
    \includegraphics[width=0.33\linewidth]{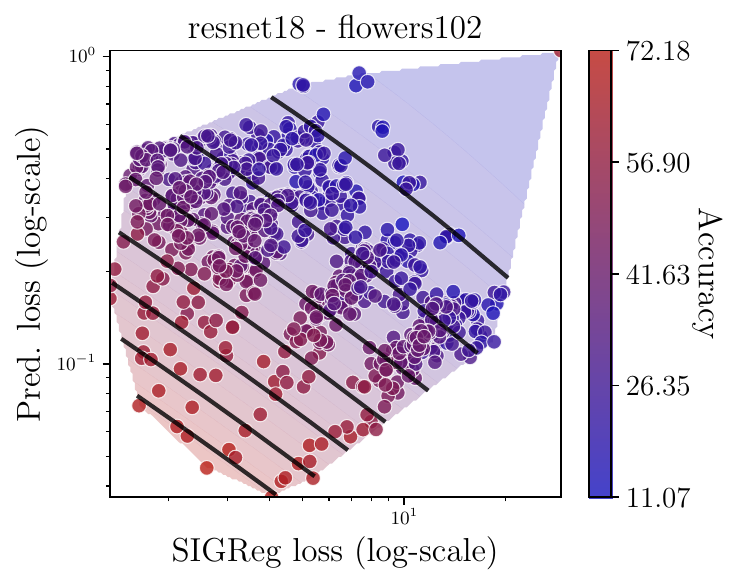}
    \caption{Additional figures for \cref{fig:heatmap_loss}.}
    \label{fig:extra_heatmap_loss}
\end{figure*}

\begin{table*}[t!]
\caption{Top 1 accuracy (in \%) with LeJEPA pretraining on Imagenet-100 for 400 epochs (All values are percentages)}
\label{tab:proj_pred}
\begin{tabular}{ll|lll|lll|lll}
\toprule
 & backbone & \multicolumn{3}{c}{resnet50} & \multicolumn{3}{c}{vit\_small\_patch8\_224} & \multicolumn{3}{c}{vit\_tiny\_patch8\_224} \\
 & Projector & 1-layer & 2-layer & 3-layer & 1-layer & 2-layer & 3-layer & 1-layer & 2-layer & 3-layer \\
w/ predictor & w/ SWA &  &  &  &  &  &  &  &  &  \\
\midrule
\multirow[c]{2}{*}{False} & False & 79.71 & 82.44 & 83.93 & 76.59 & 80.77 & 81.07 & 71.79 & 76.87 & 80.37 \\
 & True & 79.79 & 82.69 & 83.50 & 79.96 & 83.63 & 84.12 & 75.86 & 82.36 & 80.50 \\
\multirow[c]{2}{*}{True} & False & 79.41 & 82.44 & 83.57 & 77.58 & 79.41 & 81.91 & 67.74 & 77.64 & 80.73 \\
 & True & 78.87 & 82.04 & 82.82 & 77.11 & 81.77 & 82.58 & 69.53 & 78.27 & 79.77 \\
\bottomrule
\end{tabular}
\end{table*}

\begin{table*}[t!]
\caption{{\bf Small architecture in-domain LeJEPA pretraining} from random initialization across datasets and architectures, with frozen backbone linear evaluation. First, {\bf LeJEPA is able to produce near state-of-the-art performances on tiny dataset with only a thousand samples}, e.g., flowers102. Second, {\bf on non-natural image data, LeJEPA clearly outperforms the latest frontier vision models}, e.g., Galaxy10. See \cref{fig:galaxy10} for additional experiments with varying number of training samples and with full finetuning.}
\label{tab:in_domain}
\begin{tabular}{llllllll}
\toprule
 & Pretraining & flowers102 & cifar100 & food101 & inet10 & cifar10 & galaxy10 \\
 & \# train. samples & 1020 & 50000 & 75750 & 13000 & 50000 & 11008 \\
\midrule
LeJEPA (convnextv2\_nano) 14M & in-domain & 64.34 & 69.26 & 69.59 & 90.81 & 92.22 & 76.05 \\
LeJEPA (resnet18) 11M & in-domain & 74.57 & 69.94 & 73.57 & 92.36 & 92.51 & 75.32 \\
LeJEPA (resnet34) 21M & in-domain & 71.85 & 70.44 & 74.95 & 92.80 & 93.16 & 77.29 \\
LeJEPA (resnext26ts) 8M & in-domain & 82.19 & 69.10 & 76.77 & 92.82 & 91.59 & 73.78 \\
LeJEPA (swin\_tiny) 27M & in-domain & 63.94 & 65.08 & 78.40 & 92.87 & 92.67 & 74.89 \\
IJEPA-inet22k (ViT-H/14) 630M & inet1k & 85.76 & 86.93 & 81.06 & 98.65 & 97.77 & 62.93
\end{tabular}
\end{table*}

\begin{table}[t!]
    \centering
    \setlength{\tabcolsep}{0.5em}
    \caption{Time (in millisecond) to compute the proposed SIGReg loss from \cref{lst:epps-pulley-pytorch} on a Tesla V100-SXM2-16GB for varying mini-batch size ($N$), number of slices ($M$), integration points. Results are computed over $10$ runs.}
    \label{tab:times}
    \begin{tabular}{rrrrr}
\toprule
N & M & \makecell{\# integration \\ points} & mean (ms) & std (ms) \\
\midrule
 512 & 512 &  16 & 0.465236 & 0.011642 \\
 512 & 512 &  64 & 0.461317 & 0.003894 \\
 512 & 512 &  256 & 0.627644 & 0.003337 \\
 2048 & 512 &  16 & 1.406441 & 0.002415 \\
 8192 & 512 &  16 & 6.188304 & 0.007226 \\
 8192 & 8192 &  16 & 8.685009 & 0.038829 \\
 32768 & 512 &  16 & 26.373118 & 0.012732 \\
 512 & 2048 &  16 & 0.465614 & 0.005274 \\
 512 & 8192 &  16 & 0.670379 & 0.006854 \\
\bottomrule
\end{tabular}
\end{table}

\begin{table}[t!]
\caption{Number of \cref{fig:lambda_views}.}
\label{tab:lambda_perf}
\begin{tabular}{lllllllllllll}
\toprule
 & \multicolumn{12}{c}{resnet50} \\
$\lambda$ & 0.001 & 0.005 & 0.010 & 0.020 & 0.025 & 0.050 & 0.100 & 0.150 & 0.200 & 0.300 & 0.400 & 0.500 \\
\#views &  &  &  &  &  &  &  &  &  &  &  &  \\
\midrule
2 & 81.41 & 82.73 & 83.49 & 82.99 & 82.23 & - & - & - & - & - & - & - \\
4 & 79.88 & 83.04 & 84.36 & 84.68 & 84.33 & 83.00 & 82.91 & 81.05 & 78.58 & - & - & - \\
8 & 76.67 & 81.58 & 83.59 & 83.49 & 83.76 & 84.32 & 83.66 & 83.07 & 82.16 & 81.00 & 79.25 & 77.72  \\
\bottomrule
\end{tabular}
\end{table}

\section{Details on Low-Discrepancy Sequences}
\label{sec:low_discrepancy}

Quasi-Monte Carlo (QMC) methods, such as the Sobol sequence, are widely used to generate low-discrepancy samples in the unit hypercube, providing improved uniformity over purely random sampling. To obtain samples uniformly distributed on the hypersphere, each QMC point is mapped to a standard normal vector via the inverse cumulative distribution function (CDF), and then projected onto the sphere by normalization. This approach leverages the rotational invariance of the multivariate normal distribution, ensuring that the resulting directions are uniformly distributed on the sphere's surface. While the low-discrepancy property is not strictly preserved under this nonlinear mapping, the resulting samples are empirically more uniform than random samples and are standard in high-dimensional applications \cite{marsaglia1972choosing,dick2010digital,caflisch1998monte}.
\begin{algorithmic}[1]
\Require Number of points $N$, dimension $d$
\Ensure Points $\{\mathbf{y}_i\}_{i=1}^N$ quasi-uniformly distributed on $\mathbb{S}^{d-1}$
\For{$i = 1$ to $N$}
    \State Generate $\mathbf{x}_i \in [0,1]^d$ as the $i$-th point of a Sobol sequence
    \State Transform each component: $z_{i,j} = \Phi^{-1}(x_{i,j})$ for $j = 1, \ldots, d$ \Comment{$\Phi^{-1}$ is the inverse CDF of the standard normal}
    \State Normalize: $\mathbf{y}_i = \mathbf{z}_i / \|\mathbf{z}_i\|_2$
\EndFor
\end{algorithmic}

\section{Shapiro-Wilk Test}
\label{sec:shapiro_wilk}

Let X1 < X2 < . . . < Xn denote an ordered random sample of size n from a standard normal distribution. Also, let mÂ 5 (m1,m2,...,mn) be the vector of expected values of standard normal order statistics, and let V 5 (vij ) be the corresponding n 3 n covariance matrix, so that
\begin{equation}
    E\left(X_{i}\right)=m_{i} \quad \text { and } \quad \operatorname{cov}\left(X_{i}, X_{j}\right)=v_{i j}, \quad i, j=1,2, \ldots, n
\end{equation}
The W test statistic \cite{shapiro1965analysis} for normality is then denoted by
\begin{equation}
    \begin{array}{l}
W=\frac{\left(\sum_{i=1}^{n} a_{i} Y_{i}\right)}{\sum_{i=1}^{n}\left(Y_{i} -\bar{Y}\right)^{2}}=\frac{(\mathbf{a} \mathbf{Y})}{S^{2}}\\
\mathbf{a}^{\prime}=\left(a_{1}, a_{2}, \ldots, a_{n}\right)=\mathbf{m} \mathbf{V}^{-1}\left(\mathbf{m} \mathbf{V}^{-1} \mathbf{V}^{-1} \mathbf{m}\right)^{-1 / 2}\\
\mathrm{S}^{2}=\sum_{i=1}^{n}\left(Y_{i}-\bar{Y}\right)^{2}
\end{array}
\end{equation}

\cite{shapiro1972approximate} suggested replacing the covariance matrix V by the identity matrix I, because for large samples, the observations {Yi} may be treated as if they are independent (see \cite{gupta1952estimation}). Another asymptotic extension was suggested by \cite{weisburg1975approximate}
\begin{equation}
    E\left(X_{i}\right)=m_{i} \approx\Phi^{-1}\left(\frac{i-\frac{3}{8}}{n+\frac{1}{4}}\right) \quad i=1,2, \ldots, n
\end{equation}
building atop \cite{elfving1947asymptotical}'s approximation but using $3/8$ instead of $\pi/8$.

\cite{rahman1997modification} proposed another variation using the approximation for the expected values of order statistics given by \cite{blom1958statistical} and the approximations for the elements of the variance± covariance matrix given by \cite{blom1958statistical,mosteller2006some}. These approximations are
\begin{equation}
    E\left(X_{i}\right)=m_{i} \approx \Phi^{-1}\left(\frac{i}{N+1}\right), \quad i=1,2, \ldots, n
\end{equation}
\begin{equation}
    \operatorname{cov}\left(X_{i}, X_{j}\right)=v_{i j} \approx \frac{p_{i} p_{j}}{(n+2) f\left(m_{i}\right) f\left(m_{j}\right)}, \quad i, j=1,2, \ldots, n
\end{equation}
\begin{equation}
    p_{i}=\frac{i}{n+1}
\end{equation}

We know (see \cite{hammersley1954estimation,plackett1958linear})
\begin{equation}
    \begin{array}{l}
\mathbf{V}^{-1}=(n+1)(n+2) \\
\times\left(\begin{array}{cccccc}
2 \phi^{2}\left(m_{1}\right) & -\phi\left(m_{1}\right) \phi\left(m_{2}\right) & 0 & 0 & \ldots & 0 \\
-\phi\left(m_{1}\right) \phi\left(m_{2}\right) & 2 \phi^{2}\left(m_{2}\right) & -\phi\left(m_{2}\right) \phi\left(m_{3}\right) & 0 & \ldots & 0 \\
0 & -\phi\left(m_{2}\right) \phi\left(m_{3}\right) & 2 \phi^{2}\left(m_{3}\right) & -\phi\left(m_{3}\right) \phi\left(m_{4}\right) & \ldots & 0 \\
\vdots & & & & & \\
0 & 0 & 0 & 0 & \ldots & 2 \phi^{2}\left(m_{n}\right)
\end{array}\right)
\end{array}
\end{equation}

\section{Multivariate Statistics}
\label{sec:multivariate_tests}



We ideally would like to compare the distributions. One slight variation is to compare the Characteristic function of the distributions. Given samples $\vx_1,\dots,\vx_N$, the Empirical Characteristic Function (ECF) is defined as 
\begin{align*}
    \hat{\psi}_{N}(\vt) = \frac{1}{N}\sum_{n=1}^{N} e^{-i \vt^\top \vy_n}.
\end{align*}
We can now compare our ECF to the one of the target distribution and build the statistic
\begin{align*}
    N\int |\hat{\psi}_{N}(\vt) - \psi_{0}(\vt) |^2\omega(\vt) dt=N\int |\hat{\psi}_{N}(\vt) - e^{-\|\vt\|_2/2} |^2\omega(\vt) dt,
\end{align*}
if the weighting function is given by $\omega(\vt) = (2\pi \beta ^2 )^{-d/2}e^{-\frac{\|\vt\|_2^2}{2}}$ then the following simplification can be made
\begin{align*}
\mathrm{BHEP}_{n, \beta}= & \frac{1}{n} \sum_{j, k=1}^{n} \exp \left(-\frac{\beta^{2}\left\|Y_{n, j}-Y_{n, k}\right\|^{2}}{2}\right) \\
& -\frac{2}{\left(1+\beta^{2}\right)^{d / 2}} \sum_{j=1}^{n} \exp \left(-\frac{\beta^{2}\left\|Y_{n, j}\right\|^{2}}{2\left(1+\beta^{2}\right)}\right)+\frac{n}{\left(1+2 \beta^{2}\right)^{d / 2}} .
\end{align*}
with $\beta>0$, Baringhaus-Henze-Epps-Pulley. From \footnote{\url{https://www.routledge.com/Density-Estimation-for-Statistics-and-Data-Analysis/Silverman/p/book/9780412246203?srsltid=AfmBOoodlL-CtlqL0JVC-LcP6mOWw6VTt51_YstdZOW4W3iuicu1VFyg}}
leading to the HZ test \footnote{\url{https://www.tandfonline.com/doi/abs/10.1080/03610929008830400}} uses
\begin{equation}
    \beta_{n}=2^{-1 / 2}((2 d+1) n / 4)^{1 /(d+4)}
\end{equation}

the same can be done with the moment generating function \footnote{\url{https://arxiv.org/pdf/1711.07199}}
\begin{align*}
T_{n, \beta}=\pi^{d / 2}\left(\frac{1}{n} \sum_{i, j=1}^{n} \frac{1}{\beta^{d / 2}} \exp \left(\frac{\left\|Y_{n, i}+Y_{n, j}\right\|^{2}}{4 \beta}\right)+\frac{n}{(\beta-1)^{d / 2}}\right. \\
\left.-2 \sum_{j=1}^{n} \frac{1}{(\beta-1 / 2)^{d / 2}} \exp \left(\frac{\left\|Y_{n, j}\right\|^{2}}{4 \beta-2}\right)\right),
\end{align*}
here with $\beta>2$

There is also one combining both\footnote{\url{https://arxiv.org/pdf/1706.03029}}!
\begin{equation}
    \begin{array}{l}
T_{n, \gamma}:=\int_{\mathbb{R}^{d}} U_{n}^{2}(t) w_{\gamma}(t) \mathrm{d} t\\
U_{n}(t):=\sqrt{n}\left(R_{n}(t) M_{n}(t)-1\right)
\end{array}
\end{equation}
\begin{equation}
    \begin{aligned}
T_{n, \gamma}= & \left(\frac{\pi}{\gamma}\right)^{d / 2}\left\{\frac { 1 } { 2 n ^ { 3 } } \sum _ { j , k , l , m = 1 } ^ { n } \left[\exp \left(\frac{\left\|Y_{j k}^{+}\right\|^{2}-\left\|Y_{\ell m}^{-}\right\|^{2}}{4 \gamma}\right) \cos \left(\frac{Y_{j k}^{+\top} Y_{\ell m}^{-}}{2 \gamma}\right)\right.\right. \\
+ & \left.\exp \left(\frac{\left\|Y_{j k}^{+}\right\|^{2}-\left\|Y_{\ell m}^{+}\right\|^{2}}{4 \gamma}\right) \cos \left(\frac{Y_{j k}^{+\top} Y_{\ell m}^{+}}{2 \gamma}\right)\right] \\
& \left.-\frac{2}{n} \sum_{j, k=1}^{n} \exp \left(\frac{\left\|Y_{n, j}\right\|^{2}-\left\|Y_{n, k}\right\|^{2}}{4 \gamma}\right) \cos \left(\frac{Y_{n, j}^{\top} Y_{n, k}}{2 \gamma}\right)+n\right\},
\end{aligned}
\end{equation}
and its simplified version
\begin{equation}
    \widetilde{T}_{n, \gamma}:=\int_{\mathbb{R}^{d}} U_{n}(t) w_{\gamma}(t) \mathrm{d} t .
\end{equation}
\begin{equation}
    \widetilde{T}_{n, \gamma}=\left(\frac{\pi}{\gamma}\right)^{d / 2} \sqrt{n}\left(\frac{1}{n^{2}} \sum_{j, k=1}^{n} \exp \left(\frac{\left\|Y_{n, j}\right\|^{2}-\left\|Y_{n, k}\right\|^{2}}{4 \gamma}\right) \cos \left(\frac{Y_{n, j}^{\top} Y_{n, k}}{2 \gamma}\right)-1\right)
\end{equation}

Also one testing the derivative \footnote{\url{https://arxiv.org/pdf/1901.03986}}

\begin{equation}
    \mathrm{HV}_{n, \gamma}:=n \int\left\|\nabla M_{n}(t)-t M_{n}(t)\right\|^{2} \widetilde{w}_{\gamma}(t) \mathrm{d} t
\end{equation}
\begin{equation}\mathrm{HV}_{n, \gamma}=\frac{1}{n}\left(\frac{\pi}{\gamma}\right)^{d / 2} \sum_{j, k=1}^{n} \exp \left(\frac{\left\|Y_{n, j, k}^{+}\right\|^{2}}{4 \gamma}\right)
\left(Y_{n, j}^{\top} Y_{n, k}-\frac{\left\|Y_{n, j, k}^{+}\right\|^{2}}{2 \gamma}+\frac{d}{2 \gamma}+\frac{\left\|Y_{n, j, k}^{+}\right\|^{2}}{4 \gamma^{2}}\right) .
\end{equation}

skewness
\footnote{\url{https://www.jstor.org/stable/2334770}}:
\begin{equation}
    b_{1, d}=\frac{1}{n^{2}} \sum_{j, k=1}^{n}\left(Y_{n, j}^{\top} Y_{n, k}\right)^{3}
\end{equation}
skewness
\footnote{\url{https://link.springer.com/article/10.1007/s13171-020-00211-6}}:
\begin{equation}
    \widetilde{b}_{1, d}=\frac{1}{n^{2}} \sum_{j, k=1}^{n} Y_{n, j}^{\top} Y_{n, k}\left\|Y_{n, j}\right\|^{2}\left\|Y_{n, k}\right\|^{2}
\end{equation}
which should be 0 for Gaussian and Kurtosis which should be d(d+2)
\begin{equation}
    b_{2, d}=\frac{1}{n} \sum_{j=1}^{n}\left\|Y_{n, j}\right\|^{4}
\end{equation}

\end{document}